\documentclass{article}

\usepackage{microtype}
\usepackage{graphicx}
\usepackage{subfigure}
\usepackage[table]{xcolor}
\usepackage{booktabs} % for professional tables

\usepackage{hyperref}

\usepackage{amsmath, amssymb, amsfonts, amsthm}
\usepackage{enumitem}
\usepackage{blkarray}
\usepackage{mathrsfs}  

\usepackage{tikz}
\usetikzlibrary{arrows}
\usetikzlibrary{calc}
\usetikzlibrary{math}

\newtheorem{thm}{Theorem}[section]
\newtheorem*{thm*}{Theorem}

\newtheorem{conj}[thm]{Conjecture}

\newtheorem{prop}[thm]{Proposition}
\theoremstyle{definition}
\newtheorem{definition}[thm]{Definition}
\newtheorem*{definition*}{Definition}

\newtheorem*{example*}{Example}

\newtheorem*{framework*}{Unifying OT Framework}

\theoremstyle{remark}
\newtheorem{remark}{Remark}

\newtheoremstyle{smplain}
  {\topsep}   % ABOVESPACE
  {\topsep}   % BELOWSPACE
  {\itshape}  % BODYFONT
  {0pt}       % INDENT (empty value is the same as 0pt)
  {\bfseries} % HEADFONT
  {.}         % HEADPUNCT
  {5pt plus 1pt minus 1pt} % HEADSPACE
  {\thmname{#1}\thmnumber{ {#2}}\thmnote{ (#3)}}          % CUSTOM-HEAD-SPEC
  
\theoremstyle{smplain}
\newcounter{sdecorator}
\newtheorem{smthm}{Theorem}[subsection]
\newtheorem{smlem}[smthm]{Lemma}

\newcommand{\Mrow}[2]{\mathbf{{#1}}_{(#2,\_)}}
\newcommand{\Mcol}[2]{\mathbf{{#1}}_{(\_,#2)}}

\newcommand{\fun}[2]{#1\!\left(#2\right)}

\newcommand{\PP}[2]{\fun{P_{#1\!}}{#2}} 
\newcommand{\given}{\vert}

\newcommand{\dataSpace}{\mathcal D} 

\newcommand{\concept}{h} 
\newcommand{\conceptSpace}{\mathcal H}

\newcommand{\LL}{\PP{L}{\concept\given d}}
\newcommand{\LLpri}{\PP{L_0}{\concept}}
\newcommand{\LLmar}{\PP{L}{d}}

\newcommand{\TT}{\PP{T}{d\given\concept}} 
\newcommand{\TTpri}{\PP{T_0}{d}}
\newcommand{\TTmar}{\PP{T}{\concept}}

\newcommand{\rr}{\mathbf{r}} 
\newcommand{\cc}{\mathbf{c}}

\newcommand{\normalize}[2]{\mathscr{N}_{\text{#1}}\left(#2\right)}

\usepackage[accepted]{icml2020}
\icmltitlerunning{Sequential Cooperative Bayesian Inference}

\begin{document}

\twocolumn[
\icmltitle{Sequential Cooperative Bayesian Inference}



\icmlsetsymbol{equal}{*}

\begin{icmlauthorlist}
\icmlauthor{Junqi Wang}{run}
\icmlauthor{Pei Wang}{run}
\icmlauthor{Patrick Shafto}{run}
\end{icmlauthorlist}

\icmlaffiliation{run}{CoDaS Lab, Department of Math \& CS, Rutgers University at Newark,
  New Jersey, USA}

\icmlcorrespondingauthor{Junqi Wang}{junqi.wang@rutgers.edu}%{\href{mailto:junqi.wang@rutgers.edu}{junqi.wang@rutgers.edu}} 
%

\icmlkeywords{Bayesian inference, cooperative inference, multi-agent learning, Machine Learning, ICML}

\vskip 0.3in

]



\printAffiliationsAndNotice{}
% leave blank if no need to mention equal contribution
%  Abstract

\begin{abstract}
  Cooperation is often implicitly assumed when learning from other agents. 
Cooperation implies that the agent selecting the data, and the agent learning
from the data, have the same goal, that the learner infer the intended
hypothesis.
Recent models in human and machine learning have demonstrated the
possibility of cooperation.
We seek foundational theoretical results for cooperative inference 
by Bayesian agents through sequential data.
We develop novel approaches analyzing consistency, rate of convergence and 
stability of Sequential Cooperative Bayesian Inference (SCBI). Our analysis of the effectiveness, sample efficiency and robustness show
that cooperation is not only possible in specific instances but theoretically well-founded in general.
We discuss implications for human-human and human-machine cooperation.

% We seek foundational theoretical \pw{and empirical-delete} results for sequential
% cooperative Bayesian inference (SCBI) \pw{for cooperative by Bayesian agents through sequential data (SCBI)}.

% \pw{We provide results on consistency, convergence and stability}
% \pw{develop novel analysis approach.}
% \pw{Thus, we show cooperation is not only possibly but theoretically supported}
% \pw{implication for human-human, human-machine communication.}

% We analyze consistency and speed of convergence for the discrete case, combining
% novel analytic methods and empirical investigations.
% Our analysis proves that sequential cooperative Bayesian inference is
% consistent with asymptotic rate of convergence determined by the joint
% distribution we started with, and empirical evidence suggests that SCBI converges
% at a faster rate than classical Bayesian inference on average, and the
% rate of convergence advantage for SCBI increases with the size of the problem.
% Empirically we observe the continuity of SCBI to the perturbation on
% initial conditions, which suggests robustness of the model.

\end{abstract}

%%%%%%%%%%%%%%%%%%%%%%%%%%%%%%%%%%%%%%%%%%%%%%%%%%%%%%%%%%%%%%%%%%%%%%%%%%%%%%%%
%  MAIN TEXT

\section{Introduction}
\label{sec:intro}
Learning often occurs sequentially, as opposed to in batch, and from data provided by other agents, as opposed to from a fixed random sampling process. The canonical example of sequential learning from an agent occurs in educational contexts where the other agent is a teacher whose goal is to help the learner. 
However, instances appear across a wide range of contexts including informal learning, language, and robotics. 
In contrast with typical contexts considered in machine learning, it is reasonable to expect the cooperative agent to adapt their sampling process after each trial, consistent with the goal of helping the learner learn more quickly. It is also reasonable to expect that learners, in dealing with such cooperative agents, would know the other agent intends to cooperate and incorporate that knowledge when updating their beliefs. In this paper, we analyze basic statistical properties of such sequential cooperative inferences. 

Large behavioral and computational literatures highlight the importance cooperation for learning. 
Across behavioral sciences, cooperative information sharing is believed to be a core feature of human cognition. 
%These interactions are between more and less expert individuals with the shared assumption that the goal is to facilitate learning, and hence of interest in human-human and human-machine interaction. 
Education, where a teacher selects examples for a learner, is perhaps the most obvious case. Other examples appear in linguistic pragmatics \cite{frank2012predicting}, in speech directed to infants \cite{Eaves2016c}, and children's learning from demonstrations \cite{bonawitz2011double}. Indeed, theorists have posited that the ability to select data for and learn cooperatively from others explains humans' ability to learning quickly in childhood and accumulate knowledge over generations \cite{tomasello1999,csibra2009natural}.

Across computational literatures, cooperative information sharing is also believed to be central to human-machine interaction.  %\cite{Eaves2016c,zhu2013machine,Zilles2008,fisac2017pragmatic,ho2016showing}.  
%\pat{machine teaching, algorithmic teaching, bayesian teaching, also human-robot interaction: fisac, ho, milli, others}. [[examples]]
%For example, \cite{fisac2017pragmatic} introduced Pedagogic-pragmatic value alignment, which is a model for human-robot interaction in which human actions are observed by the robot, who then updates their beliefs assuming the human is cooperating to teach the robot. They demonstrate the idea by creating ChefWorld where simulated humans ``teach'' robots their goals through pedagogical interactions. 
%Such demonstrations show the potential utility of cooperation, but do not provide an understanding of the scope and limits for learning. 
%Of work that has analyzed theoretical properties of sequential cooperative learning, the approach has focused on deterministic selection and inference \citep{Zilles2008,Doliwa2014}, which is of limited use for machine learning. 
Examples include pedagogic-pragmatic value alignment in robotics \cite{fisac2017pragmatic}, cooperative inverse reinforcement learning \cite{hadfield2016cooperative}, machine teaching \cite{zhu2013machine}, and Bayesian teaching \cite{Eaves2016c} in machine learning, and Teaching dimension in learning theory \cite{Zilles2008,Doliwa2014}. 
Indeed, rather than building in knowledge or training on massive amounts of data, cooperative learning from humans is a strong candidate for advancing machine learning theory and improving human-machine teaming more generally. %\pat{need to add citation to commented link!}
%https://basicresearch.defense.gov/Portals/61/Future%20Directions%20in%20Human%20Machine%20Teaming%20Workshop%20report%20%20%28for%20public%20release%29_1.pdf

While behavioral and computational research makes clear the importance of cooperation for learning, we lack mathematical results that would establish statistical soundness. 
In the development of probability theory, proofs of consistency and rate of convergence were celebrated results that put Bayesian inference on strong mathematical footing \cite{doob1949application}. 
%\pat{not sure if there is a cite of bernstein himself?}
%\pat{what is missing? consistency, convergence and stability}
Moreover, establishment of stability with respect to mis-specification ensured that theoretical results could apply despite the small differences between the model and reality \cite{kadane1978stable,berger1994overview}. Proofs of consistency, convergence, and stability ensure that intuitions regarding probabilistic inference were formalized in ways that satisfied basic desiderata. 

Our goal is to provide a comparable foundation for sequential Cooperative Bayesian Inference as statistical inference for understanding the strengths, limitations, and behavior of cooperating agents. Grounded strongly in machine learning \citep{murphy2012machine,ghahramani2015probabilistic} and human learning \citep{tenenbaum2011grow}, we adopt a probabilistic approach. 
We approach consistency, convergence, and stability using a combination of new analytical and empirical methods. 
The result will be a model agnostic understanding of whether and under what conditions sequential cooperative interactions result in effective and efficient learning. 

Notations are introduced at the end of this section.
Section~\ref{sec:scbi_construction} introduces the model of sequential
cooperative Bayesian inference (SCBI), and Bayesian inference (BI) as the
comparison. Section~\ref{sec:consistency} presents a new analysis approach which we
apply to understanding consistency of SCBI.
Section~\ref{sec:sample_efficiency} presents empirical results analyzing the sample
efficiency of SCBI versus
BI, showing convergence of SCBI is considerably faster.
Section~\ref{sec:stability} presents the empirical results testing 
robustness of SCBI to perturbations.
Section~\ref{sec:application} introduces an application of SCBI in Grid world model.
Section~\ref{sec:related}
describes our contributions in the context of related work, and
Section~\ref{sec:conclusion} discusses implications for machine learning and
human learning.  
% Introduction

\noindent
\textbf{Preliminaries.} %\label{par:prelim}
Throughout this paper, for a vector $\theta$, we denote its $i$-th entry by
$\theta_i$ or $\theta(i)$. Similarly, for a matrix $\mathbf{M}$, we
denote the vector of $i$-th row by $\Mrow{M}{i}$, the vector of $j$-th column by
$\Mcol{M}{j}$, and the entry of $i$-th row and $j$-th column by
$\mathbf{M}_{(i,j)}$ or simply $\mathbf{M}_{ij}$. 
Further, let  $\mathbf{r},\mathbf{c}$ be the column vectors representing the row
and column marginals (sums along row/column) of $\mathbf{M}$. Let $\mathbf{e}_n$ or simply $\mathbf{e}$ be the vector of ones.
% \pw{Maybe, the desired row and column marginals of $\mathbf{M}$?}.
%
The symbol $\normalize{vec}{\theta, s}$ is used to denote the normalization of
a non-negative vector $\theta$, i.e., 
$\normalize{vec}{\theta,s}=\frac{s}{\sum\theta_i}\theta$ with $s=1$ if
absent. Similarly, 
the normalization of matrices are denoted by 
$\normalize{col}{\mathbf{M},\theta}$, with ``col'' indicating 
column normalization (for row normalization, write ``row'' instead), and $\theta$
denotes to which vector of sums the matrix is normalized.
The set of probability distributions on a finite set $\mathcal{X}$ is denoted by
$\mathcal{P}(\mathcal{X})$, we do not distinguish it with the simplex
$\Delta^{|\mathcal{X}|-1}$. 
The language of statistical models and estimators follows the
notations of the book \cite{miescke2008decision}.

\section{The Construction}
\label{sec:scbi_construction}
% In this section we construct \textbf{Sequential Cooperative Bayesian Inference
%   (SCBI)} as a natural generalization of the cooperative inference in sequential
% Form.
% The PROBLEM
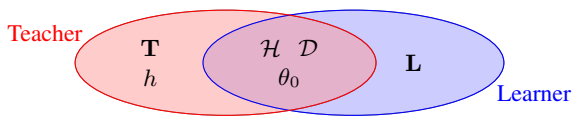
\begin{figure}[ht]
  \centering
  \begin{tikzpicture}
    \filldraw[fill=blue!40, opacity=0.5] (2.2,0) ellipse (2 and 0.7);
    \filldraw[fill=red!40, opacity=0.5] (0.5,0) ellipse (2 and 0.7);
    \draw[blue] (2.2,0) ellipse (2 and 0.7);
    \draw[red] (0.5,0) ellipse (2 and 0.7);
    \node at (-0.5, -0.2) {\small $h$};
    \node at (-0.5, 0.2) {\small $\mathbf{T}$};
    \node at (1.1,0.2) {\small $\mathcal{H}$};
    \node at (1.6,0.2) {\small $\mathcal{D}$};
    \node at (1.35,-0.2) {\small $\theta_0$};
    \node at (3, 0) {\small $\mathbf{L}$};
    \node [color=red] at (-1.9, 0.4) {\small Teacher};
    \node [color=blue] at (4.6,-0.4) {\small Learner};
  \end{tikzpicture}
  \caption{\label{fig:agent_knowledge}Two agents and their knowledge before starting.}
\end{figure}

%The problem is about a cooperative teaching-learning model, where two agents, a
%teacher and a learner, are considered. 
In this paper, we consider cooperative communication models with two agents, which we call a
teacher and a learner. Let
$\mathcal{H}=\{1,\dots,m\}$
% $\mathcal{H}=\{h_1,\dots,h_m\}$
be the set of $m$ hypotheses, i.e., concepts to teach.
The shared \textbf{goal} is for the learner to infer the correct hypothesis
$h\in\mathcal{H}$ which is only known by the teacher at the beginning. 
To facilitate learning, the teacher passes one element from
a finite data set $\mathcal{D}=\{1,\dots, n\}$ sequentially. %,
% $\mathcal{D}=\{d_1,\dots,d_n\}$, called the set of data. 
Each agent has knowledge about the relation between
$\mathcal{H}$ and $\mathcal{D}$, in terms of a positive matrix
whose normalization can be treated as the
likelihood matrix in a Bayesian sense.
Let $\mathbf{T}, \mathbf{L}\in\mathrm{Mat}_{n\times m}(\mathbb{R}^+)$ be the 
matrices for teacher and learner, respectively.

In order to construct a Bayesian theory, 
the learner has an initial prior $\theta_0$
on $\mathcal{H}$, which, along with posteriors $\theta_k (k\ge 1)$,
are elements in $\mathcal{P}(\mathcal{H})=\Delta^{m-1}:=\{\theta\in\mathbb{R}^m:\sum_{i=1}^m\theta(i)=1\}$.
Privately, the teacher knows the true hypothesis $h\in\mathcal{H}$ to teach.
To measure how well a posterior $\theta_k$ performs, 
we may view $h$ as a distribution on $\mathcal{H}$, namely $\widehat{\theta}=\delta_h\in\mathcal{P}(\mathcal{H})$, and calculate the $L^1$-distance $||\theta_k-\widehat{\theta}||_1$ on $\mathcal{P}(\mathcal{H})=\Delta^{m-1}\subseteq\mathbb{R}^m$.

% In practice, the teacher usually has only a hypothesis $h\in\mathcal{H}$ to teach,
% in which case $\widehat{\theta}=\delta_h$.

We assume that $\mathcal{H}$, $\mathcal{D}$, $\mathbf{T}$,
$\mathbf{L}$ and $\theta_0$ satisfy:

(\textbf{i}) There are no fewer data than hypotheses ($n\ge m$).

(\textbf{ii}) The hypotheses are distinguishable, i.e., there is no
$\lambda\in\mathbb{R}$ such that
$\mathbf{T}_{(\_,i)}=\lambda\mathbf{T}_{(\_,j)}$ for any
$i\ne j$, and so is $\mathbf{L}$.

(\textbf{iii}) $\mathbf{T}$ is a \textit{scaled matrix} of $\mathbf{L}$, i.e.,
there exist invertible diagonal matrices $\mathbf{E_1}$ and $\mathbf{E_2}$
such that $\mathbf{T}=\mathbf{E_1LE_2}$. (Both agents aware this assumption,
though possibly neither know the other's matrix.)

(\textbf{iv}) $\theta_0$ is known by the teacher.

Our model is constructed and studied under 
these assumptions (Sec.~\ref{sec:consistency} and Sec.~\ref{sec:sample_efficiency}). 
We also studied stability under violations of (iii) and (iv), where we assume that
$\mathbf{T}$ and teacher's knowledge $\theta_0^T$ about $\theta_0$ is slightly
different from (some scaled matrix of) $\mathbf{L}$ and $\theta_0$ (Sec.~\ref{sec:stability}).
%\pat{i think we need to make the logic of later sections clearer here}
Assumption (iii) is a relaxation of 
% \pw{how about change `stronger than' to  `analogous to' or `relaxation'}
the assumption of Bayesian inference that $\mathbf{T}=\mathbf{L}=\mathbf{M}$ is the likelihood matrix.
Practically, we may achieve (iii) by
%Assumption (iii) could be achieved from 
adding to the common ground a shared
matrix $\mathbf{M}$ (e.g. joint distribution on $\mathcal{D}$ and $\mathcal{H}$) and scaling it to $\mathbf{T}$ and $\mathbf{L}$. 
% Assumption (iii) is necessary for consistency of SCBI.
% This assumption is a relaxation of 
% % \pw{how about change `stronger than' to  `analogous to' or `relaxation'}
% that of Bayesian inference where
% $\mathbf{T}=\mathbf{L}=\mathbf{M}$ as the likelihood matrix.
We may obtain $\mathbf{M}$ by taking the same ground model or
using the same statistical data (e.g. historical statistical records).
In fact, with (iii), it does not affect the process of inference
whether $\mathbf{M}$ is accessible to agents.
% In fact, when we assume that (iii) is valid, the mechanism that produces (iii) is not mathematically important to our discussion. 

In SCBI (see details in later this section), thanks to the property 
that a matrix and its scaled matrices behave the same
in Sinkhorn scaling \cite{hershkowitz1988classifications},
the pre-processings of $\mathbf{T}$ and  of $\mathbf{L}$ lead to the same results under (iii) and (iv). Thus
%\pw{add-(see details in later this section)}, 
assumption (iii) is equivalent to:

(\textbf{iii'}) $\mathbf{T}=\mathbf{L}=\mathbf{M}$ where $\mathbf{M}$ is
a column-stochastic matrix.

%since the Sinkhorn scaling is taken at the beginning of every round of SCBI
%(see the construction later in this section). %\pw{delete this sentence? reader does not understand it either way:)}
We assume (iii') is valid until we discuss stability.

% that is,
% $\mathbf{M}:=\mathscr{N}_{\text{col}}(\mathbf{M}^{JD})$. 
  
% There are several strategies in the literature. \pw{say a bit more?} The one we follow is a
% sequential probabilistic setup \pw{add some citations?} , where the teacher samples a sequence of data
% from a sequence of distributions, and the learner learns by maintaining a
% posterior distribution on $\mathcal{H}$ through Bayesian inference (with
% a sequence of likelihood distributions).

In our setup, the teacher teaches in sequence. At each round the teacher chooses
a data point from $\mathcal{D}$ by sampling according to a distribution. And the
learner learns by maintaining a
posterior distribution on $\mathcal{H}$ through Bayesian inference with likelihood matrices not necessarily fixed.

Formally, the teacher's job is to select a sequence of data
$(d_k)_{k\in\mathbb{N}}$ by first constructing a sequence of random variables $(\mathrm{D}_k)_{k\in\mathbb{N}}$, then
sampling each $d_k$ as a realization of $\mathrm{D}_k$. Each $d_k$ is given to the learner
at round $k$.
And the learner's job is to produce a sequence of posteriors
$(\theta_k)_{k\in\mathbb{N}}$ on $\mathcal{P}(\mathcal{H})$.
To calculate $\theta_k$, learner can use the matrix $\mathbf{L}$,
the initial prior $\theta_0$ 
%\pw{how about add `over hypotheses' if this is the first time $\theta_0$ appeared} 
which is common knowledge,
and the sequence of data $(d_i)_{i\le k}$ which is visible at round $k$.
The learner find each posterior by giving a function
$S_k((d_i)_{i\le k};\mathbf{L},\theta_0)$
\footnote{we may omit $\mathbf{L}$ and (or) $\theta_0$ when there is no ambiguity.}
for $k>0$.
We may further define $S_0(\varnothing;\mathbf{L},\theta_0)=\theta_0$.

Since $(d_k)_{k\in\mathbb{N}}$ is generated by a sequence of random variables $(\mathrm{D}_k)_{k\in\mathbb{N}}$,
the function $S_k$ can be treated as
% (abused to \jw{which is better?})
a function taking $(\mathrm{D}_i)_{i\le k}$ as inputs
and producing a random variable $\Theta_k$ as output. We call the distribution of $\Theta_k$
by $\mu_k\in\mathcal{P}(\mathcal{P}(\mathcal{H}))=\mathcal{P}(\Delta^{m-1})$.
The $S_k$'s as functions of random variables are called \textit{estimators}. %\pw{ask junqi}

Being a special case of the above framework,
Bayesian inference dealing with sequential data is a well-studied model.
However, there is no cooperation in Bayesian inference since
the teaching distribution and learning likelihood are constant on time (the teacher side is typically left implicit).
%`intelligent teacher is not needed there'
To introduce cooperation following cooperative inference \cite{YangYGWVS18},
we propose Sequential Cooperative Bayesian Inference
(SCBI), which is a sequential version of the cooperative inference. 
%And we will use Bayesian inference as control.
% \pw{Maybe motive a bit on sequential and cooperation a bit?}

\subsection{Sequential Cooperative Bayesian Inference}
\label{subsec:CI_intro}
% Cooperative inference was proposed and generalized in \pw{cite}. In this
% section, we provide a brief review. 
% The core of cooperation between two agents is that 
% the teacher's selection of data depends on what the learner is likely to infer
% and vice versa.
Sequential Cooperative Bayesian Inference (SCBI) assumes that the two agents---a
teacher and a learner---cooperate to facilitate 
learning. Prior research has formalized this cooperation (in a single-round
game) as a system of two 
interrelated equations in which the teacher's choice of data depends on the
learner's inference, and the learner's inference depends on reasoning about the
teacher's choice. This prior research into such Cooperative Inference has
focused on batch selection of data \cite{YangYGWVS18,wang2018generalizing}, 
and has been shown to be formally equivalent to Sinkhorn scaling
\cite{wang2019mathematical}. Following this principle, we propose a new
\textit{sequential} setting in 
which the teacher chooses data sequentially, and both agents update the likelihood at each
round to optimize learning. 

\noindent
\textbf{Cooperative Inference.}
%\paragraph{Cooperative Inference}
Let $P_{L_{0}}(h)$ be the learner's prior of hypothesis $h\in \conceptSpace$, 
$P_{T_{0}}(d)$ be the teacher's prior of selecting data $d\in \dataSpace$. Let
$P_{T}(d|h)$ be the teacher's likelihood of selecting $d$ to convey $h$ 
and $P_{L}(h|d)$ be the learner's posterior for $h$ given $d$. %Teaching is similar. 
\textbf{Cooperative inference} is then a system of two equations shown below,
with $\LLmar$ and $\TTmar$ the normalizing constants:
\begin{equation}\label{eq:LT}
 \small
 \LL = \frac{\TT \LLpri}{\LLmar},\ 
 \TT = \frac{\LL \TTpri}{\TTmar}.
\end{equation}
It is shown \citep{wang2018generalizing,wang2019mathematical} that
Eq.~\eqref{eq:LT} can be solved using Sinkhorn scaling, where
$(\rr,\cc)$-\textbf{Sinkhorn scaling} of a matrix $\mathbf{M}$ is simply the
iterated  
alternation of row normalization of $\mathbf{M}$ with respect to $\rr$ and
column normalization of $\mathbf{M}$ with respect to $\cc$. 
The limit of such iterations exist if the sums of elements in $\rr$ and $\cc$
are the same \cite{schneider1989matrix}.  
% We use $\mathbf{M}$ here instead of $\mathbf{M}^{JD}$ since they have the same
% Sinkhorn scaling limit \cite{hershkowitz1988classifications}.

%\pw{add more properties needed later.}
%assumes the agents are independent.
%However, communication is often cooperative among human learning \pw{cite}.
%Recently, many machine learning frameworks, such cooperative inference, capturing such mutual reasoning between agents has been proposed and studied. 
%of cooperative inference,
%Inspired by these theories we propose a new cooperative communication strategy 
%based on Bayesian inference with modifying the likelihood matrix by Sinkhorn
%scaling in each step, and call it \textbf{sequential cooperative Bayesian
%inference} (SCBI). 
%\pw{rewrite}.
%\paragraph{Motivation}
%In cooperative inference, though both agents prepare for the communication with
%possibly infinitely many steps of Sinkhorn iterations, only one element or a batch of data points
%is passed to the learner. However, naturally learning happens sequentially. In practice, 
%it is unlikely to achieve high efficiency in communication with only one round.
%In classical Bayesian inference, though multiple rounds of data are passed, 
%agents operate independently without ever updating their likelihood matrices.
%Leveraging advantages from both aspects, 
%we propose Sequential Cooperative Bayesian Inference as follows.
%, and also there can be opportunities to passing more data to improve the inference.
%\jw{FIX THIS PART}
%\paragraph{Sequential Cooperation}
\noindent
\textbf{Sequential Cooperation.}
SCBI allows multiple rounds of teaching and requires each choice of data to
be generated based on cooperative inference, with the learner updating their beliefs between each round. 
In each round, based on the data being taught
% \pw{maybe, `data have been taught or transmitted'}
and the learner's
initial prior on $\mathcal{H}$ as common knowledge,
the teacher and learner update their common likelihood matrix 
according to cooperative inference (using Sinkhorn scaling),
then the data selection and inference proceed based on the updated likelihood matrix. 
\begin{algorithm}[tb]
   \caption{SCBI, without assumption (iii')}
   \label{alg:scbi}
\begin{algorithmic}
    \STATE {\bfseries == Teacher's Part: ==}
    \STATE {\bfseries Input:} $\mathbf{T}\in\mathrm{Mat}_{n\times m}(\mathbb{R}^+)$,
    $\theta_0$, $h\in\mathcal{H}$, ($\widehat{\theta}=\delta_h$)
    \STATE{\bfseries Output:} Share $(d_1,d_2,\dots)$ to learner
    \FORALL{$i\ge 1$}
    %\STATE $\mathbf{M}_k \gets \mathscr{N}_{\text{col}}\left(\mathbf{M}^{\langle n\theta_{i-1}\rangle}, \mathbf{e}_m\right)$
    \STATE sample $d_i\sim\mathscr{N}_{\text{vec}}\left(\mathbf{T}^{\langle n\theta_{i-1}\rangle}_{\phantom{==}(\_,h)}, 1\right)$.
    \STATE $\theta_i\gets\mathbf{T}^{\langle n\theta_{i-1}\rangle}_{\phantom{==}(d_i,\_)}$\ \  estimation of learner's posterior
    \ENDFOR
    \STATE {\bfseries == Learner's Part: ==}
    \STATE {\bfseries Input:} $\mathbf{L}\in\mathrm{Mat}_{n\times m}(\mathbb{R}^+)$,
    $\theta_0$, $(d_1,d_2,\dots)$
    \STATE{\bfseries Output:} $(\theta_0,\theta_1,\theta_2,\dots)$ posteriors
    \FORALL{$i\ge 1$}
    %\STATE $\mathbf{M}_k \gets \mathscr{N}_{\text{col}}\left(\mathbf{M}^{\langle n\theta_{i-1}\rangle}, \mathbf{e}_m\right)$
    \STATE $\theta_i\gets\mathbf{L}^{\langle n\theta_{i-1}\rangle}_{\phantom{==}(d_i,\_)}$
    \ENDFOR
\end{algorithmic}
\vspace{0.2cm}
{\small Note: $\mathbf{T}^{\langle n\theta_{i-1}\rangle}$, $\mathbf{L}^{\langle n\theta_{i-1}\rangle}$ are the $\mathbf{M}^{\langle \mathbf{c}_{k-1}\rangle}$ in the text.}
\end{algorithm}
% \pw{or say more detail?}
%Only the learner's posterior and teacher's estimation of it that are
%updated in each round.
% \pw{double check the last sentence.}

% and the posterior of the learner in the last round. in each round, 
%This requires further that the teacher knows the learner's prior on
%$\mathcal{H}$ before the first round (usually set to be uniform in practice), so
%that the teacher can figure out the learner's posterior after having taught any
%round. \pw{maybe delete the last sentence, and assume common knowledge somewhere.}
%\pw{improve the following when have time.}

% Precisely, starting from learner's prior $S_0=\theta_0\in\Delta^{m-1}$, in round $k$, 
% % \pw{maybe, denote the learner's prior by $\theta_0$,}
% let the data been taught be $(x_1,\dots,x_{k-1})$ and the posterior of
% the learner after the previous round 
% be
% $\theta_{k-1}=S_{k-1}(x_1,\dots,x_{k-1};\theta_0)\in\mathcal{P}(\mathcal{H})$,
% which is assumed to be predictable for both agents (obvious for $k=1$ and
% inductively correct for $k>1$ by later argument).

Precisely, starting from learner's prior $S_0=\theta_0\in\Delta^{m-1}$,
% \pw{maybe, denote the learner's prior by $\theta_0$,}
let the data been taught up to round $k$ be $(d_1,\dots,d_{k-1})$ and the posterior of
the learner after round $k\!-\!1$
be
$\theta_{k-1}=S_{k-1}(d_1,\dots,d_{k-1};\theta_0)\in\mathcal{P}(\mathcal{H})$, which is actually predictable for both agents (obvious for $k=1$ and
inductively correct for $k>1$ by later argument).
% \pw{.}
% \pw{How about, in round $k$, both agents update their common likelihood matrix by first calculating the SK of $M$ with respect to ..}
To teach, the teacher calculates the Sinkhorn 
scaling of $\mathbf{M}$ given the uniform row sums
$\mathbf{r}_{k-1}=\mathbf{e}_n=(1,1,\dots,1)^\top$ and column sums
$\mathbf{c}_{k-1}=n\theta_{k-1}$ (to make the sum of $\mathbf{r}_{k-1}$ equal
that of $\mathbf{c}_{k-1}$, which guarantees the existence of the limit in
Sinkhorn scaling), denoted by 
$\mathbf{M}^{\left\langle\mathbf{c}_{k-1}\right\rangle}$. 
The teacher's data selection is proportional to columns of
$\mathbf{M}^{\left\langle\mathbf{c}_{k-1}\right\rangle}$. Thus 
let $\mathbf{M}_k$ be the
column normalization of $\mathbf{M}^{\left\langle\mathbf{c}_{k-1}\right\rangle}$ by
$\mathbf{e}_m$, i.e., 
$\mathbf{M}_k=\normalize{col}{
  \mathbf{M}^{\left\langle\mathbf{c}_{k-1}\right\rangle},\mathbf{e}_m}$. 
Then the teacher defines $\mathrm{D}_k$ using
distribution $(\mathbf{M}_k)_{(\_,h)}$ on set $\mathcal{D}$
and samples $d_k\sim\mathrm{D}_k$,
then passes $d_k$ to the learner. %Thus, $P_{k,\widehat{\theta}}$ is defined.

% On learner's side, with datum $d_k$ past from the teacher,
% the learner calculates the likelihood matrix $\mathbf{M}_k$ in the same way and
% applies normal Bayesian inference.
On learner's side, the learner obtains the likelihood matrix $\mathbf{M}_k$ in the same way as above and
applies normal Bayesian inference with datum $d_k$ past from the teacher.
First, learner takes the prior to be the
posterior of the last round, $\theta_{k-1}=\frac{1}{n}\mathbf{c}_{k-1}$, then
multiply it by the likelihood of selecting $d_k$ --- 
the row of $\mathbf{M}_k$ corresponding to $d_k$, which results
$\mathring{\theta}_k=(\mathbf{M}_k)_{(d_k,\_)}\mathrm{diag}(\theta_{k-1})$. 
Then the posterior $\theta_k$ is obtained by row normalizing $\mathring{\theta}_k$. 
Inductively, in the next round, the learner will start with $\theta_k$ and
$\mathbf{c}_k=n\theta_k$.
The learner's calculation in round $k$ can be simulated by the teacher, so the
teacher can predict $\theta_k$, which inductively shows the assumption (teacher knows $\theta_{k-1}$) in
previous paragraph.

% The calculation of $\theta_k$ and $\mathbf{c}_k$ can be simplified. 
The calculation can be simplified. Consider
that the 
vector $\mathbf{c}_{k-1}$, being proportional to the prior, is used in
$\mathbf{M}_k$
$=\mathscr{N}_{\text{col}}(\mathbf{M}^{\left\langle\mathbf{c}_{k-1}\right\rangle},
\mathbf{e}_m)=$ $\mathbf{M}^{\left\langle\mathbf{c}_{k-1}\right\rangle}
\left(\mathrm{diag}(n\theta_{k-1})\right)^{-1}$, then
% \pw{use this notation in the previous teacher's paragraph where $\mathbf{M}_k$ is defined?}, then
$\mathring{\theta}_k$ $=$ $\left(
  \mathbf{M}^{\left\langle\mathbf{c}_{k-1}\right\rangle} 
  \left(\mathrm{diag}(n\theta_{k-1})\right)^{-1} 
  \mathrm{diag}(\theta_{k-1})
\right)_{(d_k,\_)}$ $=$ $
\frac{1}{n}\mathbf{M}^{\left\langle\mathbf{c}_{k-1}\right\rangle}_{\phantom{===}(d_k,\_)}$.  
Furthermore, since $\mathbf{M}^{\left\langle\mathbf{c}_{k-1}\right\rangle}$ is
row normalized to $\mathbf{e}_m$, each row of it is a  probability distribution on
$\mathcal{H}$. Thus
$S_k(d_1,\dots,d_k)=\theta_k=n\mathring{\theta}_{k-1}=
\mathbf{M}^{\left\langle\mathbf{c}_{k}\right\rangle}_{\phantom{=}{(d_k,\_)}}$.
%\pw{improve the above three paragraphs when have time.}
\footnote{See Supplementary Material for detailed examples. }

The simplified version of SCBI algorithm is given in
Algorithm~\ref{alg:scbi}.

%\pw{one example for SK, one example for SCBI.}
% \pw{put an example here to illustrate how SCBI works? }

\subsection{Bayesian Inference: the Control}
In order to test the performance of SCBI, we recall the classical Bayesian
inference (BI). In BI, a fixed likelihood matrix $\mathbf{M}$ is used throughout
the communication process. Bayes' rule requires $\mathbf{M}$ to be the
conditional distribution on the set of data given 
each hypothesis, thus $\mathbf{M}=\mathbf{T}=\mathbf{L}$ 
is column-stochastic as in assumption (iii').
% joint distribution matrix, i.e.,
% $\normalize{col}{\mathbf{M}^{JD},\mathbf{e}_m}$ defined before.
%\mathbf{M}^{JD} D$ where
%$D=\mathrm{diag}\left(\frac{1}{\sum_{i=1}^{n}M_{(i,1)}}, \dots,
%  \frac{1}{\sum_{i=1}^{n}M_{(i,m)}}\right)$.  \pw{If $D$ is not used later, may delete.}
% \pw{double check}

For the teacher, given $h\in\mathcal{H}$, % treated
% $\widehat{\theta}$ 
% as a column vector, 
the teaching distribution is the column vector % \pw{name it?}
$P_{h}=\mathbf{M}_{(\_,h)}\in\mathcal{P}(\mathcal{D})$.
This defines random variable $\mathrm{D}_k$.
Then the teacher selects data via i.i.d. sampling 
according to $P_{h}$. The random variables
$(\mathrm{D}_k)_{k\ge 1}$ are identical.

% a sequence of form $(x_1, \dots, x_k)$ is 
% $P_{k,\widehat{\theta}}(\{(x_1,\dots,x_k)\})=\prod_{i=1}^{k}P_{\widehat{\theta}}(x_i)$.

The learner first chooses a prior $\theta_0\in\mathcal{P}(\mathcal{H})$
($\theta_0=S_0$ is part of the model, usually the uniform distribution), 
% \pw{hyperparameter?}
then uses Bayes' rule
with likelihood $\mathbf{M}$ to update the posterior distribution repeatedly.
%Recall that Bayes' rule  \pw{may delete}
%$$P(h|d)=\frac{P(d|h)P(h)}{\sum_{i}P(d|i)P(i)}$$
Given taught datum $d$, the map from the prior $\theta$ to the posterior distribution is denoted by $B_d(\theta)=\normalize{vec}{\mathbf{M}_{(d,\_)}\mathrm{diag}(\theta),1}$.
%\pw{learner's prior is $\pi$ or $\theta$?}
Thus the learner's estimation over $\mathcal{H}$ given a sequential data $(d_1,
\dots, d_k)$
can be written recursively by $S_0=\theta_0$, and
$S_k(d_1,\dots,d_k)=B_{d_k}(S_{k-1}(d_1,\dots,d_{k-1}))$.
% $S_k(A|(x_1,\dots,x_k))=S_{k-1}(B_{x_k}^{-1}(A)|(x_1,\dots,x_{k-1}))$.
Thus, by induction, $S_k(d_1,\dots,d_k)=(B_{d_k}\circ B_{d_{k-1}}\circ\dots\circ B_{d_1})(S_0)$.

% Basic setup: Classical BI and BI by sequential Sinkhorn

\section{Consistency}\label{sec:consistency}
We investigate the effectiveness of the estimators 
in both BI and SCBI by testing their \textit{consistency}:
setting the true hypothesis 
$h\in\mathcal{H}$,
given $(\mathrm{D}_k)$, $(S_k)$ and $\theta_0$,
% in a model $(\mathfrak{M},(S_k),\theta_0)$, 
we
examine the convergence (using the $L^1$-distance
on $\mathcal{P}(\mathcal{H})$) of the posterior sequence
$(\Theta_k)=(S_k(\mathrm{D}_1,\dots,\mathrm{D}_k))$ 
as sequence of random variables and 
check whether the limit is $\widehat{\theta}$ as a constant
random variable.
% confirm whether the the true
% distribution $\widehat{\theta}$ is the limit.

% Although the definition of consistency is with a specified true distribution 
% $\widehat{\theta}\in\mathcal{P}(\mathcal{H})$, here, we focus
% on atomic cases where $\widehat{\theta}=\delta_h$ for some
% $h\in\mathcal{H}$, meaning the truth is a fixed $h$ in $\mathcal{H}$ rather than
% a general distribution on $\mathcal{H}$. 

%\pw{target? something else was used when $\widehat{\theta}$ was introduced} 
% \pw{Directly say that we will focus on the case where $\widehat{\theta}=\delta_h$ for the rest of this paper?}
% \pw{also mention that the this means the truth is a fixed h in H, rather than a distribution over H?}

% there would be discussions and conclusions here in each case more general than
% that. 
% Here, to include the most general case, $\widehat{\theta}$ is assumed to 
% be an arbitrary distribution over $\mathcal{H}$.
% In practice, often the teacher is given a single true hypothesis, thus the
% $\widehat{\theta}$ is of the form $\delta_i\in\mathcal{P}(\mathcal{H})$. 
% This is the only case we will test by simulation in the next section.
% \pw{maybe encourage the reader to keep the simplified $\widehat{\theta}$ as a
% concrete example in mind?} 

\subsection{BI and KL Divergence}

%\paragraph{Consistency}
The consistency of BI has been well studied since Bernstein and von
Mises and Doob \cite{doob1949application}. %in many literature \pw{cite}.
%\jw{Should we mention Doob?} \pw{add some citations here?}
In this section, we state it in our situation and derive a formula for the rate
of convergence, as a baseline for the cooperative theory. 
Derivations and proofs can be found in the Supplementary Material.

\begin{thm}\label{thm:BI_consistency}[\citep[Theorem 7.115]{miescke2008decision}]
  In BI, the sequence of posteriors $(S_k)$ is strongly consistent at
  $\widehat{\theta}=\delta_h$ for each $h\in\mathcal{H}$, with arbitrary choice
  of an interior point $\theta_0\in(\mathcal{P}(\mathcal{H}))^\circ$
  (i.e. $\theta_0(h)>0$ for all $h\in\mathcal{H}$) as prior.
\end{thm}
% \pw{maybe describe $(\mathcal{P}(\mathcal{H}))^\circ$ first, then introduce the notation.}
% \pat{agree!}

\begin{remark}\label{rmk:consistency}
  For a fixed true distribution $\widehat{\theta}$, \textit{strong consistency} of $(S_k)_{k\in\mathbb{N}}$ is defined to be: the
  sequence of posteriors $\Theta_k$  
  given by the estimator $S_k$, as a sequence of random variables, converges to $\widehat{\theta}$ (as a constant random variable)
  almost surely according to 
  random variables $(\mathrm{D}_k)_{k\in\mathbb{N}}$ that the 
  teacher samples from.
  % distribution $P_{\widehat{\theta}}$. 
  If the
  convergence is in probability, the sequence of estimators is said to be
  \textit{consistent}. % \pw{remind people what is $P_{\widehat{\theta}}$?}
\end{remark}

\begin{remark}
  Theorem~\ref{thm:BI_consistency} also assumes that hypotheses are
  distinguishable (Section~\ref{sec:scbi_construction}). 
  In a general theory of statistical models,
  $\widehat{\theta}$ is not necessarily $\delta_h$ for some $h\in \mathcal{H}$.
  However, in BI, it is critical to have $\widehat{\theta}=\delta_h$, %for some $h\in \mathcal{H}$ here, %\pw{$h \in \mathcal{H}$},
  since BI with a general $\widehat{\theta}\in\mathcal{P}(\mathcal{H})$ is almost 
  never consistent or strongly consistent.
%   If further we assume that the columns of $\mathbf{M}^{JD}$ are linearly
%   independent, the strong consistency holds everywhere on
%   $\mathcal{P}(\mathcal{H})=\Delta^{m-1}$. 
  % \pw{sometimes $\mathcal{P}(\mathcal{H})$ is used, sometimes $\Delta^{m-1}$ is used.}
\end{remark}

Consistency---independent of the choice of prior $\theta_0$
%  \pw{use $\theta_0$ as in the statement of thm?} if it is in the
interior of $\mathcal{P}(\mathcal{H})$---guarantees that BI is always
effective.

%\paragraph{Rate of Convergence}
\noindent
\textbf{Rate of Convergence.}
After effectiveness, we provide the efficiency of BI in terms of asymptotic rate of convergence.

\begin{thm}
  \label{thm:roc_bi}
  In BI, with $\widehat{\theta}=\delta_h$ for some $h\in\mathcal{H}$,
  let $\Theta_k(h)(\mathrm{D}_1,\dots,\mathrm{D}_k):=S_k(h|\mathrm{D}_1,\dots,\mathrm{D}_k)$ be the $h$-component of
  posterior given $\mathrm{D}_1,\dots,\mathrm{D}_k$ as random variables valued in $\mathcal{D}$. Then
  $\dfrac{1}{k}\log\left(\dfrac{\Theta_{k}(h)}{1-\Theta_{k}(h)}\right)$ converges to a constant
  $\mathrm{min}_{h'\ne h}\left\{\mathrm{KL}(\Mcol{M}{h},\Mcol{M}{h'})\right\}$
  almost surely.
%  the rate
%  of convergence is determined by $\mathrm{min}_{h'\ne
%    h}\left\{\mathrm{KL(\Mcol{M}{h},\Mcol{M}{h'})}\right\}$. \jw{More precise} 
\end{thm}

% \pw{maybe remind reader what is $M$ in the statement and emphasize
% $\theta_k(h)$ is the posterior on the true hypothesis after k steps of
% teaching?} 

\begin{remark}
  We call
  $\mathrm{min}_{h'\ne
    h}\left\{\mathrm{KL}(\Mcol{M}{h},\Mcol{M}{h'})\right\}$
  the \textit{asymptotic rate of convergence} (RoC) of BI,
  denoted by $\mathfrak{R}^{\text{b}}(\mathbf{M};h)$.
\end{remark}

\subsection{SCBI as a Markov Chain}
From the proof of Theorem~\ref{thm:BI_consistency}, the pivotal property is that the variables
$\mathrm{D}_1,\mathrm{D}_2,\dots$ are commutative in posteriors (the variables can occur in any
order without affecting the posterior) thanks to commutativity of
multiplication. 
% \pw{-- readers may not know this if they do not read the Supp. Maybe directly
% say that the pivotal property in the proof is ...} 
However, in SCBI, the commutativity does not hold, since the likelihood matrix depends
on previous outcome. Thus the method used in BI analysis no longer works here.
% \pw{Directly say there is a difference if order varies? "lost" gives negative feelings...}

%The analysis of SCBI is on a new track. 
Because the likelihood matrix
$\mathbf{M}_k=\mathbf{M}^{\left\langle\mathbf{c}_{k-1}\right\rangle}$ depends
on the predecessive state only, the process is in fact Markov, we may
analyze the model as a Markov chain on the continuous state space
$\mathcal{P}(\mathcal{H})$.

To describe this process, let $\mathcal{P}(\mathcal{H})=\Delta^{m-1}$ be the
space of states, and let
$h\in\mathcal{H}$ be the true hypothesis to teach
($\widehat{\theta}=\delta_h$), let learner's prior be
$S_0=\theta_0\in\mathcal{P}(\mathcal{H})$, or say, the distribution of learner's
initial state 
is $\mu_0=\delta_{\theta_0}\in\mathcal{P}(\mathcal{P}(\mathcal{H}))$.
% \pw{check}

%\paragraph{The operator $\Psi$.}
\noindent
\textbf{The operator $\Psi$.} 
In the Markov chain, in each round, the transition operator maps the prior as a probability
distribution on state space
$\mathcal{P}(\mathcal{H})=\Delta^{m-1}$ to the posterior as another, i.e.,
$\Psi(h):\mathcal{P(\mathcal{P}(\mathcal{H}))}\rightarrow\mathcal{P(\mathcal{P}(\mathcal{H}))}$.
% \pw{Here `prior' is no longer the learner's initial prior on $\mathcal{H}$,
% maybe worth to clarify a bit.} 
%\pw{check if $\mathcal{P(\mathcal{P}(\mathcal{H}))}$ is still needed?}

To make the formal definition of $\Psi(h)$ simpler, we need to define some maps.
For any $d\in\mathcal{D}$, 
let $T_d:\Delta^{m-1}\rightarrow\Delta^{m-1}$ be the map bringing the learner's
prior to posterior when data $d$ is chosen by the teacher, that is, $T_d$ sends
each normalized vector $\theta$ to $T_d(\theta)=\mathbf{M}^{\left\langle
    n\theta\right\rangle}_{\phantom{~~}(d,\_)}$ according to SCBI. 
   % \pw{subindex `i' shall be `d'?}
Each $T_d$ is a bijection based on the uniqueness of Sinkhorn scaling limits of
$\mathbf{M}$, shown in \cite{hershkowitz1988classifications}. Further, the map
$T_d$ is continuous on $\Delta^{m-1}$ and smooth in its interior according to
\cite{wang2019mathematical}.
Continuity and smoothness of $T_d$ make it natural to induce a push-forward
$T_{d\ast}:\mathcal{P}(\Delta^{m-1})\rightarrow\mathcal{P}(\Delta^{m-1})$ on
Borel measures. Explicitly, $(T_{d\ast}(\mu))(E)=\mu(T^{-1}_d(E))$ for each Borel
measure $\mu\in\mathcal{P}(\Delta^{m-1})$ and each Borel measurable set
$E\subseteq\Delta^{m-1}$.
Let $\tau:\mathcal{P}(\mathcal{H})\rightarrow\mathcal{P}(\mathcal{D})$ be the
map 
of teacher's adjusting sample distribution based on the learner's prior, that is, given a 
learner's prior $\theta\in\Delta^{m-1}$, by definition of SCBI, the
distribution of the teacher is adjusted to
$\tau(\theta)=\frac{\mathbf{M}^{\left\langle
      n\theta\right\rangle}_{\phantom{~~}(\_,h)}}{n\theta(h)}=
(n\theta(h))^{-1}(T_1(\theta)(h),\dots,T_n(\theta)(h))$. Each component $d$
of $\tau$ is denoted by $\tau_d$. We can use $\tau$ % is useful
only for $\theta_0=\delta_{h}$ in which case teacher can trace learner's state.
Now we can define $\Psi(h)$ formally.
\begin{definition}
  \label{def:the_operator}
  Given a fixed hypothesis $h\in\mathcal{H}$, or say
  $\delta_h\in\mathcal{P}(\mathcal{H})$, the 
  operator
  $\Psi(h):\mathcal{P}(\Delta^{m-1})\rightarrow\mathcal{P}(\Delta^{m-1})$
  translating a prior as a Borel measure
  $\mu$ % $\mu\in\mathcal{P}(\Delta^{m-1})$
  to the posterior distribution $\Psi(h)(\mu)$
  % in $\mathcal{P}(\Delta^{m-1})$
  according to one round of SCBI
  is given below, for any Borel measurable set $E\subset \Delta^{m-1}$.
  % \vspace{-0.1cm}
  \begin{equation}\small
    \label{eq:the_operator}
    \left(\Psi(h)(\mu)\right)(E):=\displaystyle\int\limits_E\sum_{d\in\mathcal{D}}
    \tau_d(T_d^{-1}(\theta))\mathrm{d}\left(T_{d\ast}(\mu)\right)(\theta).
  \vspace{-0.2cm}
  \end{equation}
  %for any Borel measurable set $E\subset \Delta^{m-1}$.
\end{definition}
\vspace{-0.1cm}
% \jw{We just need atomic case. Explain why only atomic.}
% In practice, we calculate the operator using particular forms of $\Psi$ in
% different situations.
% For $\mu\in\mathcal{P}(\Delta^{m-1})$ a measure on compact Polish space
% $\Delta^{m-1}$, according to Lebesgue decomposition theorem \citep{halmos1974measure},
% $\mu=\mu_a+\mu_s+\sum_{i=1}^{\infty}a_i\delta_{\mathbf{v_i}}$
% with $a_i>0$ for each $i$, $\sum_{i=1}^{\infty}a_i\le1$, $\mu_s\perp m$, and
% $\mu_a<\!\!<m$ where $m$ is the Lebesgue measure on $\Delta^{m-1}$. 
% Then the behavior of $\Psi(h)$ on $\mu$ is determined by its behaviors on
% $\delta_{\mathbf{v}}$, on $\mu_s$, and on $\widetilde{\mu}$.
% \jw{delete the current paragraph}
In our case, we start with a distribution $\delta_{\theta}$ where
$\theta\in\mathcal{P}(\mathcal{H})$ is the prior of the learner on the set of
hypotheses. In each round of inference, there are $n$ different
possibilities according to the data taught.
% that the teacher can teach. 
Thus in any
finite round $k$, the distribution of the posterior is the sum of at most $n^k$
atoms (actually, we can prove $n^k$ is exact). Thus in the following
discussions, we assume that $\mu$ is atomic. The $\Psi$ action on an atomic
distribution is determined by
% \pw{repeated?}
% The distribution $\delta_{\theta}$ means the prior is $\theta$ almost
% surely, thus the posterior has $n$ different situations according to the choice
% of the data, distributed according to the sampling distribution
% %$\mathscr{N}(\mathbf{M^{\left\langle\mathbf{v}\right\rangle}_{\phantom{~~}(\_,h)}})=
% %(\mathbf{v}_h)^{-1}\mathbf{M^{\left\langle\mathbf{v}\right\rangle}_{\phantom{~~}(\_,h)}}$,
% $\tau$
% and in each case $d$ the posterior will become
% $\mathbf{M}^{\left\langle n\theta\right\rangle}_{\phantom{~~}(d,\_)}$. 
% Thus the posterior is
that of an atom:
\vspace{-0.1cm}
\begin{equation}\small
  \label{eq:op_psi_delta}
  \Psi(h)(\delta_{\theta}) = 
  \sum_{i=1}^n\dfrac{\mathbf{M}^{\left\langle n\theta\right\rangle}_{\phantom{~~}(i,h)}}
  {n\theta(h)}\delta_{\left(\mathbf{M}^{\left\langle n\theta\right\rangle}_{\phantom{~~}(i,\_)}\right)}.
\end{equation}
Moreover, since the SCBI behavior depends only on the prior (with fixed
$\mathbf{M}$ and $h$) as a random variable, the same operator
$\Psi(h)$ applies to every round in 
SCBI. Thus we can conclude that the following proposition is valid:

\begin{prop}\label{prop:markov_property}
  Given 
  % a fixed hypothesis 
  $h\in\mathcal{H}$, let $\widehat{\theta}=\delta_h$,
  the sequence of estimators $(S_k)_{k\in\mathbb{N}}$ in SCBI forms a time-homogeneous Markov
  chain on state space $\mathcal{P}(\mathcal{H})$ with transition operator
  $\Psi(h)$ characterized by Eq.~\eqref{eq:the_operator} and
  Eq.~\eqref{eq:op_psi_delta}.
\end{prop}

%\paragraph{Consistency}

Thanks to the fact that the SCBI is a time homogeneous Markov process, we can
further show the consistency.

% show the consistency of SCBI by analyzing the operator $\Psi(h)$.
% \begin{lemma}\label{lem:expectation_inequality}
%   Given a fixed hypothesis $h\in\mathcal{H}$, for any $\mu\in\mathcal{P}(\Delta^{m-1})$,
%   \begin{equation}
%     \label{eq:e_mu_smaller_e_psi_mu}
%     \mathbb{E}_{\mu}(\mathbf{v}_h)\le\mathbb{E}_{\Psi(h)(\mu)}(\mathbf{v}_h).
%   \end{equation}
%   equality happens when
%   $\mathbf{M}^{\left\langle\mathbf{x}\right\rangle}_{\phantom{~~}(i,h)}
%   =\mathbf{M}^{\left\langle\mathbf{x}\right\rangle}_{\phantom{~~}(j,h)}$ for any
%   $i$, $j$ and $\mu$-almost everywhere for $x\in\Delta^{m-1}$.
% \end{lemma}
% \jw{This lemma is not as important as before, should we add new ones or delete
%   this one?}

% \begin{remark}\label{rmk:expectation_inequality}
%   Moreover, for a positive matrix $\mathbf{M}$,
%   the equality in Eq.~\eqref{eq:e_mu_smaller_e_psi_mu} holds if and only if the
%   rows of $\mathbf{M}$ are identical. Besides, if $\mathbf{M}$ is non-negative,
%   then equality can happen when the $h$-th column of $\mathbf{M}$ is zero.
% \end{remark}

% Lemma~\ref{lem:expectation_inequality} guarantees that for almost every prior
% $\mathbf{v}$, every time $\Psi(h)$ acts on $\mathbf{v}$, it pulls the prior
% towards the true parameter $\widehat{\theta}=\delta_h$.

\begin{thm}[Consistency]\label{thm:consistency_SCBI}
  In SCBI, let $\mathbf{M}$ be a positive matrix. If the teacher is teaching
  one hypothesis $h$ (i.e.,
  $\widehat{\theta}=\delta_h\in\mathcal{P}(\mathcal{H})$), and the prior
  distribution 
  $\mu_0\in\mathcal{P}(\Delta^{m-1})$ satisfies
  $\mu_0=\delta_{\theta_0}$ with $\theta_0(h)>0$, then the estimator sequence $(S_k)$ is
  consistent, for each $h\in\mathcal{H}$, i.e., 
  the posterior random variables $(\Theta_k)_{k\in\mathbb{N}}$ 
  converge to the constant random variable $\widehat{\theta}$ in probability.
\end{thm}

\begin{remark}\label{rmk:consistency_SCBI}
  The assumption in Theorem~\ref{thm:consistency_SCBI} that 
  $\theta_0(h)>0$ is necessary in any type of
  Bayesian inference 
  since it is impossible to get the correct answer in posterior by Bayes' rule,
  if it is excluded in the prior at the beginning.
  % We do not allow the prior distribution to be of nonzero measure on the subset
  % of $\Delta^{m-1}$ consisting of those priors whose probability on the correct
  % hypothesis is zero, which may make the teacher unable to teach.
  % \jw{What am I talking about???}
  In practice, the prior distribution is usually chosen to be
  $\mu_0=\delta_{\mathbf{u}}$ with the uniform distribution vector in
  $\mathcal{P}(\mathcal{H})$, i.e.,
  $\mathbf{u}=\frac{1}{m}(1,\dots,1)^\top\in\Delta^{m-1}$. 
\end{remark}

%\paragraph{Rate of Convergence} 
\noindent
\textbf{Rate of Convergence.}
Thanks to consistency, we can calculate the
asymptotic rate of convergence for SCBI.

\begin{thm}
  \label{thm:roc_scbi}
  With matrix $\mathbf{M}$, hypothesis $h\in\mathcal{H}$, and a 
  % deterministic
  prior $\mu_0=\delta_{\theta_0}\in\mathcal{P}(\Delta^{m-1})$ same as in Theorem.~\ref{thm:consistency_SCBI},
  let $\theta_k$ denote a sample value of the posterior
  $\Theta_k$ after $k$ rounds of SCBI, then 
  \begin{equation}\small
      \label{sm:eq:roc_SCBI_converge}
      \lim_{k\rightarrow\infty}\mathbb{E}_{\mu_k}\left[
        \dfrac{1}{k}\log\left(\dfrac{\theta_{k}(h)}{1-\theta_{k}(h)}\right)
      \right]=\mathfrak{R}^{\mathrm{s}}(\mathbf{M};h)
  \end{equation}
  where $\mathfrak{R}^{\mathrm{s}}(\mathbf{M};h):=\min_{h\ne
    h'}\mathrm{KL}\left(\mathbf{M}^\sharp_{(\_,h)},\mathbf{M}^\sharp_{(\_,h')}\right)$ 
  % $\mathbb{E}_{\mu_k}\left[\dfrac{1}{k}\log\left(\dfrac{\theta_{k}(h)}{1-\theta_{k}(h)}\right)\right]$ converges
  % to  $\min\limits_{h\ne
   % h'}\mathrm{KL}\left(\mathbf{M}^\sharp_{(\_,h)},\mathbf{M}^\sharp_{(\_,h')}\right)$
  % where 
  with
  $\mathbf{M}^\sharp_{\phantom{|}}=\mathscr{N}_{\text{col}}(\mathrm{diag}(\mathbf{M}_{(\_,h)})^{-1}\mathbf{M})$.
  Thus we call % $\min\limits_{h\ne h'}\mathrm{KL}\left(\mathbf{M}^\sharp_{(\_,h)},
  %    \mathbf{M}^\sharp_{(\_,h')}\right)$
  $\mathfrak{R}^{\mathrm{s}}(\mathbf{M};h)$
  the asymptotic rate of convergence (RoC) of SCBI.
\end{thm}

%%%%%%%%%%%%%%%%%%%%%%%%%%%%%%%%%%%%%%%%%%%%%%%%%%%%%%%%%%%%%
% IMPORTANT NOTES
% SECTION 3 has a push-forward of T_d with notation opposite to convention, but already used too much in SI, please check it out for final.

% Consistency problem

\section{Sample Efficiency}
\label{sec:sample_efficiency}

In this section, we present some empirical results
comparing the sample efficiency of SCBI and BI.

\subsection{Asymptotic RoC Comparison} \label{sec:asy_roc_comp}

We first compare the asymptotic rate of convergence ($\mathfrak{R}^{\text{b}}$
for BI and $\mathfrak{R}^{\text{s}}$ for SCBI, see Theorems~\ref{thm:roc_bi} and
\ref{thm:roc_scbi}). 
The matrix $\mathbf{M}$ is sampled through 
$m$ i.i.d. uniform distributions on $\Delta^{n-1}$, one for each column.
% , equivalently, after column-normalization, each column is
% sampled in $\Delta^{n-1}$ uniformly and independently, and assume the true
% hypothesis in each inference is taken uniformly in $\mathcal{H}$. 
% \pw{a little redundant?} 
% Two variables
% are considered in the comparison of RoC:
% the probability of the averaged RoC of SCBI is greater than that of BI, and the expectation of the difference of the two RoC values.

% Precisely, for a column-normalized $n\times m$ matrix $\mathbf{M}$, the two
For each column-normalized matrix $\mathbf{M}$,
we compute two variables to compare BI with SCBI: the probability 
$\mathfrak{P}:=\mathrm{Pr}\left(\frac{1}{m}\sum_{h\in\mathcal{H}}\mathfrak{R}^{\text{s}}(\mathbf{M};h)\ge 
  \frac{1}{m}\sum_{h\in\mathcal{H}}\mathfrak{R}^{\text{b}}(\mathbf{M};h)\right)$
and the expected value of averaged difference
$\mathfrak{E}:=\mathbb{E}\left[ 
    \frac{1}{m}\sum_{h\in\mathcal{H}}\mathfrak{R}^{\text{s}}(\mathbf{M};h)-
    \frac{1}{m}\sum_{h\in\mathcal{H}}\mathfrak{R}^{\text{b}}(\mathbf{M};h)\right]$.

% It is easy to see they are the same as those in fixed-hypothesis
% cases. \pw{meaning?}
% Assume the joint distribution $\mathbf{M}^{JD}$ is
% sampled uniformly (\jw{Check with Pei whether this is our case.}),
% two variables are considered in sample efficiency: the probability of the
% theoretical rate of convergence of SCBI is greater than that of BI, and the
% expected value of the difference (with sign) of the two rates of convergence.

% In the uniform case, we can equivalently calculate the 

%\paragraph{Two-column Cases}
\noindent
\textbf{Two-column Cases.}
Consider the case where $\mathbf{M}$ is of shape $n\times2$
with the two columns sampled from $\Delta^{n-1}$ uniformly 
and independently, we simulated for $n=2,3,\dots,50$ with a size-$10^{10}$
Monte Carlo method for each $n$ to calculate $\mathfrak{P}$ and $\mathfrak{E}$.
The result is shown in Fig.~\ref{fig:roc}(A)(B).

We can reduce the calculation of $\mathfrak{E}$ to a numerical % high-order
integral $\mathfrak{E} = 
\int_{(\Delta^{n\!-\!1})^2}\ln\left(\sum_{i=1}^{n}\frac{\mathbf{x}_i}{\mathbf{y}_i}\right)
\mathrm{d}\mathbf{x}\mathrm{d}\mathbf{y} 
-\ln n-\frac{n-1}{n}$.
\footnote{Details can be found in Supplementary Material.}
% \jw{Do we want to have $\frac{1}{2}\ln(x(x+1)/(x-1.5))+0.1x-0.3$ here?}
% \pat{this is the function you guessed? maybe a footnote?}
% \pw{maybe motivate a bit why graph $-\ln(1-\mathfrak{P})$ and provide some indication from the pictures.}

Since $\mathfrak{P}$ goes too close to $1$ as the rank grows, we 
use $-\ln(1-\mathfrak{P})$ to show the increasing in detail.
\footnote{We guess an empirical formula  $-\ln(1-\mathfrak{P})\approx\frac{1}{2}\ln(x(x+1)/(x-1.5))+0.1x-0.3$, see Supplementary Material.}
% \vspace{-2mm}

\begin{figure*}[t]
  \centering
  %\raisebox{3cm}{\textbf{a}}
  \includegraphics[scale=0.5]{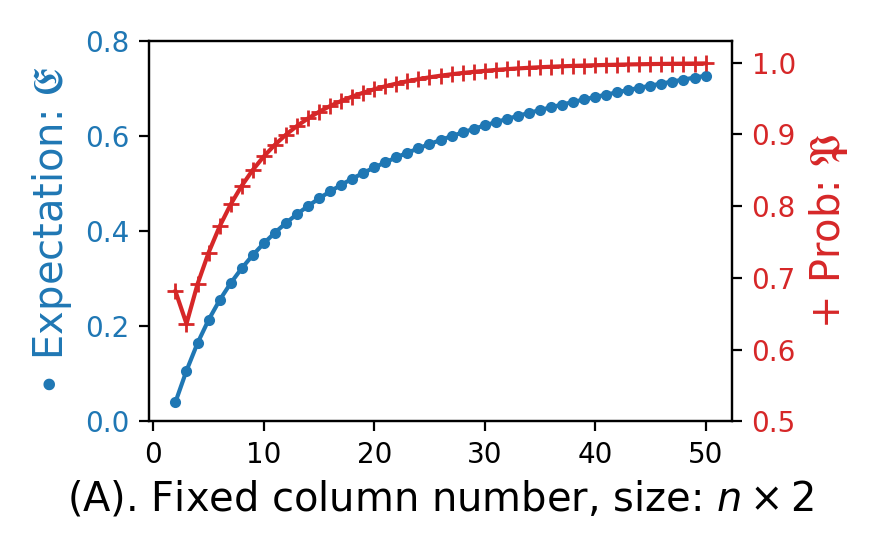}
  %\raisebox{3cm}{\textbf{c}}
  \includegraphics[scale=0.5]{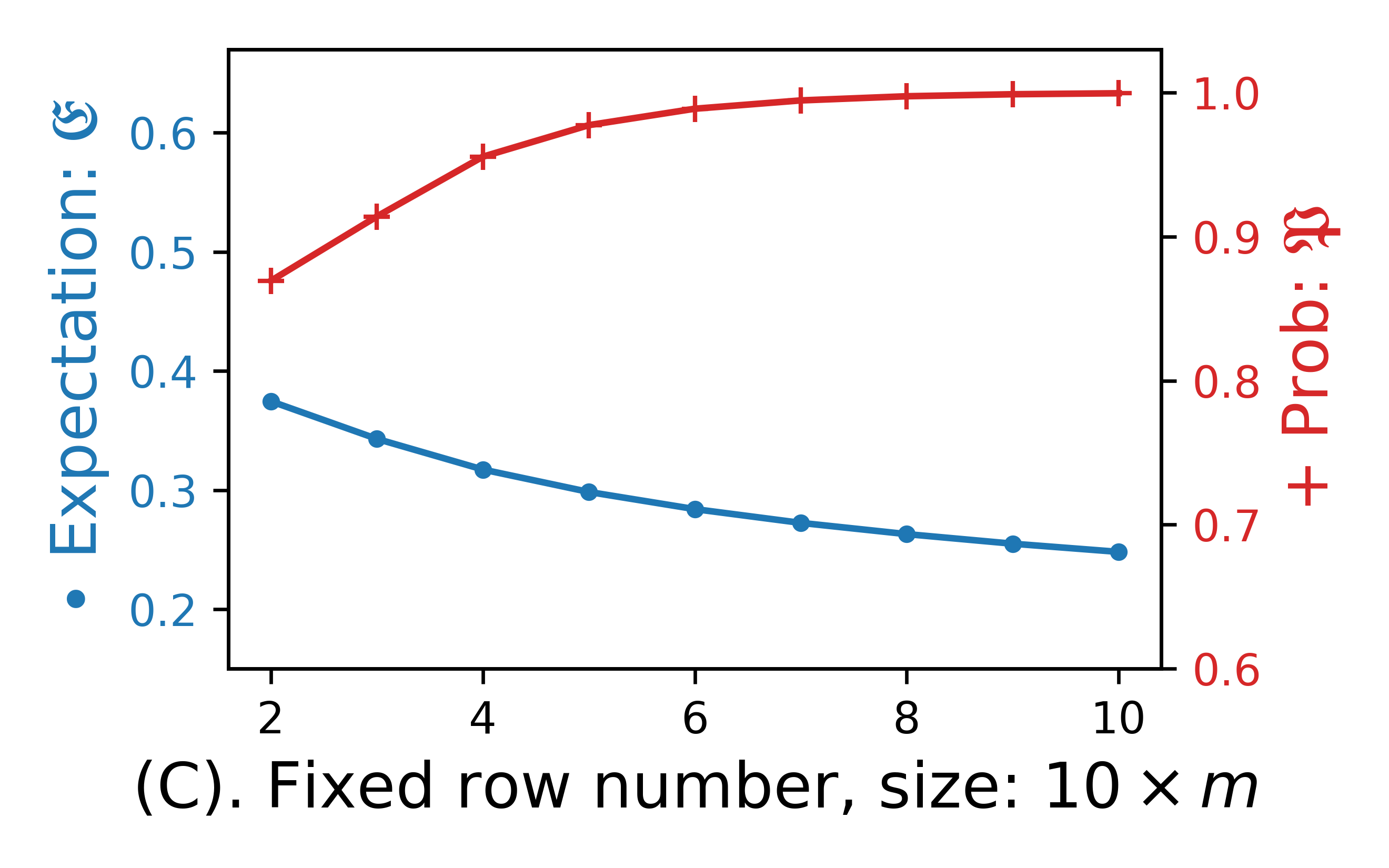}
  %\raisebox{3cm}{\textbf{e}}
  \includegraphics[scale=0.51]{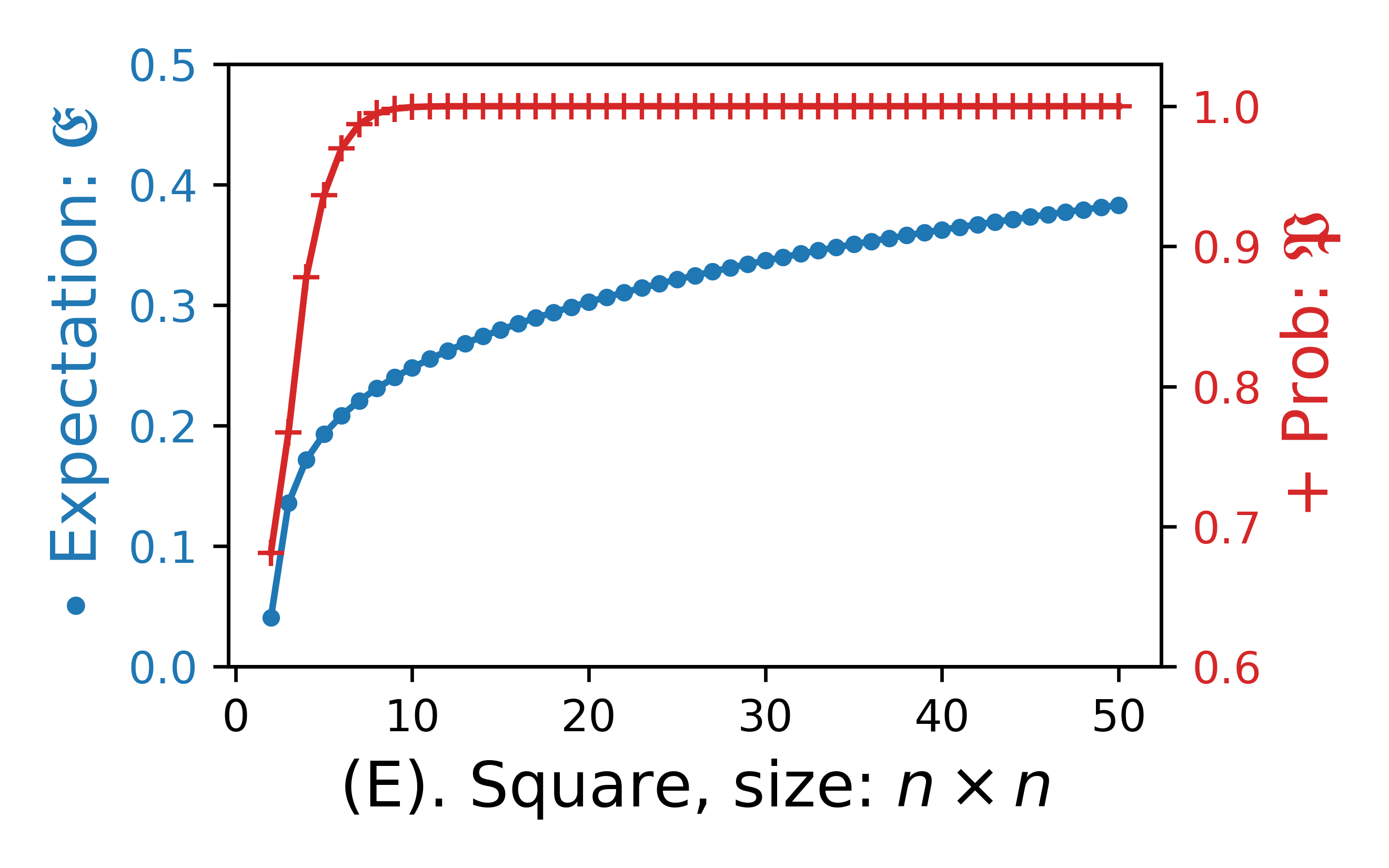}
  %\raisebox{2.8cm}{\textbf{b}}
  \raisebox{0.1cm}{\includegraphics[scale=0.5]{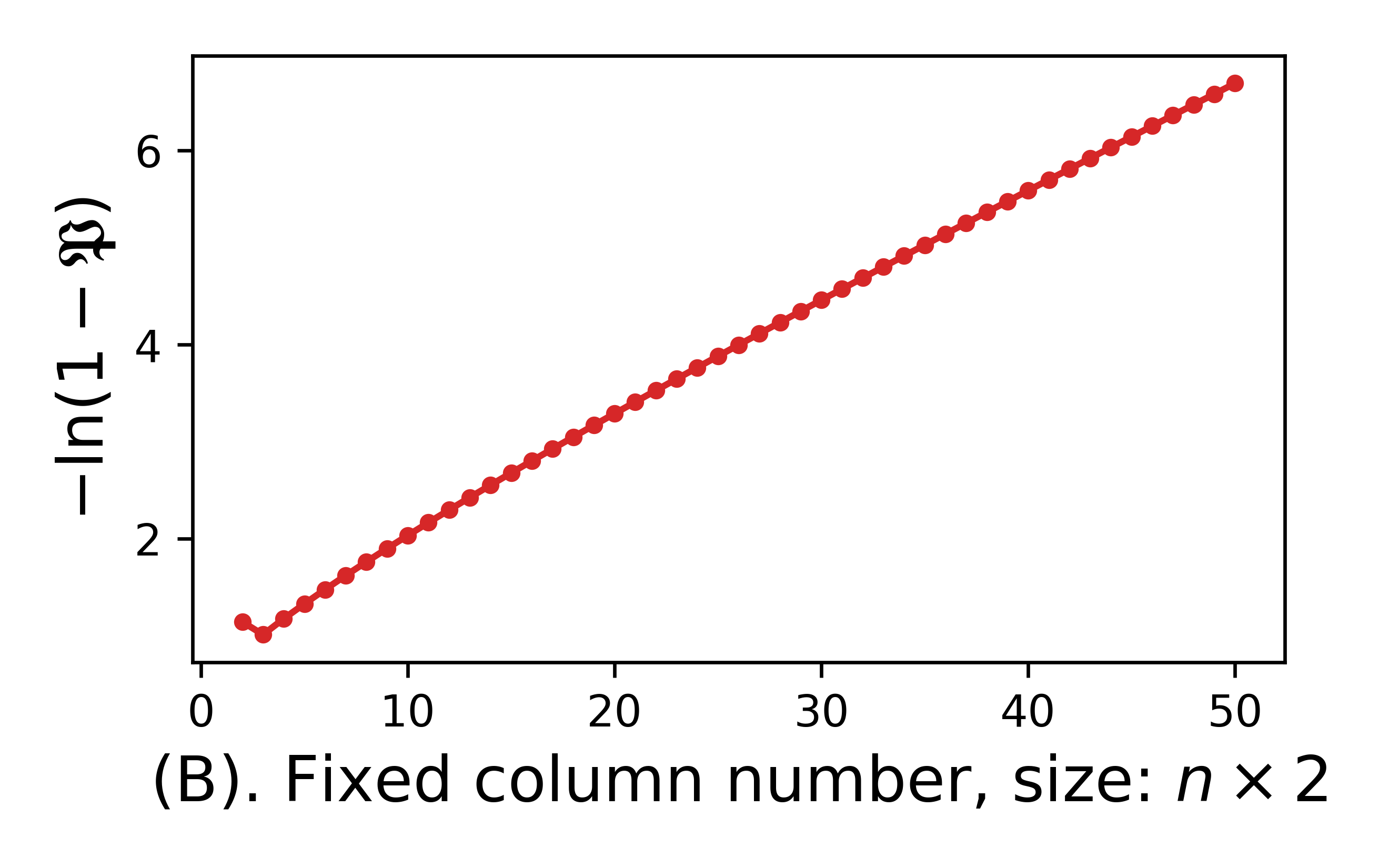}}
  %\raisebox{2.8cm}{\textbf{d}}
  \raisebox{0.12cm}{\includegraphics[scale=0.49]{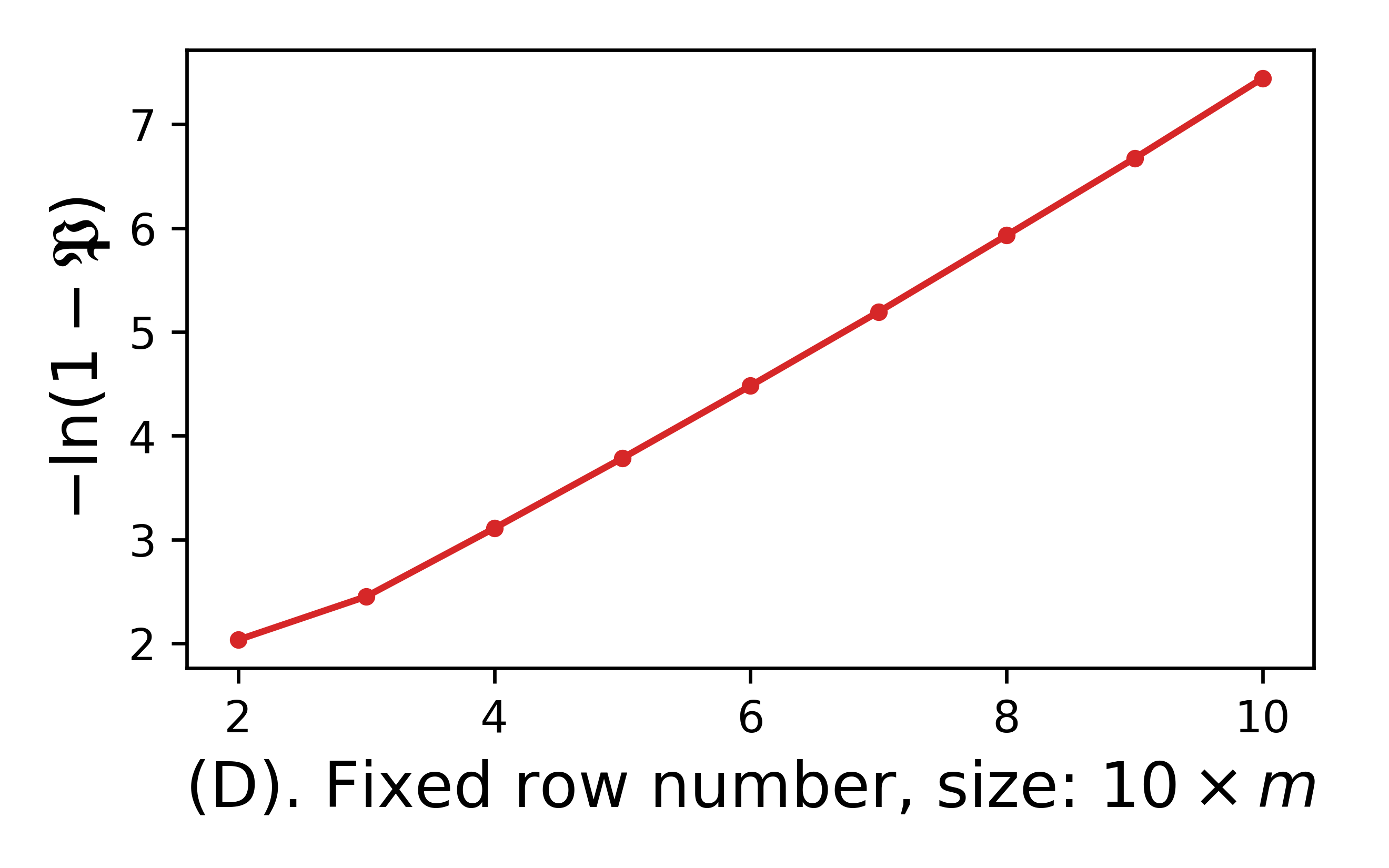}}
  %\raisebox{2.8cm}{\textbf{f}}
  \includegraphics[scale=0.51]{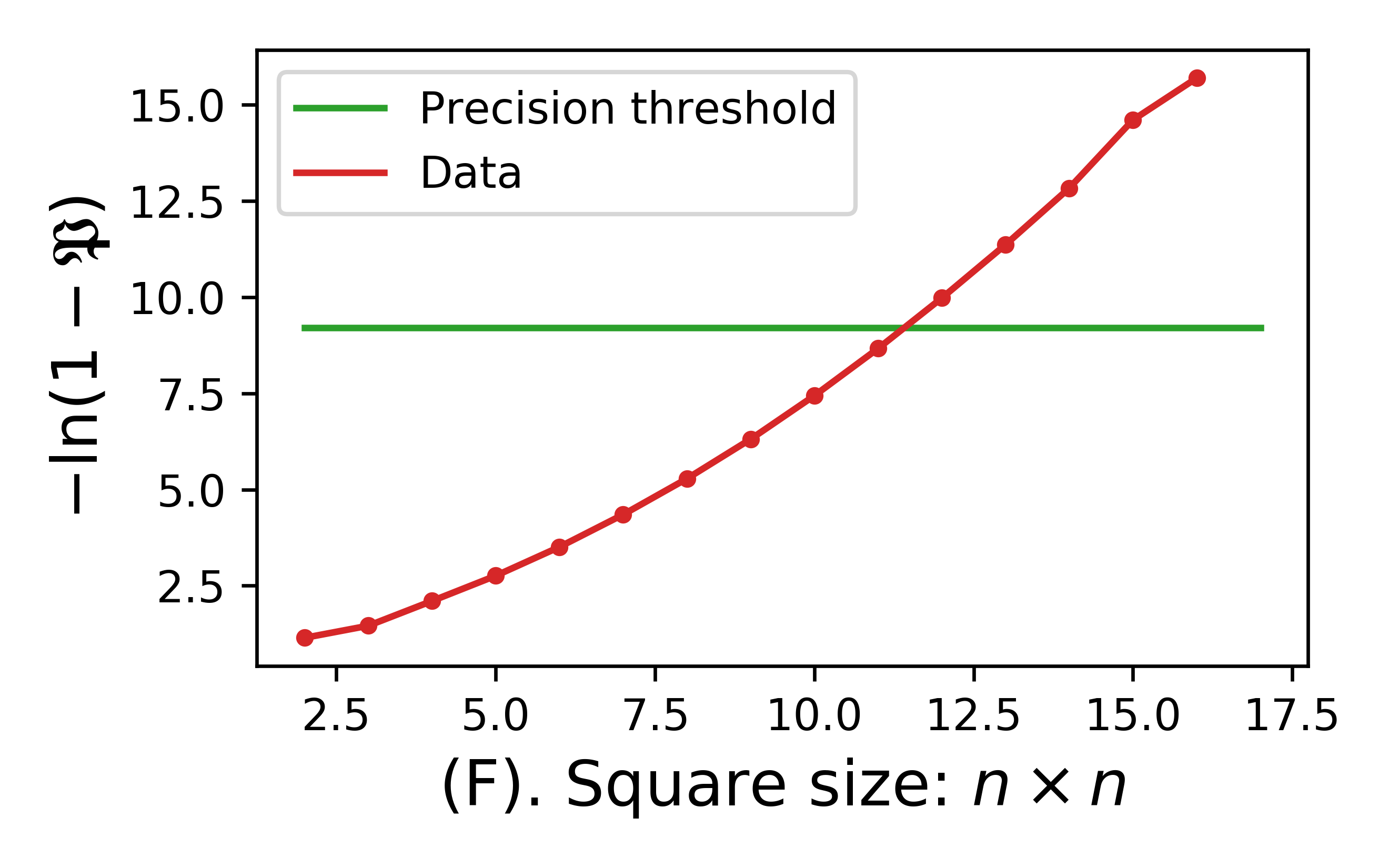}
  \caption{
    Comparison of RoC between BI and SCBI. (A), (C), (E): the comparison
    on $\mathfrak{P}$ in {\color[RGB]{31,119,179}blue} and on $\mathfrak{E}$ in
    {\color[RGB]{214,39,40}red}. (B), (D), (F): plotting $-\ln(1-\mathfrak{P})$.
    (A), (B): two-column case, number of rows from $2$ to $50$. Monte Carlo
    of $10^{10}$ samples for each point on figure. (C), (D): 10-row case, number of
    columns from $2$ to $10$. Monte Carlo of size $10^8$. (E), (F): square case,
    number of rows from $2$ to $50$. Monte Carlo of size $10^8$. The horizontal line
    in (F) is the theoretical threshold of precision by central limit theorem. For
    $n>17$, MC provides $\mathfrak{P}=1$ ($\mathfrak{R}^{\text{s}}>\mathfrak{R}^{\text{b}}$
    for all samples). From the figures, except in (C) where $\mathfrak{E}$ decays slowly
    when column number grows, the two values $\mathfrak{E}$ and $\mathfrak{P}$ increases
    as size grows in all the other cases.
    Moreover, $\mathfrak{P}$ grows to $1$ logistically in all situations.
    \label{fig:roc} }
\end{figure*}

% \begin{figure}[!ht]
%   \centering
%   \includegraphics[scale=0.4]{figures/RoC_2col_mean_prob}
%   \includegraphics[scale=0.4]{figures/RoC_2col_log}
%   \vspace{-0.4cm}
%   \caption{Two-column cases, number of rows from $2$ to $50$. Monte Carlo
%     with $10^{10}$ data points. 
%     Left: $\mathfrak{P}$ in blue and $\mathfrak{E}$ in
%     red. Right: $-\ln(1-\mathfrak{P})$ in blue.
%     Graphs show that both $\mathfrak{P}$ and $\mathfrak{E}$ are in favor of SCBI and growing.
%     \label{fig:roc_2col} }
% \end{figure}
%\paragraph{More Columns of a Fixed Row Size}
\noindent 
\textbf{More Columns of a Fixed Row Size.}
To verify the general cases, we simulated $\mathfrak{P}$ and $\mathfrak{E}$ by
Monte Carlo on matrices of $10$-row and various-column shapes, see Fig.~\ref{fig:roc}(C)(D).
We sampled $10^8$ different $\mathbf{M}$ of shape $10\times m$ for each $2\le m\le10$.
Empirical results show that $\mathfrak{E}$ decreases slowly but $\mathfrak{P}$
still increase logistically as $m$ grows.
% \vspace{-3mm}
% \begin{figure}[!ht]
%   \centering
%   \includegraphics[scale=0.4]{figures/RoC_Row10_mean_prob}
%   \includegraphics[scale=0.4]{figures/RoC_Row10_log}
%   \vspace{-0.4cm}
%   \caption{Ten-row cases, number of columns from $2$ to $10$. Monte Carlo
%     with $10^8$ data points. 
%     Left: $\mathfrak{P}$ in blue and $\mathfrak{E}$ in
%     red. Right: $-\ln(1-\mathfrak{P})$ in blue.\label{fig:roc_10row}
%     Figures show that $\mathfrak{E}$ decreases when column number increases, but $\mathfrak{P}$ still have logistic growth to $1$.
%     % \pat{good to provide more detail here. in general, i like figure captions that tell people what they should conclude.}
%     }
% \end{figure}
% \vspace{-3mm}
% \paragraph{Square Matrix Cases}

\noindent
\textbf{Square Matrices.}
% The last set is on square cases with size from $2$ to $50$, simulated by size
% $10^8$ Monte Carlo, shown in Fig.~\ref{fig:roc_square}.
Fig.~\ref{fig:roc}(E)(F) shows the square cases with size from $2$ to $50$, simulated by size
$10^8$ Monte Carlo.
% \vspace{-3mm}
% \begin{figure}[ht]
%   \centering
%   \includegraphics[scale=0.6]{figures/RoC_square_mean_prob}
%   \includegraphics[scale=0.6]{figures/RoC_square_log}
%   \vspace{-0.4cm}
%   \caption{Two-column cases, number of columns from $2$ to $50$. Monte Carlo
%     with $10^8$ data points. Left: $\mathfrak{P}$ in blue and $\mathfrak{E}$ in
%     red. Right: $-\ln(1-\mathfrak{P})$ in blue and precision threshold from
%     central limit theorem in red, from $n=17$ on, all MC sample points have a
%     larger $\mathfrak{R}^{\text{s}}$ than $\mathfrak{R}^{\text{b}}$ such that
%     $\ln(1-\mathfrak{P})$ is not well 
%     defined. From figures, $\mathfrak{P}$ and $\mathfrak{E}$ both favor SCBI and grows as size increases.
%     % \pat{we should tell folks what the conclusion should be}
%     } 
%   \label{fig:roc_square}
% \end{figure}
% \vspace{-0.3cm}

The empirical $\mathfrak{P}$ is the mean of $N$ (sample-size)
i.i.d. variables valued $0$ or $1$, thus the standard deviation of a single variable 
is smaller than $1$. By Central Limit Theorem, the standard deviation
$\sigma(\mathfrak{P})<N^{-1/2}$ (precision threshold). So we draw
lines $y=N^{-1/2}$ in each log-figure, but only in one figure the line
lies in the view area.

% By Central Limit Theorem, Monte Carlo method of size $10^8$ provides a
% precision of $10^{-4}$ \pat{not sure how this is computed, not super critical, but good to add detail} \jw{ Range$=1$, stdev$<1$, so precision in $n^2$ size is $1/n$}, thus we draw in Fig.~\ref{fig:roc_square} a red line
% $y=4\ln(10)\approx9.21$ and treat all data points above it as non-precise
% points, while in the other two figures this line \pat{not sure what the line is here} is above the upper border \pat{revisit this sentence to clarify?}. 

In all simulated cases, we observe that $\mathfrak{E}>0$ and $\mathfrak{P}>0.5$, 
indicating that SCBI converges faster than BI in most cases and in average.
It is also observed that SCBI behaves even better as matrix size grows, especially when the teacher has more choices on the data to be chosen (i.e., more rows).
% We can see from the above results that the expectation of
% $\mathfrak{R}^{\text{s}}$ is always greater than $\mathfrak{R}^{\text{b}}$,
% and $\mathfrak{P}$ is always greater than $0.5$. Moreover, as the columns be 
% of higher dimension, both $\mathfrak{P}$ and $\mathfrak{E}$ grows. And
% although $\mathfrak{E}$ decays as number of column grows with a fixed number
% of rows, $\mathfrak{P}$ still increases. \pat{good to restate in english too.}

% \pw{we need to say out loud that simulations suggest that SCBI converges faster in expectation, and with higher probability.}

\subsection{Distribution of Inference Results}
%To find out the finite-round behavior of SCBI, we may also compare SCBI
%with BI on a fixed round result.
% \pw{step and round are both used.}
The promises of cooperation is that one may infer hypotheses from small amounts of data. Hence, we compare SCBI with BI after small, fixed numbers of rounds. 

We sample matrices of shape $20\times20$ whose columns are distributed evenly
in $\Delta^{19}$ to demonstrate. Equivalently, they are column-normalizations of
the uniformly sampled matrices whose sum of all entries is one.
% \pw{sentence is too long?}

Assume that the correct hypothesis to teach is $h\in\mathcal{P}(\mathcal{H})$
We first simulate a $5$-round inference behavior, exploring all possible
ways that the teacher may teach, 
%(with their own probability), 
then calculate the
expectation and standard deviation of $\theta(h)$. With $300$ matrices sampled
in the above way, Fig.~\ref{fig:tree_tracing_step_5} shows this comparison
between BI and SCBI. 
% \vspace{-2mm}
\begin{figure}[!ht]
  \centering
  \includegraphics[scale=0.055]{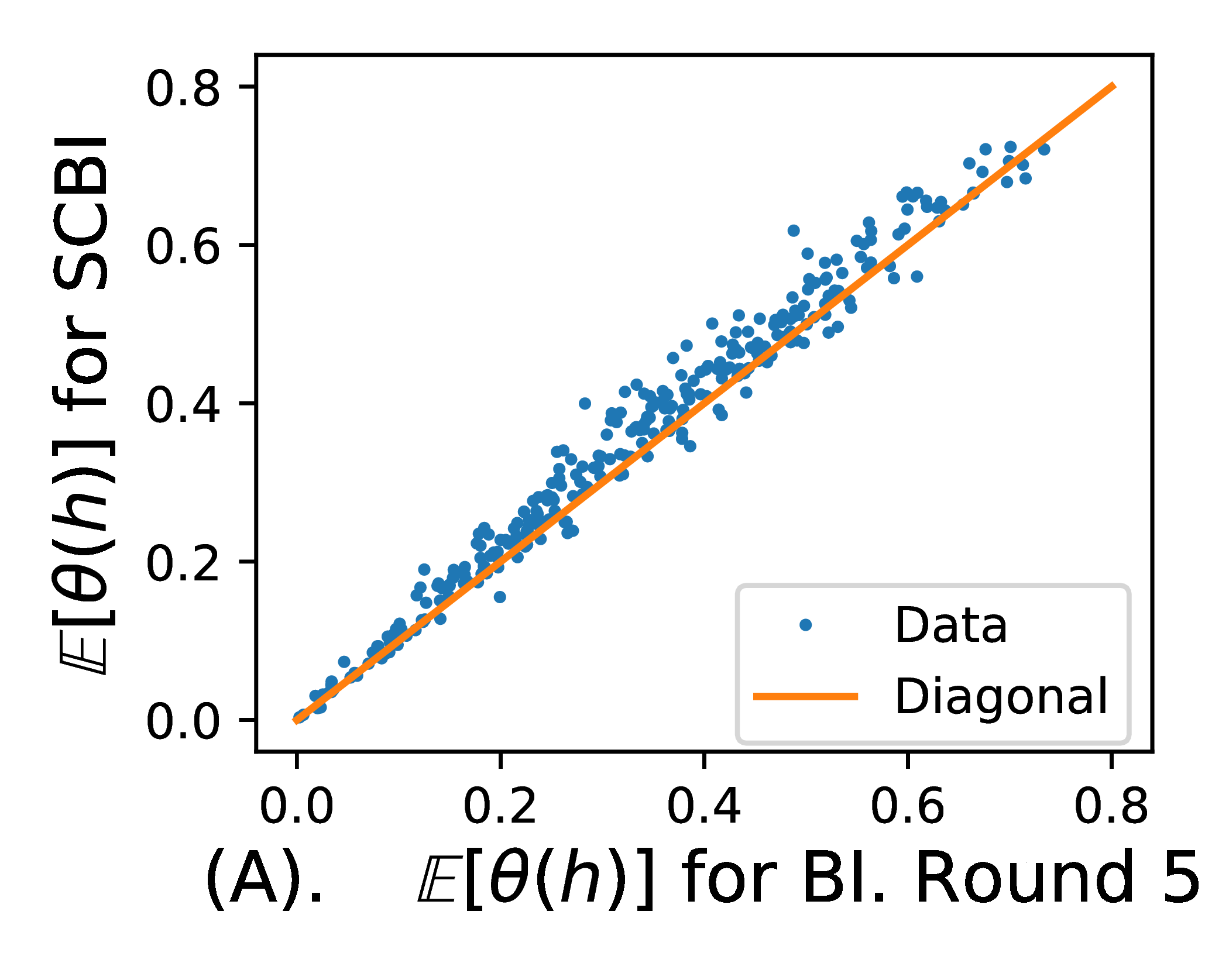}
  \includegraphics[scale=0.055]{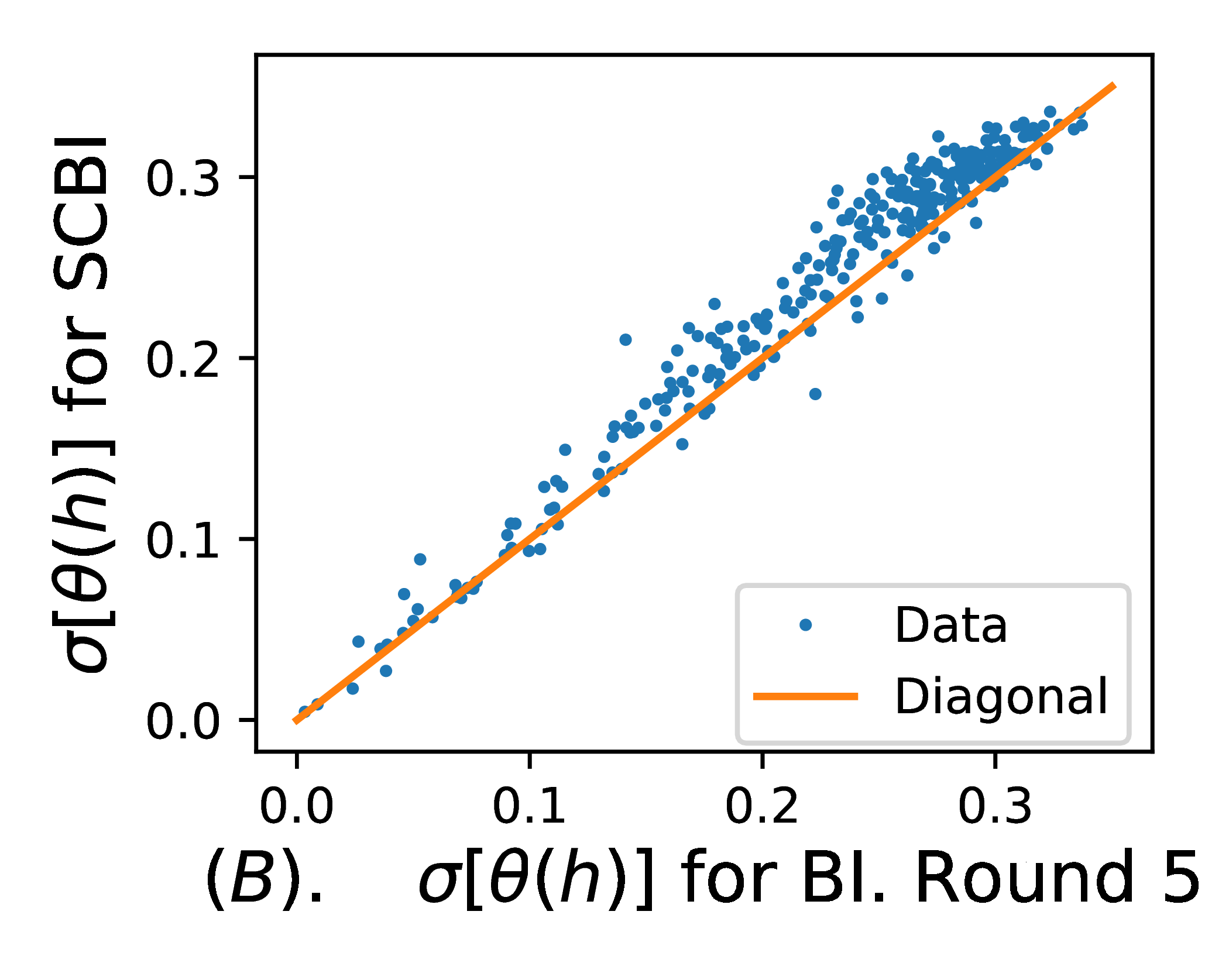}
  
  \includegraphics[scale=0.055]{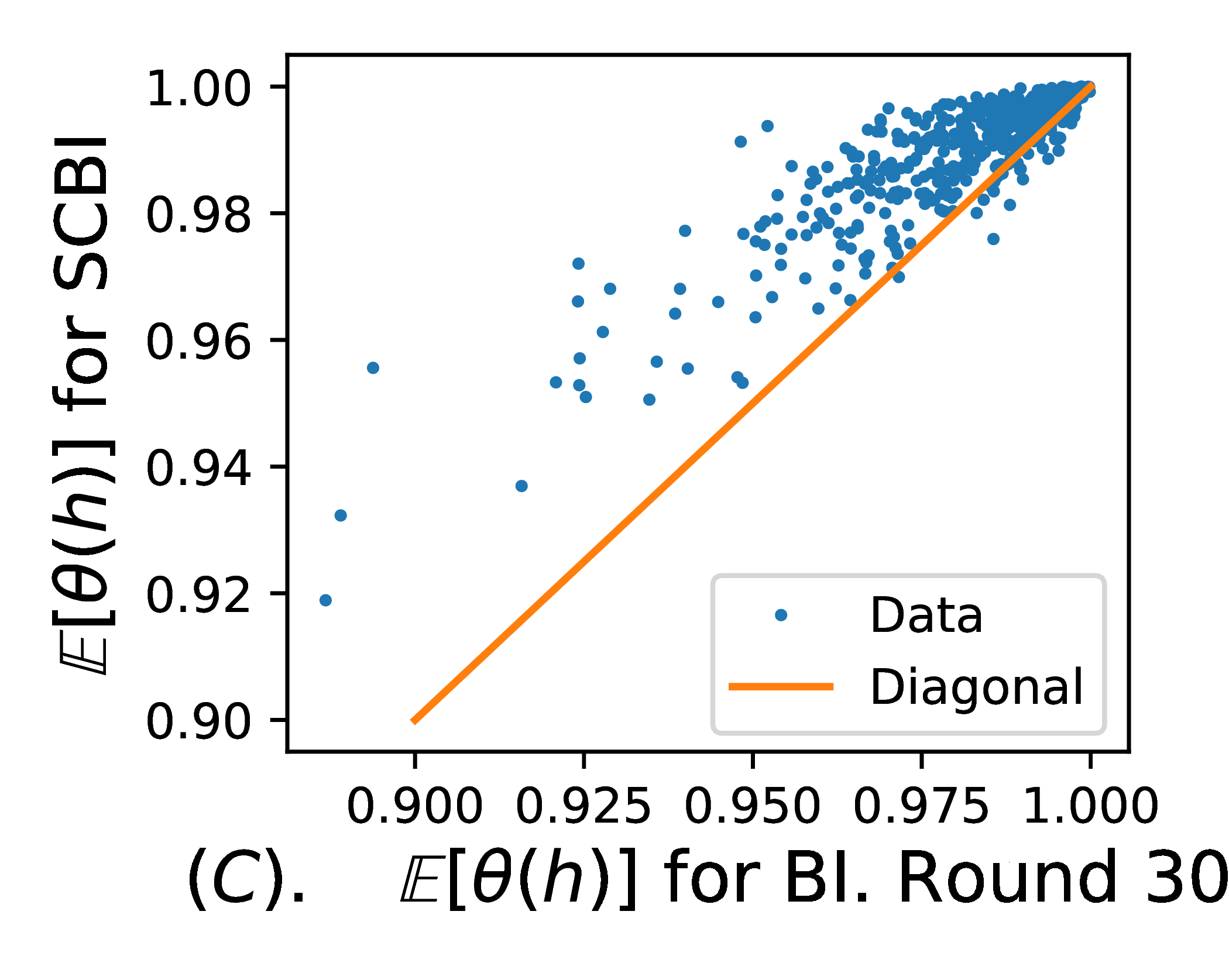}
  \includegraphics[scale=0.055]{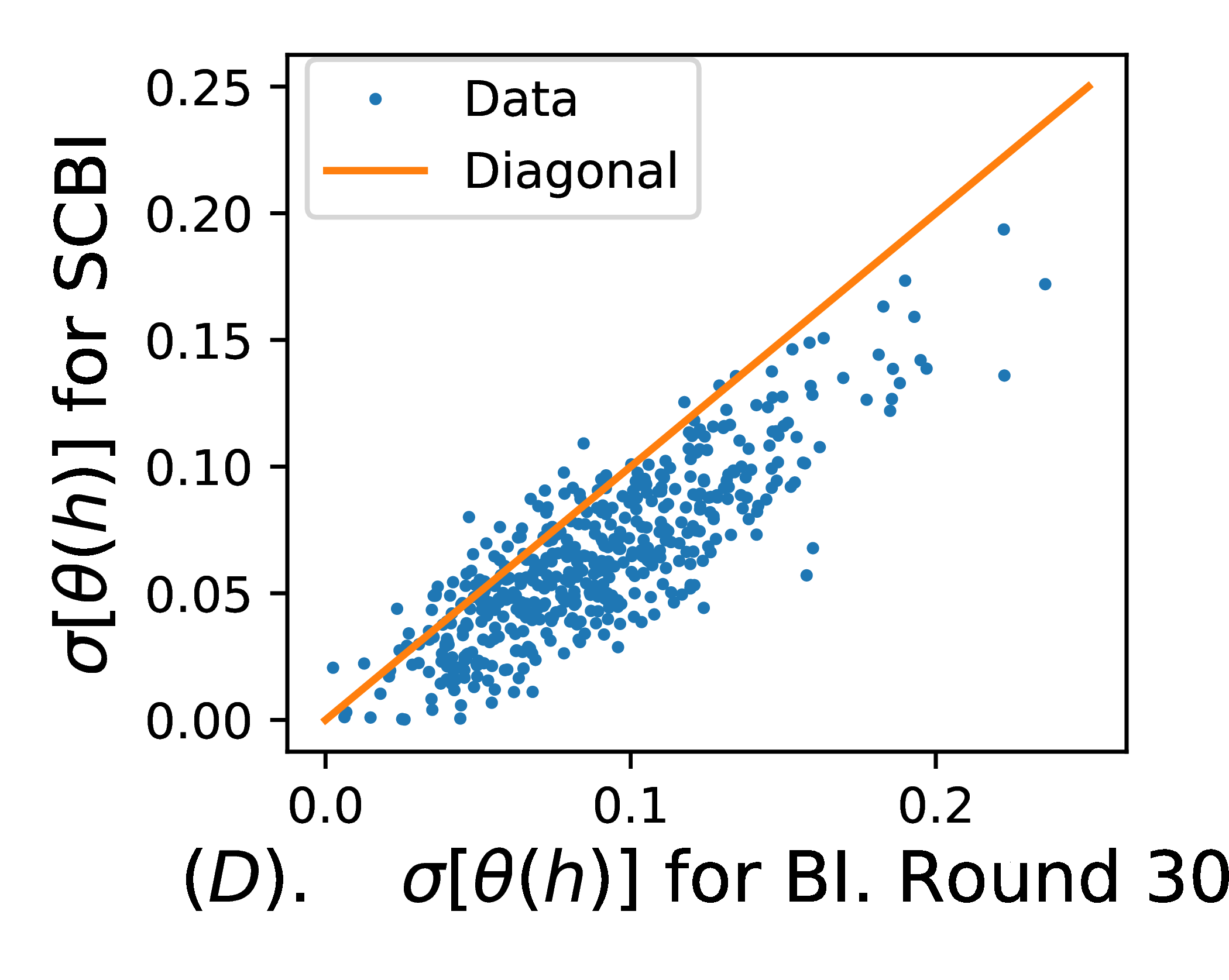}
  \vspace{-0.5cm}
  \caption{Comparison between BI and SCBI on $20\times20$ matrices: Top: $300$ points (matrices) of round $5$ accurate value.
  Bottom: $500$ points of round $30$ using Monte Carlo of size $2000$.
  Left: comparison on expectations of learner's posterior on $h$.
  Right: comparison on the standard deviations. Orange
  line is the diagonal. 
  %In the left figures, points above the diagonal indicate SCBI yields a higher expectation. In the right figure, points above the diagonal indicate SCBI yields higher expected standard deviation at round $5$ and lower on round $30$. In the vast majority of cases, SCBI yields higher expected probability of the correct hypothesis.
  }
  \label{fig:tree_tracing_step_5}
\end{figure}
% \begin{figure}[ht]
%   \centering
%   \includegraphics[scale=0.3]{figures/tree_tracing_20_mean}
%   \includegraphics[scale=0.3]{figures/tree_tracing_20_std}
  
%   \includegraphics[scale=0.3]{figures/30step_MC_20_mean}
%   \includegraphics[scale=0.3]{figures/30step_MC_20_std}
%   \vspace{-0.4cm}
%   \caption{Comparison between BI and SCBI after $5$ rounds:
%   each point, corresponding to a $20\times20$-matrix, represents the learner's expected posterior on the true hypothesis 
%   averaged over all possible teaching sequences of length $5$. Orange
%   line is the diagonal. In the left figure, points above the diagonal indicate SCBI yields a higher expectation. In the right figure, points above the diagonal indicate SCBI yields higher expected standard deviation. In the vast majority of cases, SCBI yields higher expected probability of the correct hypothesis.}
%   \label{fig:tree_tracing_step_5}
% \end{figure}
% \vspace{-0.3cm}
Similarly, we extend the number of rounds to $30$ by Monte Carlo since an exact
calculation on exhausting all possible teaching paths becomes impossible. With
sampling $500$ matrices independently, we simulate a teacher teaches $2000$
times to round $30$ for each matrix, and the statistics are also shown in
Fig.~\ref{fig:tree_tracing_step_5}. %~\ref{fig:step_30}.
%
% \begin{figure}[ht]
%   \centering
%   \includegraphics[scale=0.3]{figures/30step_MC_20_mean}
%   \includegraphics[scale=0.3]{figures/30step_MC_20_std}
%   \vspace{-0.4cm}
%   \caption{Comparison between BI and SCBI to round $30$: $500$ points on the
%     graph, each representing a $20\times20$-matrix, mean and standard deviations
%     are taken on a data set from $2000$ repeats of Monte Carlo for each matrix.
%     Orange line is the diagonal. The effect of cooperation grows over rounds to provide higher probability of the correct hypothesis, and lower standard deviation.}
%   \label{fig:step_30}
% \end{figure}
% \vspace{-0.3cm}
%
% \pw{We probably should state the obvious that the graphs suggest SCBI have better expectation and less variance in the long run.}
From Fig.~\ref{fig:tree_tracing_step_5}, % and Fig.~\ref{fig:step_30}, 
 we 
observe that SCBI have better expectation and less variance in the short run.

% As a conclusion, we see from empirical data that SCBI behaves better on sample
% efficiency in both short-term and asymptotic comparisons.
In conclusion, experiments indicate that SCBI is both 
more efficient asymptotically, and in the short run.

% Sample Efficiency

\section{Stability}
\label{sec:stability}

% In this section, we focus on the situation that the initial conditions
% of the two agents are not identical. 
% \pw{How about something like, in this section, we investigate the situation where agents do not have full access 
% to their partner's exact beliefs. }
In this section, we study the robustness of SCBI by setting the initial
conditions of teacher and learner different. This could happen when 
agents do not have full access to their partner's exact state.

%\paragraph{Theory}
\noindent
\textbf{Theory.}
In this section, we no longer have assumption (iii). Let $\mathbf{T}$ and $\mathbf{L}$ be
matrices of teacher and learner (not necessarily have (iii)).
Let $\theta_0^T$ and $\theta_0^L$
be elements in $\mathcal{P}(\mathcal{H})$ representing the prior on hypotheses
that the teacher and learner use in the estimation of inference (teacher) and
in the actual inference (learner), 
% \pw{keep the order that learner goes first, then teacher?}
i.e., $\mu_0^T=\delta_{\theta_0^T}$ and 
$\mu_0^L=\delta_{\theta_0^L}$. During the inference, let $\mu_k^T$ and
$\mu_k^L$ be the distribution of posteriors of the teacher and the
learner after round $k$, and denote the corresponding random variables
by $\theta_k^T$ and $\theta_k^L$, for all positive $k$ and $\infty$,
where $\infty$ represents the limit in probability.

% Our analysis splits into long-term analysis and short-term analysis.

% \subsection{Long-term Analysis}
%\label{subsec:long-term_analysis}

Let $\mathrm{D}$ be a random variable on $\mathcal{D}$, 
we define an operator 
$\Psi_{\mathrm{D}}^{\mathbf{L}}:\mathcal{P}(\mathcal{P}(\mathcal{H}))
\longrightarrow\mathcal{P}(\mathcal{P}(\mathcal{H}))$ similar
to the $\Psi$ in Section~\ref{sec:consistency}. Let 
$T_d(\theta)=\mathbf{L}^{\left\langle n\theta\right\rangle}_{(d,\_)}$,
then $\mathrm{d}(\Psi_{\mathrm{D}}^{\mathbf{L}}(\mu))(\theta):=
\sum\limits_{d\in\mathcal{D}}{\mathrm{P}(\mathrm{D}=d)}\mathrm{d}(T_{d\ast}\mu)(\theta)$.
%\pw{pei needs to read this again and check the proof.}

\begin{prop}
   \label{prop:asympt_posterior}
  Given a sequence of identical independent $\mathcal{D}$-valued 
  random variables $(\mathrm{D}_i)_{i\ge 1}$ following the uniform distribution.
  Let $\mu_0\in\mathcal{P}(\mathcal{P}(\mathcal{H}))$ be a prior distribution on
  $\mathcal{P}(\mathcal{H})$, and
  $\mu_{k+1}=\Psi_{\mathrm{D}_{k+1}}^{\mathbf{L}}(\mu_k)$, then 
  $\mu_k$ converges, in probability, to
  $\sum_{i\in\mathcal{H}}a_i\delta_{i}$ where
  $a_i=\mathbb{E}_{\mu_0}\left[\theta(i)\right]$.
\end{prop}
\begin{remark}
  This proposition helps accelerate the simulation, that
  one may terminate the teaching process when $\theta_k^T$ 
  is sufficiently close to $\delta_h$, since
  Prop.~\ref{prop:asympt_posterior} guarantees that the 
  expectation of the learner's posterior on the true hypothesis $h$ 
  at that time is close enough to the eventual probability of getting
  $\delta_h$, i.e. 
  $\mathbb{E}\theta_\infty^L(h)\approx\mathbb{E}\theta_k^L(h)$.
  
  % we can stop teaching when the teacher's posterior (estimation on learner's
  % posterior) , $\theta_k^T$, is sufficiently close to $\delta_h$,
  % \pw{how about, one may safely terminate the teaching process as when $\theta_k^T$ is sufficiently close to $\delta_h$ since
  % Prop.~\ref{prop:asympt_posterior} guarantees that the expectation of the Learner's posterior on the true hypothesis is stable 
  % from this situation on, i.e. $\mathbb{E}\theta_\infty^L(h)\approx\mathbb{E}\theta_k^L(h)$.}
  % then Prop.~\ref{prop:asympt_posterior} guarantees that the probability
  % of learner's correct inference from this situation on is
  % $\mathbb{E}\theta_\infty^L(h)\approx\mathbb{E}\theta_k^L(h)$.
\end{remark}
\begin{definition}
  \label{def:successful_rate}
  We call $\mathbb{E}\theta_\infty^L(h):=
  \lim\limits_{k\rightarrow\infty}\mathbb{E}_{\mu_k}(\theta(h))$
  the 
  \textbf{successful rate} of the inference given $\mathbf{T}$, $\mathbf{L}$,
  $\theta_0^T$ and $\theta_0^L$. By the setup in
  Section~\ref{sec:scbi_construction}, the failure probability,
  $1-\mathbb{E}\theta_\infty^L(h)$, is 
  $\frac{1}{2}||\mathbb{E}\theta_\infty^L-{\delta_h}||^{\phantom{|}}_1$, half of the $1$-distance on $\mathcal{P}(\mathcal{H})$.
\end{definition}
% \jw{it seems that we need to prove the well-definedness of successful rate?}
% \subsection{Under Perturbation}
% \label{subsec:perturb}
% Using Prop.~\ref{prop:asympt_posterior}, we may 

%\paragraph{Simulations with Perturbation on Priors}
\noindent
\textbf{Simulations with Perturbation on Priors.}
We simulated the square cases of rank $3$ and $4$. 
% Our simulation for stability is bounded in $3\times3$ and $4\times4$ cases.
% \pw{we shall word this in a better way.}
We sample $5$ matrices ($\mathbf{M}_1$ to $\mathbf{M}_5$) of size
$3\times3$, whose columns distribute
uniformly on $\mathcal{P}(\{d_1,d_2,d_3\})=\Delta^{2}$, and $5$
priors ($\theta_1$ to $\theta_5$) in $\mathcal{P}(\mathcal{H})$,
%\pw{this is on hypotheses space}, 
used as $\theta_0^T$.
Similarly, we sample $3$ matrices ($\mathbf{M}_1'$, $\mathbf{M}_2'$ and $\mathbf{M}_3'$)
of size $4\times4$, and $3$ 
% corresponding \pw{delete, no correspondence here?} 
priors ($\theta_1'$, $\theta_2'$, $\theta_3'$)
from $\Delta^3$ in the same way as above. In both cases, 
% \pw{without loss, we may assume} 
we assume $h=1\in\mathcal{H}$ to be the true hypothesis to teach.
% \jw{Here for given matrices, it is actually not without loss...}

Our simulation is based on Monte Carlo method of $10^4$ teaching sequences (for
each single point plotted) then use
Proposition~\ref{prop:asympt_posterior} to calculate the successful rate of
inference. 
For $3\times3$ matrices, we perturb $\theta_0^L$ in two ways: (1) take $\theta_0^L$ around
$\theta_0^T$ distributed evenly on concentric circles, thus $630$ points for each
$\theta_0^T$ are taken. In this area, there are $84$ points lying on 
$6$ given directions ($60^\circ$ apart, see Supplementary Material for figures).
% \pw{How about, we obtain perturbed $\theta_0^L$ in two ways: (1) in $\Delta^2$, along $6$ randomly (?) chosen directions in $\Delta^2$, within radius ??, evenly take $15$ points on each direction (see pic ?? in SI).}
(2) sample $\theta_0^L$ evenly in the whole simplex $\mathcal{P}(\mathcal{H})=\Delta^2$ ($300$
points for each $\theta_0^T$). For $4\times4$ matrices, we perturb $\theta_0^L$
in two ways: (1) along $15$ randomly chosen directions in $\Delta^3$ evenly take
$21$ points on each direction, and (2) sample $300$ points evenly in $\Delta^3$.
Then we have the following figure samples (for figures demonstrating the entire simulation, please see
Supplementary Material).
%
% \vspace{-3mm}
\begin{figure*}[t]
  \centering
  \includegraphics[scale=0.056]{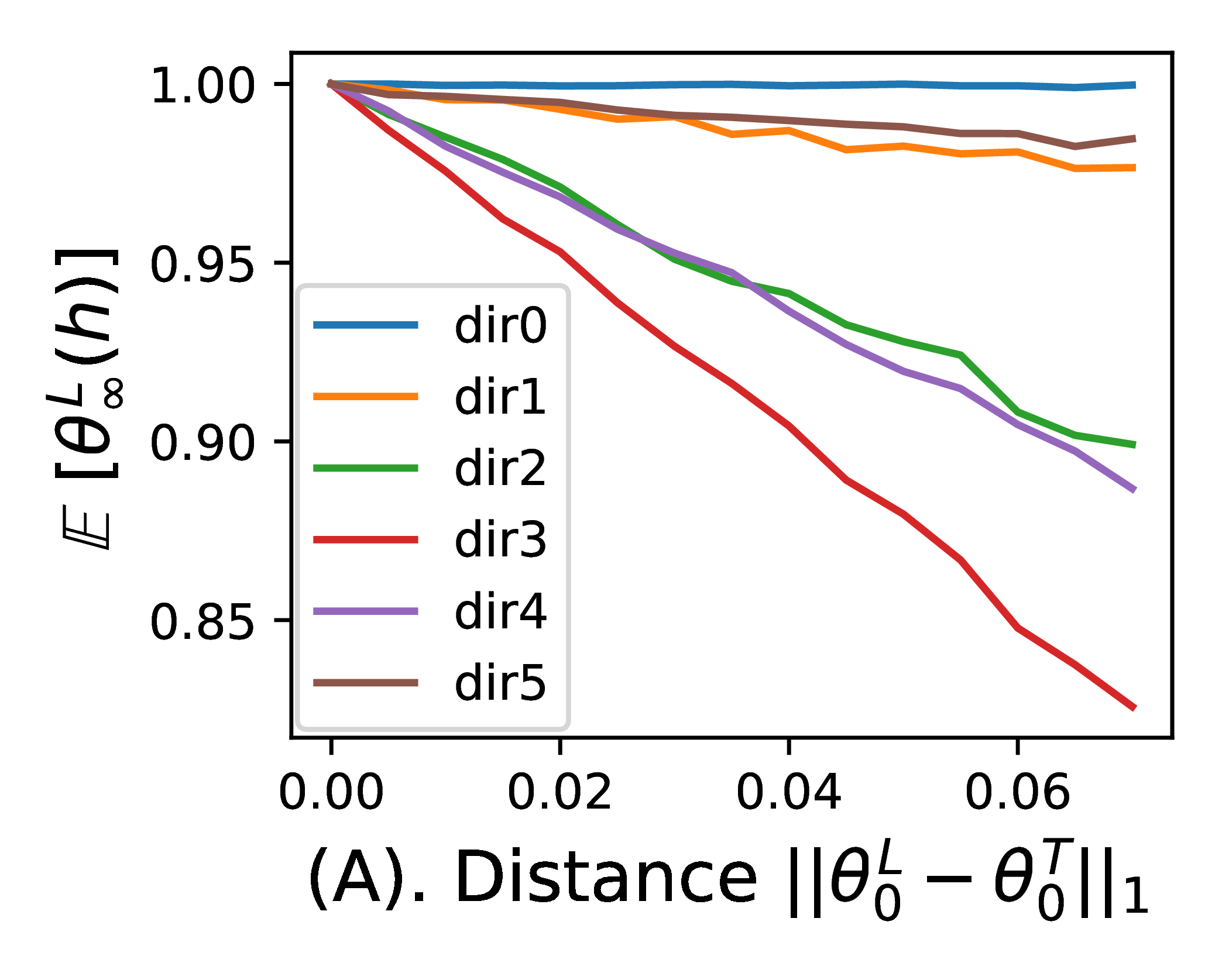}
  \includegraphics[scale=0.056]{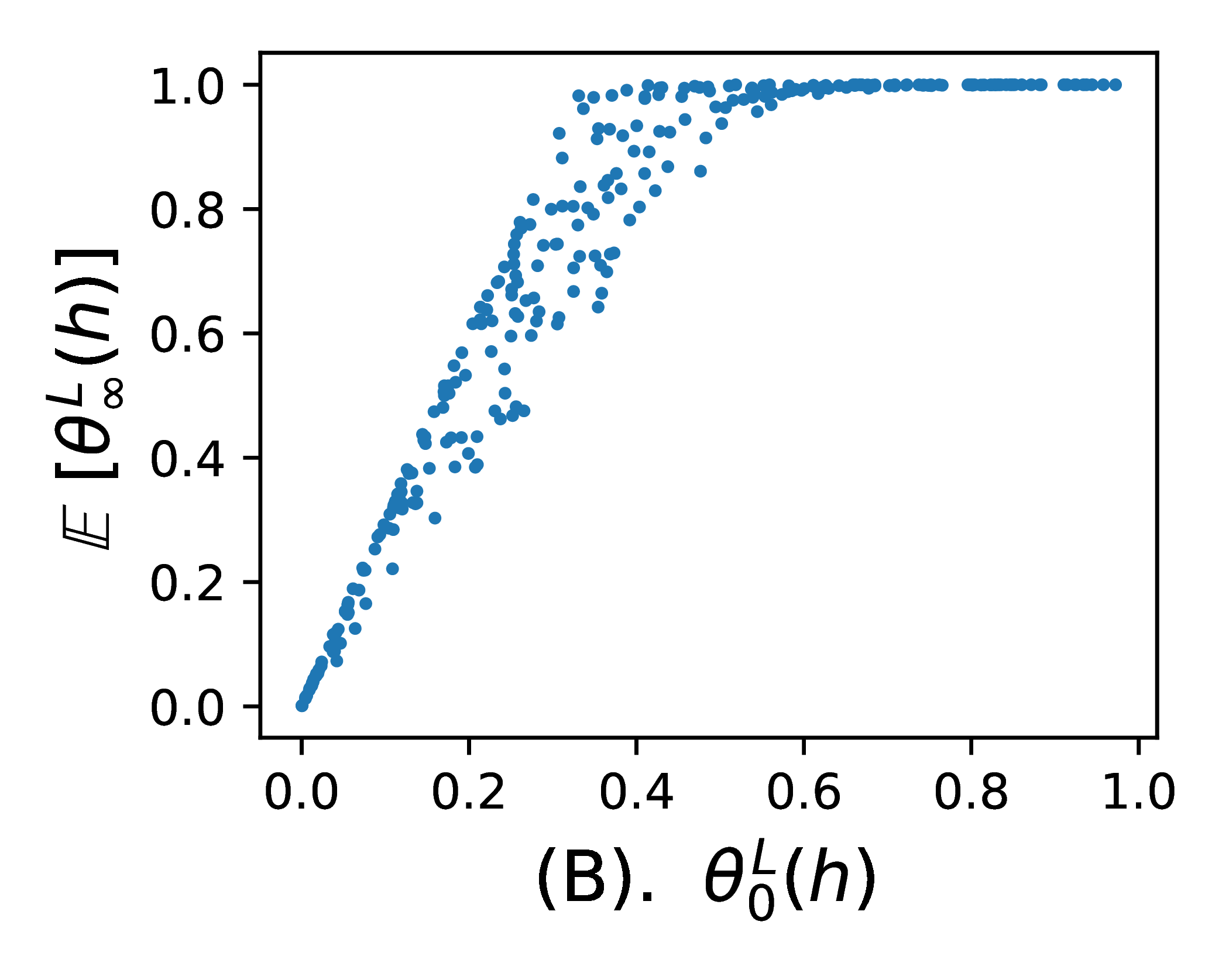}
  \includegraphics[scale=0.056]{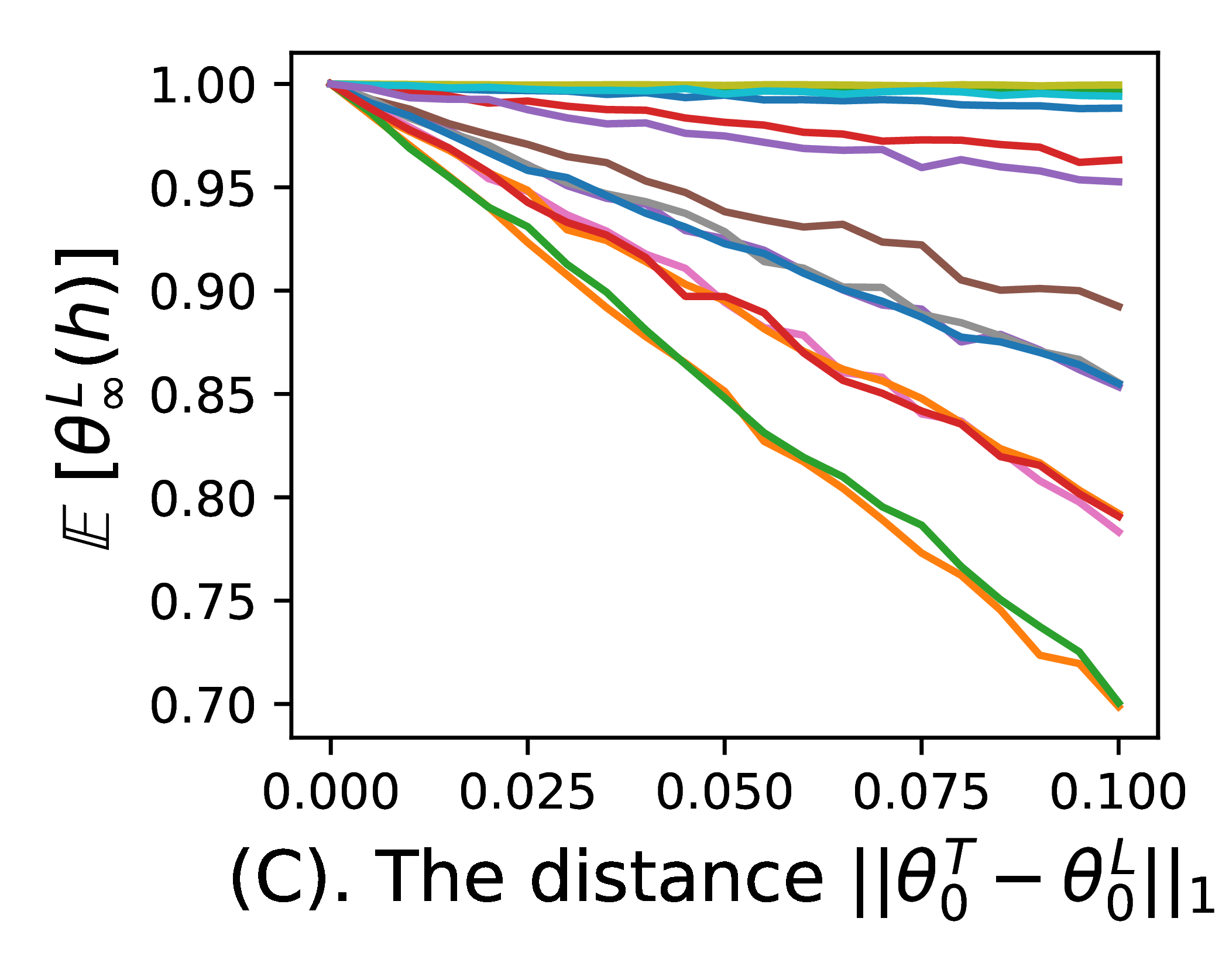}
  \includegraphics[scale=0.056]{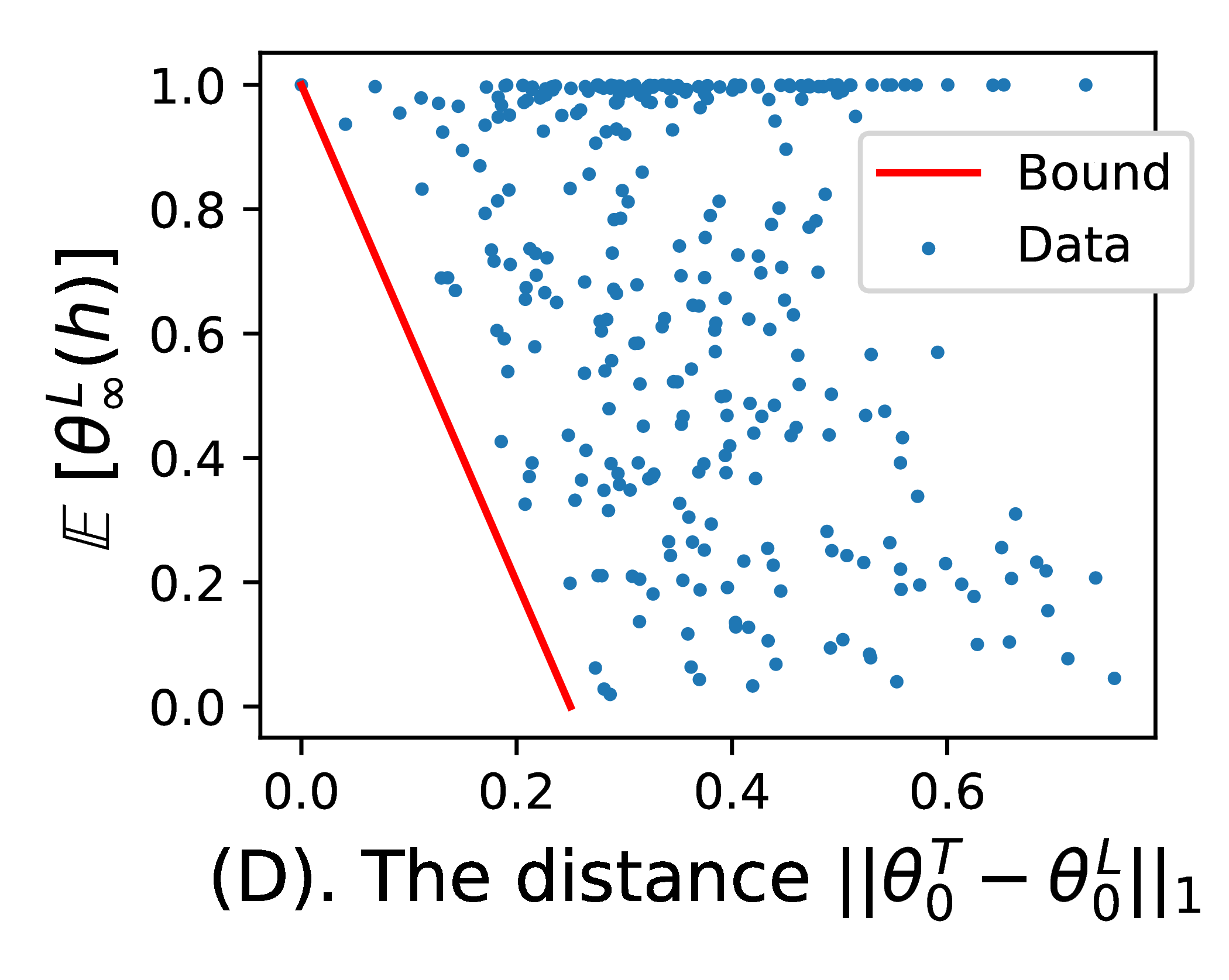}
  \vspace{-3mm}
  \caption{From left to right: (A). Rank 3, $\mathbf{M}_3$ and $\theta_1$, $\theta^L_0$
    is perturbed along six directions. 
    (B). Rank 3, $\mathbf{M}_3$ and
    $\theta_1$, sample $\theta_0^L$ uniformly in $\Delta^2$.
    (C). Rank 4, $\mathbf{M}_1'$ and $\theta_1'$, along $15$ different
    directions. 
    (D). Rank 4, $\mathbf{M}_1'$
    and $\theta_1'$, sample $\theta_0^L$ uniformly in $\Delta^3$.
    \label{fig:prior_var}}
\end{figure*}
% \pw{Maybe include some observations from the above pictures, for example:}
From the figures we see:
1. left pictures indicate that the learner's expected posterior on $h$ is roughly linear to perturbations along a line.
2. right pictures indicate that the learner's expected posterior on $h$ is closely bounded by a multiple of the learner's prior on true $h$. 
% talk a bit about the turning points on the graph?
%
Thus we have the following conjecture:
% From the all the figures (same kind for other combinations of $\mathbf{M}$ and
% $\theta_0^T$), we may conclude in the following conjecture:
\begin{conj}
  \label{conj:stability_on_prior}
  Given $\mathbf{L}=\mathbf{T}=\mathbf{M}$ and $\theta_0^T$, let $h$ be the true
  hypothesis to teach. For any $\epsilon>0$, let $\theta_0^L$ be learner's prior
  with a distance to $\theta_0^T$ less than $\epsilon$.
  Then the successful rate for sufficiently many rounds is
  greater than $1-C\epsilon$, where $C=\frac{1}{\theta_0^T(h)}$.
\end{conj}
% \pw{double check}
% \begin{remark}
%   There could be a problem caused by convergence in probability, which can 
%   hardly be observed in Monte Carlo.
% \end{remark}
% \jw{I am not very confident about the conjecture, the $K$ part is added to
%   prevent problems induced by convergence in probability. Need to think it
%   through.} 
%\paragraph{Simulations with Perturbation on Matrices}
\noindent
\textbf{Simulations with Perturbation on Matrices.}
% We restrict our simulation on matrices $\mathbf{M}_1$ to $\mathbf{M_5}$ and only
% one prior $\theta_1$ (which is uniform on $\mathcal{H}$) mentioned above, in the
% $3\times3$ case. We consider the perturbations on each column of $\mathbf{M}$.
% \pw{How about, we now investigate robustness of SCBI to perturbations on agents' common or shared likelihood matrix $M$. 
% The simulations are performed on the matrices $\mathbf{M}_1$ to $\mathbf{M_5}$ as in previous section with a fixed common prior $\theta_1$.}
We now investigate robustness of SCBI to differences between agents' matrices.
Let $\mathbf{T}$ and $\mathbf{L}$ be stochastic, and let $\mathbf{L}$ be perturbed from $\mathbf{T}$.
The simulations are performed on the matrices $\mathbf{M}_1$ to $\mathbf{M}_5$ mentioned above with a fixed common prior $\theta_1$.

Let all matrices mentioned be column-normalized
(this does not affect SCBI since cross-ratios and marginal conditions
determines the Sinkhorn scaling results), we call the column determined by the
true hypothesis $h$ (the first column in our simulation) the target column
(``tr. h'' on Fig.~\ref{fig:variation_on_matrix}),  the column which $\mathfrak{R}^{\text{s}}$
uses (argmin column) the relevant column (``rel. h'')  
% \pw{maybe add, i.e. the column is closest to h under KL } \jw{in general, it is not, unfortunately}
and the other column the irrelevant
column (``irr. h''). 
% We take $\mathbf{T}=\mathbf{M}$ and perturb
% $\mathbf{M}$ along the relevant and irrelevant columns, respectively.
Let $\mathbf{T}$ be given, and let $\mathbf{L}$ be obtained from $\mathbf{T}$ by perturbing along the relevant / irrelevant column.

% We consider the learner's matrix $\mathbf{L}$ being different from $\mathbf{T}$
% on one column only.
Without loss of generality, we assume that only one column of the learner's matrix $\mathbf{L}$ is perturbed at a time 
as other perturbations may be treated as compositions of such.

For each $\mathbf{T}$ and each column $h'$, we apply $330$ perturbations
on concentric circles around $\mathbf{T}$ (the disc), 
$90$ perturbations preserving the normalized-KL divergence
($\mathrm{KL}(\mathbf{e}/n,\mathscr{N}_{\text{vec}}(\mathbf{L}_{(\_,h')}/\mathbf{L}_{(\_,1)},1))$
used in $\mathfrak{R}^{\text{s}}$) from the target column 
% \pw{between the varying column and the target column before and after perturbations?},  
and $50$ linear
interpolations 
% \jw{Or linear mixtures? How can I describe?} 
with target column.
% \pw{double check, numbers do not add up.} \jw{added up to 470...}
Each point in Fig.~\ref{fig:variation_on_matrix} is estimated using a size-$10^4$ Monte Carlo method using Proposition~\ref{prop:asympt_posterior}.
\begin{figure*}[t]
  \centering
  \includegraphics[scale=0.56, trim=10 0 20 0,clip]{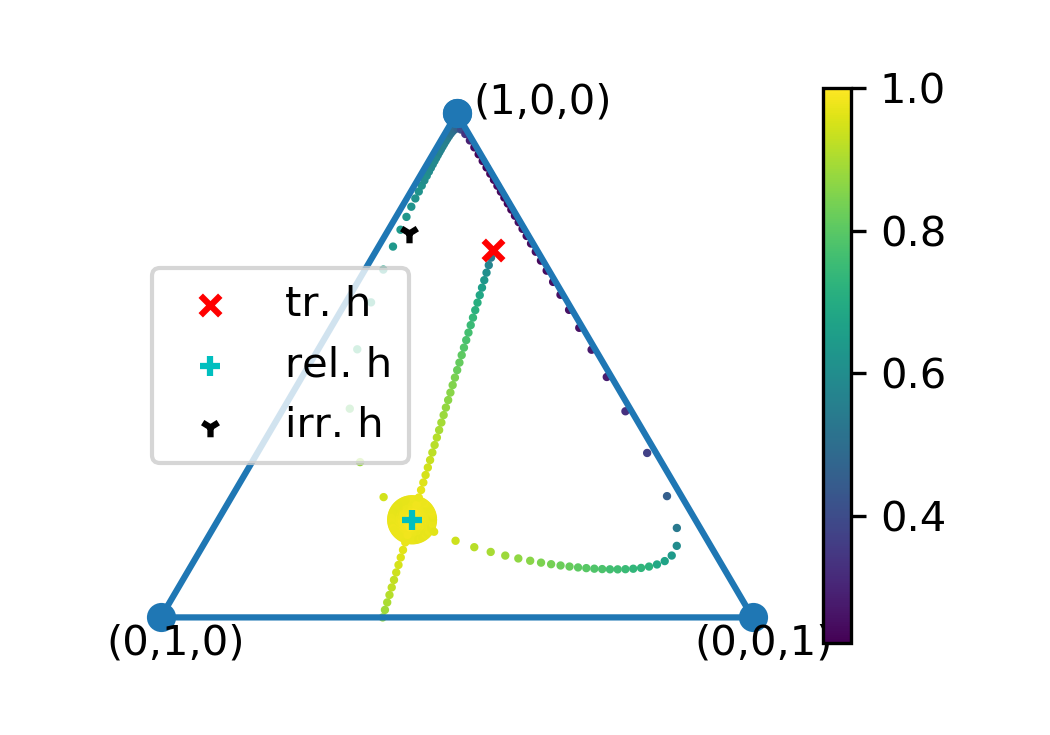}
  \includegraphics[scale=0.56, trim=20 0 10 0,clip]{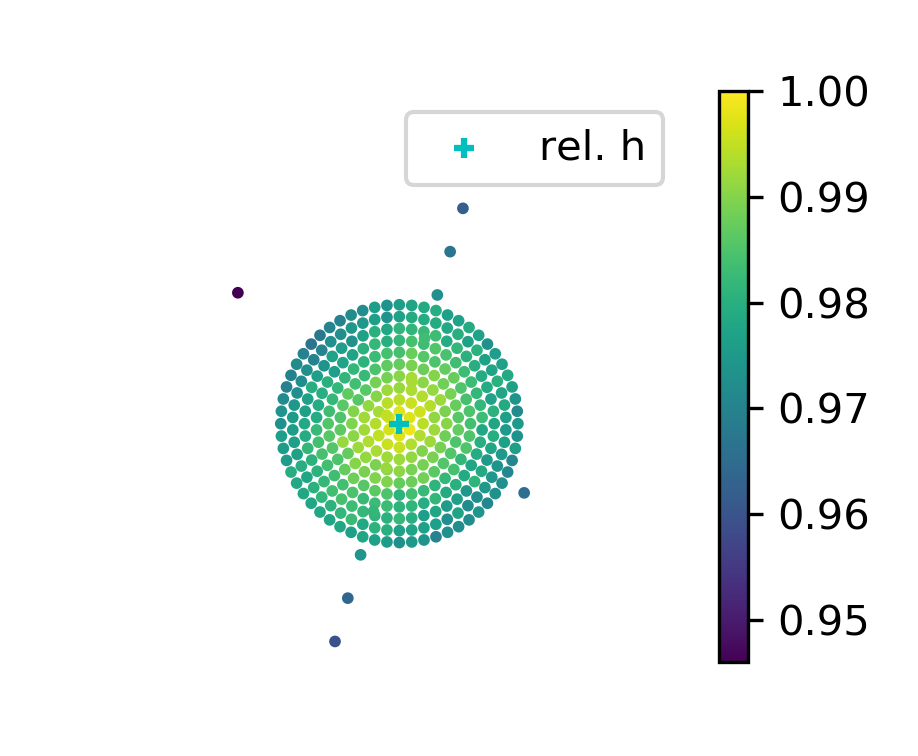}
  \includegraphics[scale=0.56, trim=0 0 20 0,clip]{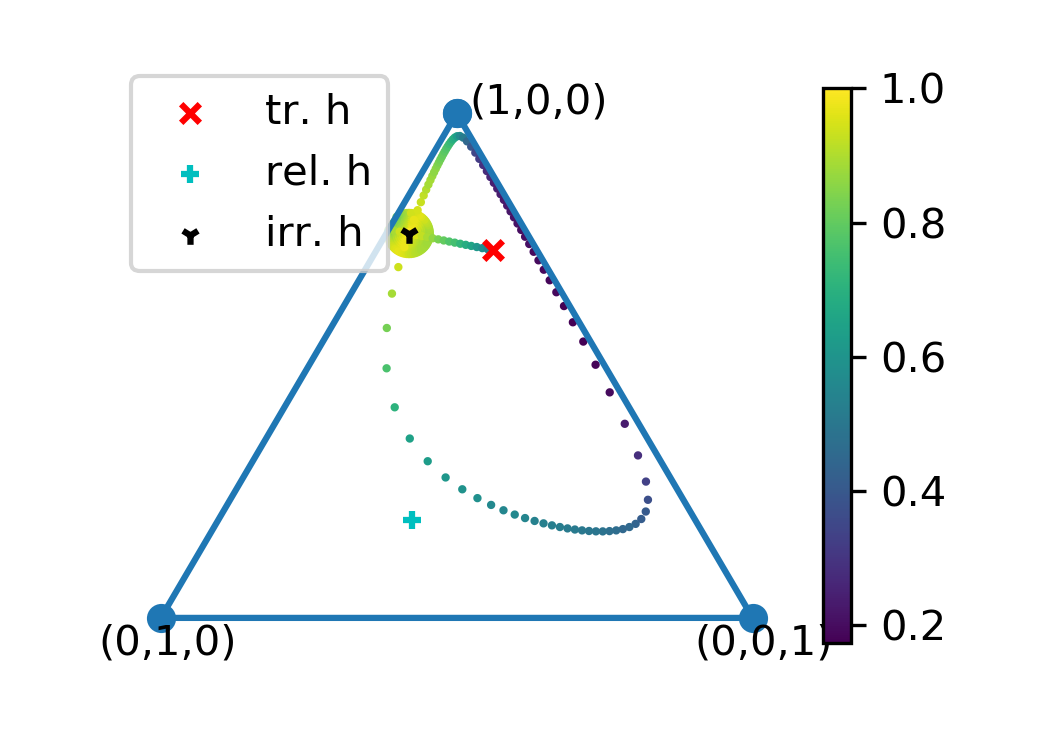}% \hspace{1.2em}
  \includegraphics[scale=0.56, trim=10 0 10 0,clip]{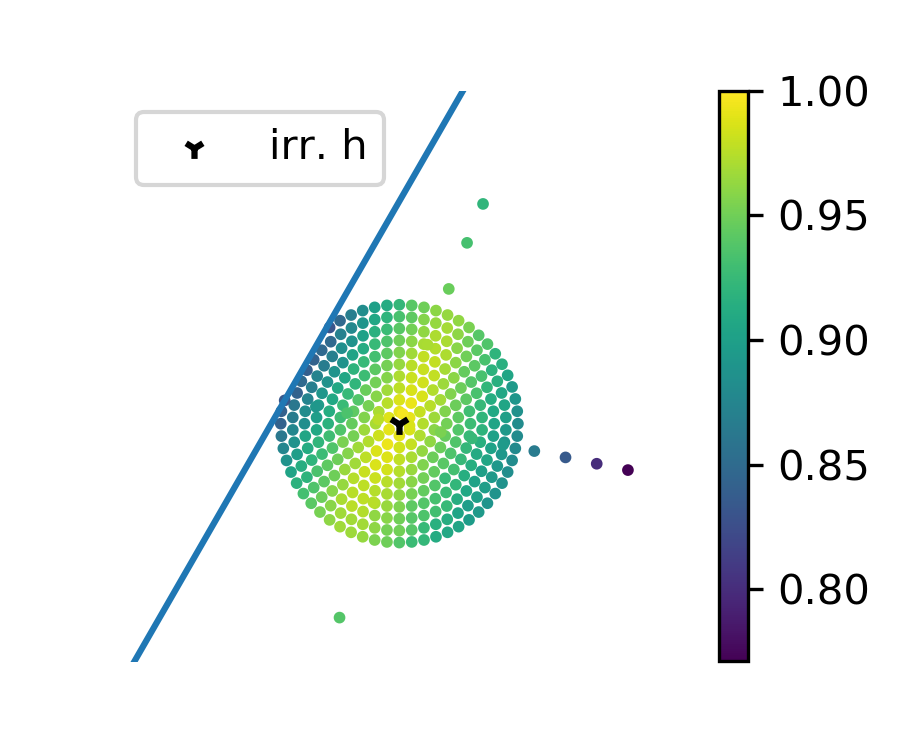}
  \vspace{-5mm}
  \caption{The perturbations on $\mathbf{M}_3$ along a column, and their
    zoomed-in version (with different color scale). The crosses shows
    the position of three normalized columns of $\mathbf{T}=\mathbf{M}_3$, the
    location of the dots
    represent the perturbed column of $\mathbf{T}$ (unperturbed columns are
    represented by crosses on figures which are not the center of disc)
    and whereas their colors depict the successful
    rate of inference. 
    % $\theta_0^T=\theta_0^L=\theta_1$ \pw{may delete?}. 
    Left two figures are perturbations on the
    irrelevant column. Right two figures are perturbations on the relevant column.}
  \label{fig:variation_on_matrix} 
  \vspace{-1mm}
\end{figure*}
From the graphs, we can see that the successful rate varies continuously on
perturbations, slow on one direction (the yellow strip crossing the center) and rapid on the perpendicular direction (color changed to blue rapidly).

%\subsection{Short-term Analysis}
%\label{subsec:short-term_analysis}

\section{Grid World: an Application}
\label{sec:application}
% \pw{we shall find a better place to put this.}
% \pw{some motivation to illustrate that what we are doing is closely related to other modeling of cooperative communication, for example cooperative RL.}
Consider a $3\times 5$ grid world with two possible terminal goals, A and B, and a starting
position $S$ as shown below.
Let the reward at the terminal position $h_t$ be $R$. 
Assuming no step costs, the value of a grid that distanced $k$ from $h_t$ is then $R\times \gamma^k$ (in the RL-sense), where $\gamma\!<\!1$ is the discount factor.

\begin{center}
\vspace{-2mm}
%\begin{tabular}{}
\begin{tabular}{ |c|c|c|c|c|c|c| } 
 \hline
 A & & & \text{\hspace{0.11in}} & & & B \\ 
 \hline
 \text{\hspace{0.11in}}& \cellcolor{blue!15} \text{\hspace{0.11in}}& \cellcolor{blue!15} & \cellcolor{blue!15} $\Uparrow$ & \cellcolor{blue!15}  & \cellcolor{blue!15}\text{\hspace{0.11in}} & \text{\hspace{0.11in}}\\ 
 \hline
 & \cellcolor{blue!15} & \cellcolor{blue!15}\! $\Leftarrow$
 & \cellcolor{blue!15} S & \cellcolor{blue!15} \!$\Rightarrow$ & \cellcolor{blue!15} &  \\ 
 \hline
\end{tabular}
\end{center}
\vspace{-2mm}
Suppose the structure of the world is accessible to both agents whereas the true location of the goal $h_t$ is only known to a teacher. The teacher performs a sequence of actions to teach $h_t$ to a learner. At each round, there are three available actions, \textit{left}, \textit{up} and \textit{right}. After observing the teacher's actions, the learner updates their belief on $h_t$ accordingly. 

We now compare BI and SCBI agents' behaviours under this grid world.
In terms of previous notations, the hypothesis set $\mathcal{H} = \{A, B\}$, the data set $\mathcal{D} = \{\textit{left}, \textit{up}, \textit{right}\}$.
Let the learner's prior over $\mathcal{H}$ be $\theta_0 = (0.5, 0.5)$ and the true hypothesis $h_t$ be $ A$, then at each \textcolor{blue!35}{blue} grid,  agents' (unnormalized) initial matrix $\tiny{\mathbf{M} = \begin{blockarray}{ccc}
&A & B  \\
\begin{block}{c (cc)}
 \text{left} & \gamma^{(k-1)} &\gamma^{(k+1)}  \\
 \text{up}& \gamma^{(k-1)} & \gamma^{(k-1)} \\
  \text{right}& \gamma^{(k+1)} & \gamma^{(k-1)} \\
\end{block}
\end{blockarray}}$. Assume both BI teacher and SCBI teacher start with grid~$S$. Based on $\mathbf{M}$, the BI teacher would choose equally between \textit{left} and \textit{up}, whereas the SCBI teacher is more likely to choose \textit{left} as the teacher's likelihood matrix 
%$\tiny{\mathbf{T} = \begin{pmatrix} \frac{2}{3+3\gamma^2} &  \frac{2\gamma^2}{3+3\gamma^2} \\ \frac{1}{3} & \frac{1}{3} \\ \frac{2\gamma^2}{3+3\gamma^2} & \frac{2}{3+3\gamma^2} \end{pmatrix}}$, 
$\tiny{\mathbf{T} = \begin{pmatrix} 2/(3+3\gamma^2) &  2\gamma^2 /(3+3\gamma^2) \\ 1/3 & 1/3 \\ 2\gamma^2 /(3+3\gamma^2) & 2/(3+3\gamma^2) \end{pmatrix}}$, obtained from Sinkhorn scaling on $\mathbf{M}$, assigns higher probability for \textit{left}. Hence, comparing to the BI teacher who only aims for the final goal, the SCBI teacher tends to cooperate with the learner by selecting less ambiguous moves towards the goal. This point is aligned with the core idea 
of many existing models of cooperation in cognitive development \cite{jara2016naive,bridgers2019young}, pragmatic reasoning \cite{frank2012predicting,goodman2013knowledge} and robotics \cite{ho2016showing,fisac2017pragmatic}.

Moreover, even under the same teaching data, the SCBI learner is more likely to infer $h_t$ than the BI learner.
For instance, given the teacher's trajectory $\{\textit{left}, \textit{up}\}$, 
the left plot in Fig.~\ref{fig:grid_world} shows the SCBI and BI learners' posteriors on the true hypothesis $h_t$. 
Hence, comparing to the BI learner who reads the teacher's action literally, 
the SCBI learner interprets teacher's data corporately by updating belief sequentially after each round.

%\pat{say something about how this scales with the board size? e.g. as the board gets wider...}
\begin{figure}[t]
  \centering
  \includegraphics[scale=0.08]{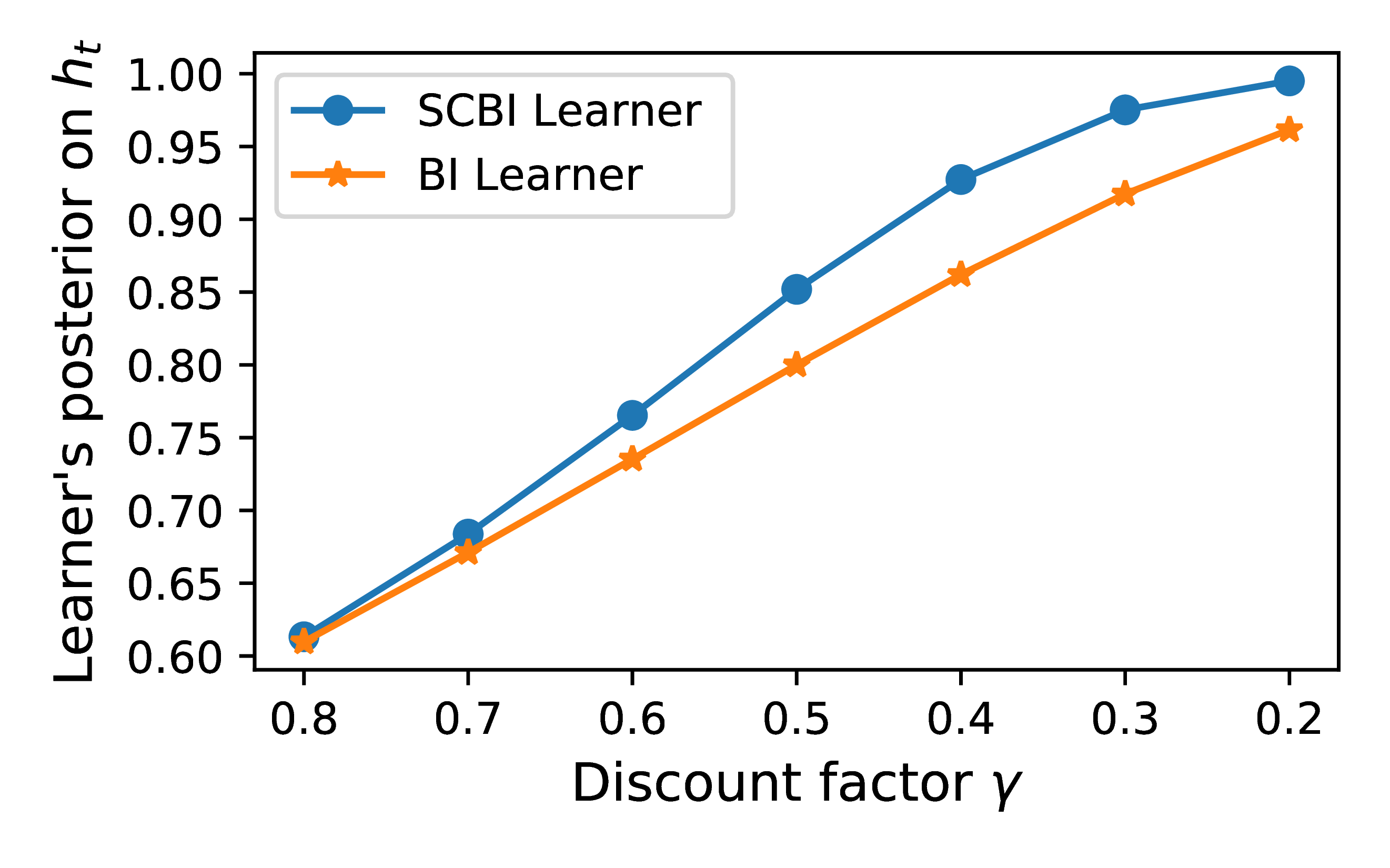}
  \includegraphics[scale=0.08]{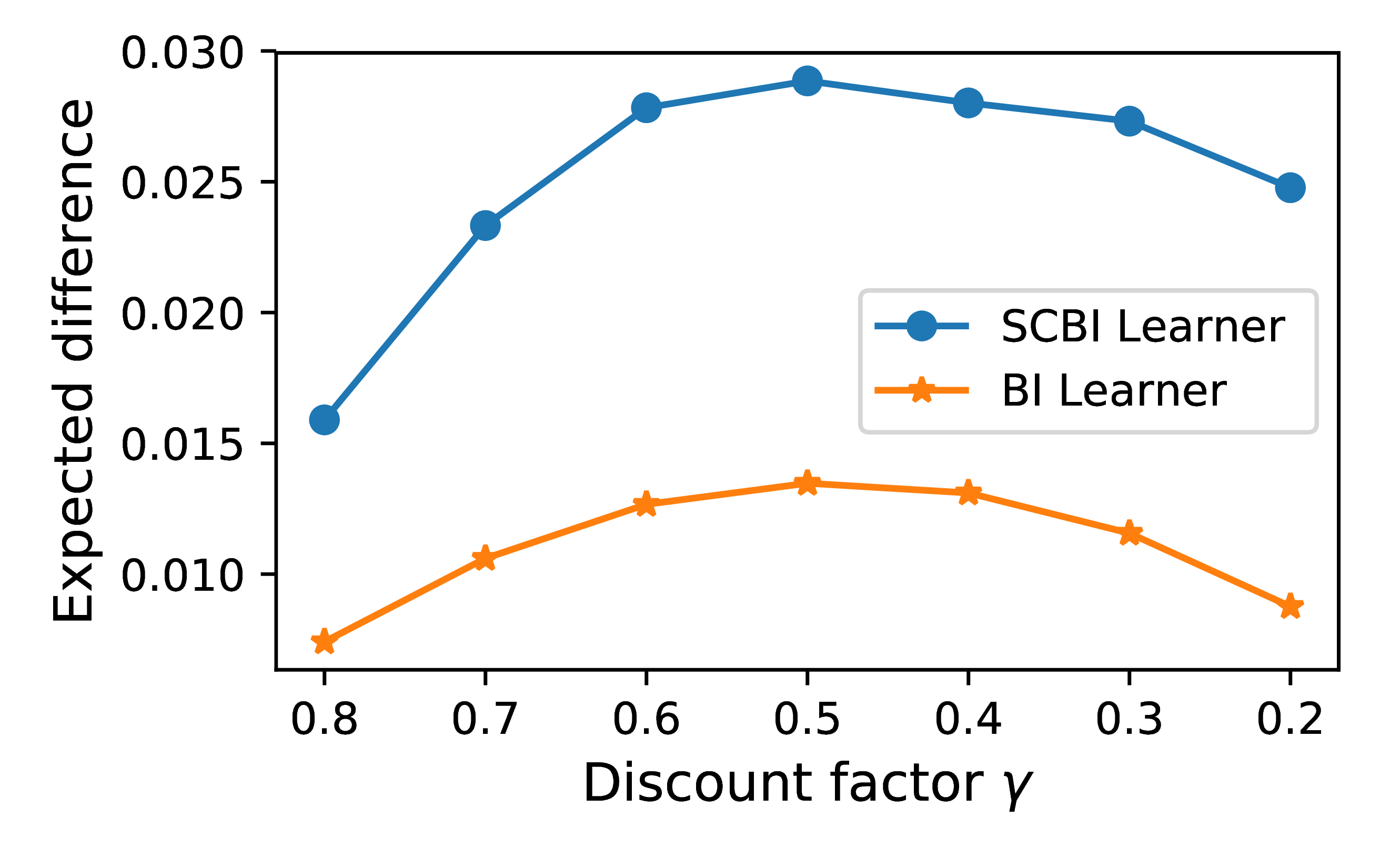}
    \vspace{-3mm}
  \caption{The top plot demonstrates that both BI and SCBI converge to the true hypothesis with SCBI having higher sample efficiency. 
  The bottom plot shows that both BI and SCBI agents are robust to perturbations with SCBI relatively less stable.}
  \vspace{-3mm}
  \label{fig:grid_world}
\end{figure}

Regarding the stability, consider the case where the learner's discount factor is either greater or less (with equal probability) 
than the teacher's by 0.1. The right plot in Fig.~\ref{fig:grid_world} illustrates the expected difference between the learner's posterior on $h_t$ 
after observing a teacher's trajectory of length $2$ and the teacher's estimation of it.

As discussed in Sec~\ref{sec:asy_roc_comp}, showing in Fig.~\ref{fig:roc},
as the board gets wider and the number of possible goals gets more (i.e. the number of hypotheses increases), 
the gap between posteriors of SCBI and BI learners will increase whereas the expected difference between agents for the same magnitude 
of perturbation will decrease. Thus, this example illustrates the consistency, sample efficiency, and stability of SCBI versus~BI.
%\pat{this is a perfect last paragraph of this section.}

%\pat{how does this compare with BI? how does this scale with wider boards? this is relevant to the smitha milli paper that we read a while back.}

% \pat{cand proably cut this paragraph}
% Thus, intuitively, comparing to the BI teacher who only aims for the final
% goal, the SCBI teacher tends to cooperate with the learner by selecting less ambiguous moves towards the goal. 
% Comparing to the BI learner who reads the teacher's action literally, 
% the SCBI learner interprets teacher's data corporately by updating his belief sequentially after each round.
% %with the learner by selecting less ambiguous moves towards the goal. 
% This captures the core idea of pragmatic reasoning and human-like inverse planing. 
% In fact, a single round of SCBI is aligned with many cooperative communication models.  \pw{include model and/or behavior evidences here?}

% Stability

\section{Related Work}\label{sec:related}

Literatures on Bayesian teaching \citep{Eaves2016a,Eaves2016c}, Rational Speech act theory \citep{frank2012predicting,goodman2013knowledge}, and machine teaching \citep{Zhu2015,zhu2013machine} consider the problem of selecting examples that improve a learner's chances of inferring a concept. These literatures differ in that they consider the single step, rather than sequential problem, that they do not formalize learners who reason about the teacher's selection process, and that they models without a mathematical analysis. 

The literature on pedagogical reasoning in human learning \citep{shafto2008teaching,shafto2012learning,Shafto2014} and cooperative inference \citep{YangYGWVS18,wang2018generalizing,wang2019mathematical} in machine learning formalize full recursive reasoning from the perspectives of both the teacher and the learner. These only consider the problem of a single interaction between the teacher and learner. 

The literature on curriculum learning considers sequential interactions with a learner by a teacher in which the teacher presents data in an ordered sequence \citep{bengio2009curriculum}, and traces back to various literatures on human and animal learning \citep{skinner1958teaching,elman1993learning}. Curriculum learning involves one of a number of methods for optimizing the sequence of data presented to the learner, most commonly starting with easier / simpler examples first and gradually moving toward more complex or less typical examples. Curriculum learning considers only problems where the teacher optimizes the sequence of examples, where the learner does not reason about the teaching. 

%Similarity between CL and CI: CL also proposes that in learning, the teacher shall gradually update the weight applied to each example.
%Formally, let $P(d|h)$ be the likelihood of $d$ given $h$ derived from the initial joint distribution, 
%and $0 \leq W_{\lambda}(d)\leq 1$ be the weight applied to $d$ at step $\lambda$ in the curriculum sequence, with $ 0\leq \lambda \leq 1$ and $W_{1}(d) = 1$. Thus at each step $\lambda$, the teacher teaches according to $Q_{\lambda}(d|h) \propto W_{\lambda}(d) P(d|h)$. 
%A sequence of $Q_{\lambda}$ is called a \textit{curriculum} if the entropy of these distributions increase $W_{\lambda}(d)$ is monotonically increasing \cite{bengio2009curriculum}.

%Difference: CL is a `pre-training strategy' that to start teaching with a small set of easy examples and slowly increase the magnitude and diversity of the data.
%Such approach is suggested by human learning, supported by empirical results and theoretically conjectured to work as well chosen curriculum strategy can act as a continuation method \cite{allgower2012numerical}. Whereas CBI is a general framework that focuses on cooperation. Consistency of CBI is carefully proved; 
%The teaching matrix at each step is completely determined by the existing data and common ground (does not involve artificial layer of choosing $W_{\lambda}(d)$).

%\jw{FIXME}
% curriculum learning

\section{Conclusions}\label{sec:conclusion}

Cooperation is central to learning in humans and machines. We set out to provide a mathematical foundation for sequential cooperative Bayesian inference (SCBI). We presented new analytic results demonstrating the consistency and asymptotic rate of convergence of SCBI. Empirically, we demonstrated the sample efficiency and stability to perturbations as compared to Bayesian inference, and illustrated with a simple reinforcement learning problem. We therefore provide strong evidence that SCBI satisfies basic desiderata. Future work will aim to provide mathematical proofs of the empirically observed efficiency and stability. 
% Conclusion

\appendix

\section{Proof of Consistency Theorems}
\subsection{Proof of Theorem \ref{sm:thm:BI_consistency}}

\begin{thm}\label{sm:thm:BI_consistency}[Theorem~\ref{thm:BI_consistency},\citep[Theorem 7.115]{miescke2008decision}]
  In BI, the sequence of posteriors $(S_k)$ is strongly consistent at
  $\widehat{\theta}=\delta_h$ for each $h\in\mathcal{H}$, with arbitrary choice
  of an interior point $\theta_0\in(\mathcal{P}(\mathcal{H}))^\circ$
  (i.e. $\theta_0(h)>0$ for all $h\in\mathcal{H}$) as prior.
\end{thm}

\begin{proof}
  We follow the same line as discussed right after this theorem in the paper.
  Let $\theta_0=(\theta_0(1),\theta_0(2),\dots,\theta_0(n))$ be the original prior, and let
  $\theta_k=(\theta_{k}(1),\theta_{k}(2),\dots,\theta_{k}(m))$ be the posterior after having $k$
  data points $d_1,d_2,\dots, d_k$. Then for $l\le k$ and $h\in\mathcal{H}$,
  the posterior $\theta_{l}(i)=\left( \mathscr{N}_{\text{vec}}(\mathrm{diag}(\mathbf{M}_{(d_l,\_)})\theta_{l-1})
  \right)(i)$ by Bayes' rule. %\pw{check with jq} 
  In other words,
  \begin{equation}
    \theta_{l}(i)=\dfrac{\mathbf{M}_{(d_l,i)}[\theta_{(l-1)}(i)]}
    {\sum_{j=1}^{m}\mathbf{M}_{(d_l,j)}[\theta_{(l-1)}(j)]}.
  \end{equation}
  This is a recursive formula, so we may move forward to calculate $\theta_{l}(i)$
  from a smaller round index $\theta_{t}(i)$ with $t<l$:
  \begin{equation}
    \theta_{l}(i)=\dfrac{\left[\prod_{s=t}^{l}\mathbf{M}_{(d_s,i)}\right]\theta_{(t-1)}(i)}
    {\sum_{j=1}^{m}\left[\prod_{s=t}^{l}\mathbf{M}_{(d_s,j)}\right]\theta_{(t-1)}(j)}.\nonumber
  \end{equation}
  This recursion stops at prior $\theta_0$, so we have an explicit expression of
  $\theta_k$:
  \begin{equation}
    \label{sm:eq:BI_explicit_formula}
    \theta_{k}(i)=\dfrac{\left[\prod_{s=1}^{k}\mathbf{M}_{(d_s,i)}\right]\theta_{0}(i)}
    {\sum_{j=1}^{m}\left[\prod_{s=1}^{k}\mathbf{M}_{(d_s,j)}\right]\theta_{0}(j)}. 
  \end{equation}
  It can be seen that for each hypothesis $i$, the denominator of
  the $k$-th  posterior on $i$ are the same, so we have
  \begin{equation}
    \label{sm:eq:posterior_ratio}
    \dfrac{\theta_{k}(i)}{\theta_{k}(h)}=
    \dfrac{\left[\prod_{s=1}^{k}\mathbf{M}_{(d_s,i)}\right]\theta_{0}(i)}
    {\left[\prod_{s=1}^{k}\mathbf{M}_{(d_s,h)}\right]\theta_{0}(h)}.
  \end{equation}
  So we define $\alpha_k(d)$ to be the frequency of the occurrence of
  data $d$ in the first $k$ rounds of a episode. And then
  \begin{equation}\small
    \label{sm:eq:log_posterior_ratio}
    \log\left(\dfrac{\theta_{k}(i)}{\theta_{k}(h)}\right)=\log\left(\dfrac{\theta_{0}(i)}{\theta_{0}(h)}\right)+
    \sum_{d=1}^{n}\alpha_k(d)\log\left(\dfrac{\mathbf{M}_{(d,i)}}{\mathbf{M}_{(d,h)}}\right).
  \end{equation}
  Since we know that the data $(d_i)$ in the model is sampled following the
  i.i.d. with distribution $\mathbf{M}_{(\_,h)}$, then for a fixed $k$,
  $\alpha_k(i)$ follows the multinomial distribution with parameter
  $\mathbf{M}_{(\_,h)}$.

  By the strong law of large numbers,
  $\frac{\alpha_k(i)}{k}\rightarrow\mathbf{M}_{(i,h)}$ almost surely as
  $k\rightarrow\infty$. 
  Thus, when we rewrite the sample values 
  to random variable version,
  \begin{equation}
    \label{sm:eq:average_log_posterior_1}
    \dfrac{1}{k}\log\left(\dfrac{\Theta_{k}(i)}{\Theta_{k}(h)}\right)\rightarrow
    \sum_{d=1}^{n}\mathbf{M}_{(d,h)}\log\left(\dfrac{\mathbf{M}_{(d,i)}}{\mathbf{M}_{(d,h)}}\right)
    \quad\text{a.s.}
  \end{equation}
  
  That is,
  \begin{equation}
    \label{sm:eq:average_log_posterior}
    \dfrac{1}{k}\log\left(\dfrac{\Theta_{k}(i)}{\Theta_{k}(h)}\right)\longrightarrow
    -\mathrm{KL}\left(\mathbf{M}_{(\_,h)},\mathbf{M}_{(\_,i)}\right)
    \quad\text{a.s.}
  \end{equation}

  By the assumption in Section~2 of the paper that $\mathbf{M}$ has distinct
  columns, the KL divergence between the $i$-th column and the $h$-th column
  is strictly positive, thus almost surely,
  $\log\left(\dfrac{\Theta_{k}(i)}{\Theta_{k}(h)}\right)\rightarrow-\infty$, or
  equivalently, $\dfrac{\Theta_{k}(i)}{\Theta_{k}(h)}\rightarrow0$, for any $i\ne h$.

  Therefore, $\theta_k=(\theta_{k}(1),\theta_{k}(2),\dots,\theta_{k}(m))\rightarrow\delta_h$
  almost surely, equivalently, BI at $\widehat{\theta}$ is strongly consistent.
\end{proof}

\subsection{Proof of Theorem~\ref{thm:roc_bi}}

\begin{thm}[Theorem~\ref{thm:roc_bi}]
  \label{sm:thm:RoC_BI}  
  In BI, with $\widehat{\theta}=\delta_h$ for some $h\in\mathcal{H}$,
  let $\Theta_k(h)(\mathrm{D}_1,\dots,\mathrm{D}_k):=S_k(h|\mathrm{D}_1,\dots,\mathrm{D}_k)$ be the $h$-component of
  posterior given $\mathrm{D}_1,\dots,\mathrm{D}_k$ as random variables valued in $\mathcal{D}$. Then
  $\dfrac{1}{k}\log\left(\dfrac{\Theta_{k}(h)}{1-\Theta_{k}(h)}\right)$ converges to a constant
  $\mathrm{min}_{h'\ne h}\left\{\mathrm{KL}(\Mcol{M}{h},\Mcol{M}{h'})\right\}$
  almost surely.
\end{thm}
\begin{proof}
  Follow the previous proof. First recall that
  $\dfrac{1}{k}\log\left(\dfrac{\Theta_{k}(i)}{\Theta_{k}(h)}\right)\rightarrow
  -\mathrm{KL}(\Mcol{M}{h},\Mcol{M}{i})$ almost surely. Let
  $\eta:=\mathrm{argmin}_{i\ne h}
  \left\{\mathrm{KL}(\Mcol{M}{h},\Mcol{M}{i})\right\}$, then $\Theta_{k}(\eta)$ decays
  slowest among $\left\{\Theta_{k}(i):i\ne h\right\}$ almost surely.

  Therefore, for the sample values $\theta_k$'s, asymptotically,
  \begin{equation}
    \label{sm:eq:RoC_bounds}\small
    \dfrac{1}{k}\log\left[\dfrac{\theta_{k}(\eta)}{\theta_{k}(h)}\right]\!\le\!
    \dfrac{1}{k}\log\left[\dfrac{1-\theta_{k}(h)}{\theta_{k}(h)}\right]\!\le\!
    \dfrac{1}{k}\log\left[\dfrac{(m-1)\theta_{k}(\eta)}{\theta_{k}(h)}\right].\nonumber
  \end{equation}
  So when we are taking limits $k\rightarrow\infty$, with probability one, we
  have
  \begin{eqnarray}
    \label{sm:eq:estimation_log_posteriors}\small
    &&-\mathrm{KL}(\Mcol{M}{h},\Mcol{M}{\eta})\le
    \lim_{k\rightarrow\infty}\dfrac{1}{k}\log\left[\dfrac{1-\theta_{k}(h)}{\theta_{k}(h)}\right]\nonumber\\
    &&\le
    \lim_{k\rightarrow\infty}-\mathrm{KL}(\Mcol{M}{h},\Mcol{M}{\eta})+\dfrac{1}{k}\log(m-1)\nonumber\\
    &&=-\mathrm{KL}(\Mcol{M}{h},\Mcol{M}{\eta}).
  \end{eqnarray}
\end{proof}

\subsection{Proof of Theorem~\ref{thm:consistency_SCBI}}

To prove Theorem~\ref{sm:thm:consistency_SCBI}, we need the following lemmas.

\begin{smlem}
  \label{sm:lem:expectation_inequality}
  Given a fixed hypothesis $h\in\mathcal{H}$, for any $\mu\in\mathcal{P}(\Delta^{m-1})$,
  \begin{equation}
    \label{sm:eq:e_mu_smaller_e_psi_mu}
    \mathbb{E}_{\mu}(\theta(h))\le\mathbb{E}_{\Psi(h)(\mu)}(\theta(h)).
  \end{equation}
  equality happens when
  $\mathbf{M}^{\left\langle n\mathbf{x}\right\rangle}_{\phantom{~~}(i,h)}
  =\mathbf{M}^{\left\langle n\mathbf{x}\right\rangle}_{\phantom{~~}(j,h)}$ for any
  $i$, $j$ and $\mu$-almost everywhere for $\mathbf{x}\in\Delta^{m-1}$.
\end{smlem}
\begin{remark}
  This lemma shows that the expectation of $\theta(h)$, in each
  round is increasing, thus the sequence obtained from all the rounds
  has an limit since the sequence is monotonic and upper bounded by 
  $1$. To prove the theorem we, then just need to show the limit is $1$.
\end{remark}

\begin{proof}
  % First we rewrite Eq.~\eqref{sm:eq:e_mu_smaller_e_psi_mu} explicitly.
  % \begin{equation}
  %   \label{sm:eq:explicit_exp_psi_mu}
  %   \int_{\Delta^{m-1}}\theta(h)\mathrm{d}\mu(\theta)\le
  %   \int_{\Delta^{m-1}}\theta(h)\mathrm{d}(\Psi(h)(\mu))(\theta)
  % \end{equation}
  
  We start from the right hand side of Eq.~\ref{sm:eq:e_mu_smaller_e_psi_mu}.
  Let $\Delta$ denote $\Delta^{m-1}$ for short.
  \begin{eqnarray}
    \label{sm:eq:calculation}
    && \mathbb{E}_{\Psi(h)(\mu)}(\theta(h))\nonumber\\
    &=&\!\!\!\!
    \int_{\Delta}\theta(h)\mathrm{d}(\Psi(h)(\mu))(\theta)\nonumber\\
    &=&\!\!\!\!\int\limits_{\Delta}\sum_{d=1}^{n}
        \tau_d(T_d^{-1}(\theta))\theta(h)\mathrm{d}(T_{d\ast}(\mu))(\theta)\nonumber\\
    &=&\!\!\!\!\sum_{d=1}^{n}\int\limits_{\Delta}\tau_d(\theta)(T_d(\theta))(h)\mathrm{d}
        \left(T_{d\ast}(\mu)\right)(T_d(\theta))\nonumber\\
    &=&\!\!\!\!\sum_{d=1}^{n}\int\limits_{\Delta}\tau_d(\theta)(T_d(\theta))(h)\mathrm{d}
        \mu(\theta)\nonumber\\
    &=&\!\!\!\!\sum_{d=1}^{n}\int\limits_{\Delta}\dfrac{T_d(\theta)(h)}{n\theta(h)}T_d(\theta)(h)
        \mathrm{d}\mu(\theta)\nonumber\\
    &=&\!\!\!\!\int\limits_{\Delta}\sum_{d=1}^{n}\dfrac{T_d(\theta)(h)^2}{n\theta(h)}
        \mathrm{d}\mu(\theta)\nonumber
  \end{eqnarray}
  In the calculation, the bijectivity of $T_d$ and the formula
  $(T_{d\ast}(\mu))(E)=\mu(T_d^{-1}(E))$ is used (and will be used 
  repetitively later).
  
  Consider that by definition of the bijection $T_d$, the sum $\sum_{d=1}^{n}T_d(\theta)(h)
  = n\theta(h)$ ($T_d$ is the $d$-th row of Sinkhorn scaling by column sums
  $n\theta$).
  Thus
  \begin{eqnarray}
    \label{sm:eq:calculation2}
    \mathbb{E}_{\Psi(h)(\mu)}(\theta(h)) % \nonumber\\
    &=&\!\!\!\!\int\limits_{\Delta}\dfrac{\sum_{d=1}^{n}T_d(\theta)(h)^2}{\sum_{d=1}^{n}T_d(\theta)(h)}
        \mathrm{d}\mu(\theta)\nonumber\\
    &\ge&\!\!\!\!\int\limits_{\Delta}\dfrac{\left(\sum_{d=1}^{n}
        T_d(\theta)(h)\right)^2}{n\sum_{d=1}^{n}T_d(\theta)(h)}
        \mathrm{d}\mu(\theta)\nonumber\\
    &=&\!\!\!\!\int\limits_{\Delta}\dfrac{1}{n}\sum_{d=1}^{n}T_d(\theta)(h)
        \mathrm{d}\mu(\theta)\nonumber\\
    &=&\!\!\!\!\int\limits_{\Delta}\theta(h)
        \mathrm{d}\mu(\theta)\nonumber\\
    &=&\mathbb{E}_{\mu}(\theta(h)),
  \end{eqnarray}
  where $\sum_{d=1}^{n}T_d(\theta)(h)^2\ge\dfrac{1}{n}\left(
    \sum_{d=1}^{n}T_d(\theta)(h)\right)^2$ by Cauchy-Schwarz inequality, with
  equality achieved if and only if $T_d(\theta)(h)$ is constant on $d$.
  Therefore, the equality of Eq.~\eqref{sm:eq:calculation2} is achieved when
  $\mathbf{M}^{\left\langle n\mathbf{x}\right\rangle}_{\phantom{~~}(d,h)}$ is
  constant on $d$, $\mu$-almost everywhere for $\mathbf{x}\in\Delta^{m-1}$.
  
  % According to Lebesgue decomposition theorem, it suffices to prove this lemma
  % in three 
  % cases: a Dirac delta measure $\delta_{\theta}$, a singular continuous
  % measure $\mu_s$, and an absolutely continuous measure $\mu_a$.
\end{proof}
\begin{figure}[!ht]
    \centering
    \begin{tikzpicture}
      \foreach \i/\t in {0/0, 4/1, 1/a, 3.2/b} {
       \draw[dashed] (-2.5,\i) -- (3,\i);
       \node at (-3.3,\i) {$\theta(h)=\t$};}
       \node at (0, 1.8) {\small Lemma~\ref{sm:lem:middle_vacuum}};
       \node at (0, 0.5) {Lemma~\ref{sm:lem:expectation_property}};
       \draw[stealth-stealth] (2,3.2) -- (2,4) node[pos=0.5,fill=white]{$\epsilon$};
       \draw (0,4) -- (-2.31,0) -- (2.31,0) -- (0,4);
    \end{tikzpicture}
    \caption{Sketch of $\Delta^{m-1}$, 
    for a general $\theta$, its $y$-coordinate is
    $\theta(h)$. The levels are compatible with proof of 
    Theorem~\ref{sm:thm:consistency_SCBI}. 
    Lemma~\ref{sm:lem:expectation_property} and 
    Lemma~\ref{sm:lem:middle_vacuum} are located where they
    contribute to prove the vanishing of measure in the limit.
    }
    \label{sm:fig:sketch}
\end{figure}
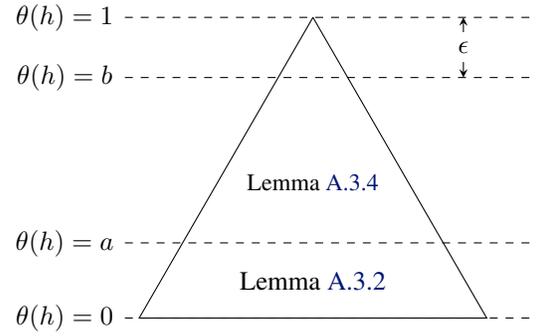
The following lemmas helps showing that
the measure $\mu_k$ of the complement of a neighborhood of $\delta_h\in\mathcal{P}(\mathcal{H})$ has limit $0$.
\begin{smlem}
  \label{sm:lem:expectation_property}
  Given $\mathbf{M}$, $h\in\mathcal{H}$ and prior
  $\mu_0\in\mathcal{P}(\Delta^{m-1})$ satisfying assumptions, % ASSUMPTIONS
  we have
  \begin{equation}
    \label{sm:eq:invariant_expectation}
    \mathbb{E}_{\mu_k}\left(\dfrac{\theta(h')}{\theta(h)}\right)=
    \mathbb{E}_{\mu_0}\left(\dfrac{\theta(h')}{\theta(h)}\right)  
  \end{equation}
  for any $k\ge0$ and any $h'\ne h$.
\end{smlem}

\begin{proof}
  It suffices to prove the $k=1$ case for a general $\mu_0$ (then we have the
  rest by induction).

  \begin{eqnarray}
    \label{sm:eq:expectation_of_ratio}
    &&\mathbb{E}_{\mu_1}\left(\dfrac{\theta(h')}{\theta(h)}\right)=
    \int\limits_{\Delta}\left(\dfrac{\theta(h')}{\theta(h)}\right)\mathrm{d}\mu_1(\theta)
    \nonumber\\
    &=&\int\limits_{\Delta}\left(\dfrac{\theta(h')}{\theta(h)}\right)\mathrm{d}(\Psi(h)(\mu_0))(\theta)\nonumber\\
    &=&\int\limits_{\Delta}\sum_{d\in\mathcal{D}}\tau_d(T_d^{-1}(\theta))\dfrac{\theta(h')}{\theta(h)}\mathrm{d}(T_{d\ast}(\mu_0))(\theta)\nonumber\\
    &=&\sum_{d\in\mathcal{D}}\int\limits_{\Delta}\tau_d(\theta)\dfrac{T_d(\theta)(h')}{T_d(\theta)(h)}\mathrm{d}(T_{d\ast}(\mu_0))(T_d(\theta))\nonumber\\
    &=&\sum_{d\in\mathcal{D}}\int\limits_{\Delta}\dfrac{T_d(\theta)(h)}{n\theta(h)}\dfrac{T_d(\theta)(h')}{T_d(\theta)(h)}\mathrm{d}(\mu_0)(\theta)\nonumber\\
    &=&\int\limits_{\Delta}\sum_{d\in\mathcal{D}}\dfrac{T_d(\theta)(h')}{n\theta(h)}\mathrm{d}(\mu_0)(\theta)\nonumber\\
    &=&\int\limits_{\Delta}\dfrac{\sum_{d\in\mathcal{D}}T_d(\theta)(h')}{n\theta(h)}\mathrm{d}(\mu_0)(\theta)\nonumber\\
    &=&\int\limits_{\Delta}\dfrac{n\theta(h')}{n\theta(h)}\mathrm{d}(\mu_0)(\theta)\nonumber\\
    &=&\mathbb{E}_{\mu_0}\left(\dfrac{\theta(h')}{\theta(h)}\right).
  \end{eqnarray}
\end{proof}

\begin{smlem}
  \label{sm:lem:linearity_of_Psi}
  The operator $\Psi(h)$ preserves convex combinations of probability measures,
  i.e., for positive $a_1,a_2,\dots,a_l$ with $\sum_{i=1}^{l}a_i=1$ and
  probability measures $\mu_1,\mu_2,\dots,\mu_l$,
  $$\Psi(h)\left(\sum_{i=1}^la_i\mu_i\right)=\sum_{i=1}^{l}a_i\Psi(h)(\mu_i).$$
\end{smlem}

\begin{proof}
  By definition, for any measurable set $E$ in Borel $\sigma$ algebra $\mathfrak{A}$,
  $$\Psi(h)(\mu)(E):=\int_E\sum_{d=1}^{n}\tau_d(T^{-1}_d(\theta))\mathrm{d}(T_{d\ast}(\mu))(\theta).$$
  where every summand commutes with convex combination.
\end{proof}

\begin{smlem}
  \label{sm:lem:middle_vacuum}
  Given $\mathbf{M}$, $h$, and $\mu_0$ satisfying the assumptions,
  then for any $0<a<b<1$,
  \begin{eqnarray}
    \label{sm:eq:vacuum}
    \lim_{k\rightarrow\infty}\mu_k(\{\theta\in\Delta^{m-1}:a\le\theta(h)\le b\})=0
  \end{eqnarray}
\end{smlem}

\begin{proof}
  We first show a property of $\mu$ on the set
  $\Delta_{[a,b]}:=\{\theta\in\Delta^{m-1}:a\le\theta(h)\le b\}$.

  For any $\mu$ supported on $\Delta_{[a,b]}$ (that is,
  $\mu(\Delta_{[a,b]})=1$), there 
  is a positive number $\epsilon_0$, such that
  \begin{equation}
    \label{sm:eq:lower_bound}
    \mathbb{E}_{\Psi^2(h)(\mu)}(\theta(h))-\mathbb{E}_{\mu}(\theta(h))\ge\epsilon_0.
  \end{equation}

  According to the calculation in Lemma~\ref{sm:lem:expectation_inequality},
  especially the first step of Eq.~\eqref{sm:eq:calculation2},
  \begin{eqnarray}
    \label{sm:eq:double_action_expectation}
    &&\mathbb{E}_{\Psi^2(h)(\mu)}(\theta(h)) \nonumber\\
    &=&\!\!\!\!\int\limits_{\Delta}\dfrac{\sum_{d=1}^{n}T_d(\theta)(h)^2}{\sum_{d=1}^{n}T_d(\theta)(h)}
        \mathrm{d}(\Psi(h)(\mu))(\theta)\nonumber\\
    &=&\!\!\!\!\int\limits_{\Delta}\sum_{e=1}^{n}\tau_e(T_e^{-1}(\theta))
        \dfrac{\sum_{d=1}^{n}T_d(\theta)(h)^2}{\sum_{d=1}^{n}T_d(\theta)(h)}
        \mathrm{d}(T_{e\ast}(\mu))(\theta)\nonumber\\
    &=&\!\!\!\!\int\limits_{\Delta}\sum_{e=1}^{n}\tau_e(\theta)
        \dfrac{\sum_{d=1}^{n}T_d(T_e(\theta))(h)^2}{\sum_{d=1}^{n}T_d(T_e(\theta))(h)}
        \mathrm{d}\mu(\theta)
  \end{eqnarray}
  Thus
  \begin{eqnarray}
    \label{sm:eq:lower_bound_integral}
    \!\!\!\!\!\!\!\!\!\!&&\mathbb{E}_{\Psi^2(h)(\mu)}(\theta(h))-\mathbb{E}_{\mu}(\theta(h))\nonumber\\
    \!\!\!\!\!\!\!\!\!\!&=&\!\!\!\!\int\limits_{\Delta}\sum_{e=1}^{n}\tau_e(\theta)
        \dfrac{\sum\limits_{d=1}^{n}T_d(T_e(\theta))(h)^2}{\sum\limits_{d=1}^{n}T_d(T_e(\theta))(h)}
        -\theta(h)\mathrm{d}\mu(\theta)
  \end{eqnarray}
  
  To show the claim, it suffices to find a positive lower bound of the integrand
  of Eq.~\eqref{sm:eq:lower_bound_integral}, $\mathfrak{I}(\theta):=\sum_{e=1}^{n}\tau_e(\theta) 
  \dfrac{\sum_{d=1}^{n}T_d(T_e(\theta))(h)^2}{\sum_{d=1}^{n}T_d(T_e(\theta))(h)}-\theta(h)$,
  for all $\theta\in\Delta_{[a,b]}$. Moreover, since $\Delta_{[a,b]}$ is
  compact, we just need to show $\mathfrak{I}(\theta)>0$ on $\Delta_{[a,b]}$.

  With Cauchy-Schwarz inequality used in
  Lemma~\ref{sm:lem:expectation_inequality}, we know

  \begin{eqnarray}
    \label{sm:eq:non_zero}
    \mathfrak{I}(\theta)\!\!\!\!&=&\!\!\!\!\sum_{e=1}^{n}\tau_e(\theta) 
    \dfrac{\sum_{d=1}^{n}T_d(T_e(\theta))(h)^2}{\sum_{d=1}^{n}T_d(T_e(\theta))(h)}-\theta(h)\nonumber\\
    &\ge&\!\!\!\!\sum_{e=1}^{n}\tau_e(\theta)\dfrac{1}{n}\left(\sum_{d=1}^{n}T_d(T_e(\theta))(h)\right)-\theta(h)\nonumber\\
    &=&\!\!\!\!\sum_{e=1}^{n}\tau_e(\theta)T_e(\theta)(h)-\theta(h)\nonumber\\
    &=&\!\!\!\!\sum_{e=1}^{n}\dfrac{T_e(\theta)(h)}{n\theta(h)}T_e(\theta)(h)-\theta(h)\nonumber\\
    &\ge&\!\!\!\!\dfrac{1}{n}T_e(\theta)(h)-\theta(h)\nonumber\\
    &=&\!\!\!\!\theta(h)-\theta(h)=0
  \end{eqnarray}
  $\mathfrak{I}(\theta)$ vanishes if and only if both line $2$ and line $5$ has
  equality, and we will discuss why these can not happen simultaneously.
  
  The equality in line $5$ requires that $T_e(\theta)(h)$ are identical for all
  $e\in\mathcal{D}$, or more precisely, the vector $\mathbf{M}^{\left\langle
    n\theta\right\rangle}_{\phantom{~~}(\_,h)}=t\mathbf{e}_n$ has identical
  components. Further if equality in line $2$ holds, the terms
  $T_d(T_e(\theta))(h)$ are the same for all $d\in\mathcal{D}$. That is, by
  condition $\sum_{e=1}^{n}T_e(\theta)(h)=n\theta(h)$ and
  $\sum_{d=1}^{n}T_d(T_e(\theta))(h)=nT_e(\theta)(h)$,
  % \pw{maybe missing a coefficient n.}we can conclude that
  $\mathfrak{I}(\theta)$ vanishes if and only if
  $T_d(T_e(\theta))(h)=T_e(\theta)(h)=\theta(h)$ for all $d,e\in\mathcal{D}$.

  We analyze the Sinkhorn scaled matrices in detail: Let
  $\mathbf{M}^\star=\mathbf{M}^{\left\langle n\theta\right\rangle}$ be the scaled
  matrix whose $e$-th row is $T_e(\theta)$, and let
  $\mathbf{M}^{(e)}=\mathbf{M}^{\left\langle nT_e(\theta)\right\rangle}$ be the 
  scaled matrix whose $d$-th row is $T_d(T_e(\theta))$. Since $\mathbf{M}^\star$
  and each $\mathbf{M}^{(e)}$ has the same $h$-th column, there are diagonal
  matrices $\mathbf{D}^{(e)}=\mathrm{diag}(
  \frac{n\mathbf{M}^\star_{(e,1)}}{\sum_{i=1}^{n}\mathbf{M}^\star_{(i,1)}}, 
  \frac{n\mathbf{M}^\star_{(e,2)}}{\sum_{i=1}^{n}\mathbf{M}^\star_{(i,2)}},
  \dots,
  \frac{n\mathbf{M}^\star_{(e,n)}}{\sum_{i=1}^{n}\mathbf{M}^\star_{(i,n)}})$
  such that $\mathbf{M}^{(e)}=\mathbf{M}^\star\mathbf{D}^{(e)}$.
  Since $\mathbf{M}^\star$ and all $\mathbf{M}^{(e)}$ are row-normalized to
  $\mathbf{e}$ (i.e., their row sums are $1$), we have the following equations
  from the row sums:
  \begin{eqnarray}
    \label{sm:eq:row_sums}
    \mathfrak{S}(d,e):=\sum_{j=1}^{m}\mathbf{M}^\star_{(d,j)}\dfrac{n\mathbf{M}^\star_{(e,j)}}{\sum_{i=1}^{n}\mathbf{M}^\star_{(i,j)}}=1
  \end{eqnarray}
  for all $d,e\in\mathcal{D}$ representing the $d$-th row-sum of
  $\mathbf{M}^{(e)}$.

  Then we calculate $(n-1)\sum_{e=1}^{n}\mathfrak{S}(e,e)-\sum_{d\ne e}\mathfrak{S}(d,e)$.
  On the right hand side, since $\mathfrak{S}(d,e)=1$ for every $d,e$, we have
  $$(n-1)\sum_{e=1}^{n}\mathfrak{S}(e,e)-\sum_{d\ne e}\mathfrak{S}(d,e)=(n-1)n-(n^2-n)=0.$$
  Meanwhile,
  \begin{eqnarray}
    \label{sm:eq:complete_squares}
    &&(n-1)\sum_{e=1}^{n}\mathfrak{S}(e,e)-\sum_{d\ne e}\mathfrak{S}(d,e)\nonumber\\
    &=&\!\!\!\!\sum_{j=1}^{m}\!\dfrac{n}{\sum\limits_{i=1}^{n}\mathbf{M}^\star_{(i,j)}}\left(
        \sum\limits_{e=1}^{n}(n-1)(\mathbf{M}^\star_{(e,j)})^2\right.\nonumber\\
      & &  \left.-\sum_{d\ne e}\mathbf{M}^\star_{(d,j)}\mathbf{M}^\star_{(e,j)}
                      \right)\nonumber\\
    &=&\!\!\!\!\sum_{j=1}^{m}\!\dfrac{n}{\sum\limits_{i=1}^{n}\mathbf{M}^\star_{(i,j)}}\left(
        \sum_{d<e}(\mathbf{M}^\star_{(d,j)}-\mathbf{M}^\star_{(e,j)})^2\right)\nonumber\\
    &=&0
  \end{eqnarray}
  Therefore, $\mathbf{M}^\star_{(d,j)}=\mathbf{M}^\star_{(e,j)}$ for any $d,e$,
  and $j$. Therefore, the rows of $\mathbf{M}^\star$ are identical, so the
  columns of $\mathbf{M}^\star$ are all parallel (or say, collinear as vectors, i.e. one is a scalar-multiple of the other) to each other.
%  \pw{parallel used before?}
  
  By Sinkhorn scaling theory \cite{fienberg1970iterative}, the cross-ratios are
  invariant. Since $\mathbf{M}$ is a positive matrix and has distinct (non-parallel)
  columns, the $2\times2$ cross-ratios are not identically $1$, however,
  $\mathbf{M}^\star$ --- a scaled matrix of $\mathbf{M}$ --- has cross ratios
  identically $1$. Therefore our assumption that $\mathfrak{I}(\theta)=0$ cannot
  happen, and by compactness of $\Delta_{[a,b]}$ and continuity of
  $\mathfrak{I}(\theta)$, we can conclude that $\mathfrak{I}(\theta)$ has a
  lower bound $\epsilon_0>0$ on $\Delta_{[a,b]}$. 
  
  Therefore,

  \begin{eqnarray}
    \label{sm:eq:positive_increasing_expectation}
    &&\mathbb{E}_{\Psi^2(h)(\mu)}(\theta(h))-\mathbb{E}_{\mu}(\theta(h))\nonumber\\
    &=&\!\!\!\!\int\limits_{\Delta}\mathfrak{I}(\theta)\mathrm{d}\mu(\theta)\nonumber\\
    &\ge&\!\!\!\!\int\limits_{\Delta}\epsilon_0\mathrm{d}\mu(\theta)\nonumber\\
    &=&\epsilon_0\mu(\Delta)=\epsilon_0.
  \end{eqnarray}

  Thus we prove the property Eq.~\eqref{sm:eq:lower_bound}.
  %%%  BREAK POINT!!!
  
  We prove the lemma by contradiction:

  Suppose the limit does not exist or the limit is nonzero. In either case,
  there exists a positive real number $\epsilon>0$, such that there are
  infinitely many integers, or say a sequence $(k_i)$ such that
  $$\mu_{k_i}(\{a\le\theta(h)\le b\})>\epsilon.$$
  We may assume $k_i$ contains no consecutive elements, i.e., $k_{i+1}-k_i>1$
  for all $i$, otherwise, we can always find a subsequence satisfying this (for
  example, choose the sequence of all odd or even $k_i$'s, at least one of them
  is infinite, so we have a sequence).

  For a $\mu$-measurable set
  $E$, let $\mu|_E$ be the restriction of $\mu$ on $E$, which can be treated as
  a measure on $\Delta$ by setting the measure of the complement
  $E^c$ zero (but the measure of $\Delta$ is no longer $1$). We scale it to 
  $\widehat{\mu}|_E:=(\mu(E))^{-1}\mu|_E$ to make it a probability measure, then
  $\mu_{k_i}=[\mu_{k_i}(\Delta_{[a,b]})]\widehat{\mu}_{k_i}|_{\Delta_{[a,b]}}+[1-\mu_{k_i}(\Delta_{[a,b]})]\widehat{\mu}_{k_i}|_{(\Delta-\Delta_{[a,b]})}$.
  Thus according to Lemma~\ref{sm:lem:linearity_of_Psi}, 
  \begin{eqnarray}
    &&\mathbb{E}_{\Psi^2(h)(\mu_{k_i})}(\theta(h))\nonumber\\
    &=&\mu_{k_i}(\Delta_{[a,b]})\mathbb{E}_{\Psi^2(h)(\widehat{\mu}_{k_i}|_{\Delta_{[a,b]}})}(\theta(h))\nonumber\\
    &&+(1-\mu_{k_i}(\Delta_{[a,b]}))\mathbb{E}_{\Psi^2(h)(\widehat{\mu}_{k_i}|_{(\Delta-\Delta_{[a,b]})})}(\theta(h))\nonumber\\
    &\ge&\mu_{k_i}(\Delta_{[a,b]})(\mathbb{E}_{\widehat{\mu}_{k_i}|_{\Delta_{[a,b]}}}(\theta(h))+\epsilon_0)\nonumber\\
    &&+(1-\mu_{k_i}(\Delta_{[a,b]}))\mathbb{E}_{\widehat{\mu}_{k_i}|_{(\Delta-\Delta_{[a,b]})}}(\theta(h))\nonumber\\
    &\ge&\epsilon\epsilon_0+\mathbb{E}_{\mu_{k_i}}(\theta(h))
  \end{eqnarray}

  By Lemma~\ref{sm:lem:expectation_inequality}, we can see that
  $\mathbb{E}_{\mu^{\phantom{A}}_{k+1}}[\theta(h)]\ge
  \mathbb{E}_{\mu^{\phantom{A}}_{k}}[\theta(h)]$,
  and there is a sequence $(k_i)$ such that
  $\mathbb{E}_{\mu^{\phantom{A}}_{k_{i}+2}}[\theta(h)]\ge
  \mathbb{E}_{\mu^{\phantom{A}}_{k_i}}[\theta(h)]+\epsilon_0\epsilon$.
  Thus
  $\mathbb{E}_{\mu^{\phantom{A}}_{k_{i}+2}}[\theta(h)]\ge
  \mathbb{E}_{\mu^{\phantom{A}}_0}[\theta(h)]+i\epsilon_0\epsilon$,
  so
  $\lim\limits_{k\rightarrow\infty}\mathbb{E}_{\mu^{\phantom{A}}_{k}}[\theta(h)]=\infty$.

  However, $\theta(h)\le1$, we have $\mathbb{E}_{\mu_{k}^{\phantom{A}}}(\theta(h))\le1$ for
  all $k$, which is a contradiction. Therefore, we know that
  Eq.~\eqref{sm:eq:vacuum} holds.
\end{proof}

%===============================================================================
% To show the theorem, we need one more assumption:
%  $\mathbf{M}$
% has distinct columns, and $\mu_0$ satisfy 
% $\mathbb{E}_{\mu_0}\left(\dfrac{1-\theta(h)}{\theta(h)}\right)<\infty$. 
% This condition is equivalent to the assumptions (i) and (ii) in 
% Section~2, in the special case that $\mu_0=\delta_{\theta_0}$.
%===============================================================================

\begin{thm}[Theorem~\ref{thm:consistency_SCBI}]\label{sm:thm:consistency_SCBI}
  In SCBI, let $\mathbf{M}$ be a positive matrix. If the teacher is teaching
  one hypothesis $h$ (i.e.,
  $\widehat{\theta}=\delta_h\in\mathcal{P}(\mathcal{H})$), and the prior
  distribution 
  $\mu_0\in\mathcal{P}(\Delta^{m-1})$ satisfies
  $\mu_0=\delta_{\theta_0}$ with $\theta_0(h)>0$, then the estimator sequence $(S_k)$ is
  consistent, for each $h\in\mathcal{H}$, i.e., 
  the posterior random variables $(\Theta_k)_{k\in\mathbb{N}}$ 
  converge to the constant random variable $\widehat{\theta}$ in probability.
\end{thm}
Some notions used in the proof are visualized in Fig.~\ref{sm:fig:sketch}.
\begin{proof}
  Let $Z_0$ be a random variable with sample space
  $\Delta^{m-1}$ such that the $\mathrm{Law}(Z_0)=\mu_0$.
  This is the initial state in SCBI. 
  The posteriors in the following rounds are determined by the sequence of data taught by teacher,
  which makes the posteriors random variables as well. Let $Z_k$ be
  the random variable representing the posterior after $k$-rounds of SCBI,
  the law of $Z_k$ is given by
  $\mathrm{Law}(Z_k)=\mu_k=\left[\Psi(h)\right]^k(\mu_0)$ 
  according to the definition of $\Psi(h)$.
  
  The
  consistency mentioned in the theorem is equivalent to that the sequence
  $(Z_k)$ converges to $\widehat{Z}$ with
  $\mathrm{Law}(\widehat{Z})=\widehat{\mu}$ in probability where
  $\widehat{\mu}={\delta_{\widehat{\theta}}}$. 

  We prove the theorem by contradiction. Suppose $Z_k\rightarrow\widehat{Z}$ in
  probability is not valid, i.e., there exists $\epsilon>0$ such that
  \begin{equation}
    \lim_{k\rightarrow\infty}\mathrm{Pr}(d(Z_k,\widehat{Z})>\epsilon)
  \end{equation}
  does not exist or the limit is positive,
  where the metric $d$ on $\Delta^{m-1}$ is the Euclidean distance inherited
  from $\mathbb{R}^m$. In either case, there is a real number $C>0$ such that
  \begin{equation}
    \label{sm:eq:contra_condition}
    \mathrm{Pr}(d(Z_{k'},\widehat{Z})>\epsilon)>C
  \end{equation}
  for a subsequence $(Z_{k'})$ of $(Z_k)$.

  Let $R:=\mathbb{E}_{\mu_0}\left[\frac{1-\theta(h)}{\theta(h)}\right]=\frac{1-\theta_0(h)}{\theta_0(h)}$, let $a=\frac{1}{4R/C+1}$ and $b=1-\epsilon$. By
  Lemma~\ref{sm:lem:middle_vacuum}, there exists $N>0$ such that for all $k>N$,
  $$\mu_k(\Delta_{[a,b]})<C/2.$$

  Therefore, for all the terms in $(k')$ satisfying $k'>N$,
  $\mu_{k'}(\{\theta:\theta(h)<a\})>C/2$. 
  Furthermore,
  \begin{eqnarray}
    \label{sm:eq:estimate_ratio}
    &&\mathbb{E}_{\mu_{k'}}\left[\dfrac{1-\theta(h)}{\theta(h)}\right]\nonumber\\
    &\ge&\int_{\{\theta:\theta(h)<a\}}\left[\dfrac{1-\theta(h)}{\theta(h)}\right]\mathrm{d}\mu_{k'}(\theta)\nonumber\\
    &\ge&\left[\dfrac{1-a}{a}\right]\dfrac{C}{2}\nonumber\\
    &=&\left[\dfrac{4R}{C}\right]\dfrac{C}{2}\nonumber\\
    &=&2R>R.
  \end{eqnarray}

  However, by Lemma~\ref{sm:lem:expectation_property},
  \begin{eqnarray}
    \label{sm:eq:true_ratio}
    &&\mathbb{E}_{\mu_{k'}}\left[\dfrac{1-\theta(h)}{\theta(h)}\right]\nonumber\\
    &=&\mathbb{E}_{\mu_{k'}}\left[\dfrac{\sum_{h'\ne h}\theta(h')}
        {\theta(h)}\right]\nonumber\\
    &=&\sum_{h'\ne h}\mathbb{E}_{\mu_{k'}}\left[\dfrac{\theta(h')}
        {\theta(h)}\right]\nonumber\\
    &=&\sum_{h'\ne h}\mathbb{E}_{\mu_{0}}\left[\dfrac{\theta(h')}
        {\theta(h)}\right]\nonumber\\
    &=&\mathbb{E}_{\mu_{0}}\left[\dfrac{\sum_{h'\ne h}\theta(h')}
        {\theta(h)}\right]\nonumber\\
    &=&\mathbb{E}_{\mu_{0}}\left[\dfrac{1-\theta(h)}{\theta(h)}\right]\nonumber\\
    &=&R,
  \end{eqnarray}
  which is a contradiction. Therefore,
  \begin{equation}
    \lim_{k\rightarrow\infty}\mathrm{Pr}(d(Z_k,\widehat{Z})>\epsilon)=0.
  \end{equation}  
  And the sequence of SCBI estimators is consistent at $\widehat{\theta}$.

\end{proof}

\subsection{Proof of Theorem~\ref{thm:roc_scbi}}
\begin{thm}[Theorem~\ref{thm:roc_scbi}]\label{sm:thm:RoC_SCBI}
  With matrix $\mathbf{M}$, hypothesis $h\in\mathcal{H}$, and a 
  % deterministic
  prior $\mu_0=\delta_{\theta_0}\in\mathcal{P}(\Delta^{m-1})$ same as in Theorem.~\ref{sm:thm:consistency_SCBI},
  let $\theta_k$ denote a sample value of the posterior
  $\Theta_k$ after $k$ rounds of SCBI, then 
  \begin{equation}\small
      \label{sm:eq:roc_SCBI_converge}
      \lim_{k\rightarrow\infty}\mathbb{E}_{\mu_k}\left[
        \dfrac{1}{k}\log\left(\dfrac{\theta_{k}(h)}{1-\theta_{k}(h)}\right)
      \right]=\mathfrak{R}^{\text{s}}(\mathbf{M};h)
  \end{equation}
  where $\mathfrak{R}^{\text{s}}(\mathbf{M};h):=\min_{h\ne
    h'}\mathrm{KL}\left(\mathbf{M}^\sharp_{(\_,h)},\mathbf{M}^\sharp_{(\_,h')}\right)$ 
  % $\mathbb{E}_{\mu_k}\left[\dfrac{1}{k}\log\left(\dfrac{\theta_{k}(h)}{1-\theta_{k}(h)}\right)\right]$ converges
  % to  $\min\limits_{h\ne
   % h'}\mathrm{KL}\left(\mathbf{M}^\sharp_{(\_,h)},\mathbf{M}^\sharp_{(\_,h')}\right)$
  % where 
  with
  $\mathbf{M}^\sharp_{\phantom{|}}=\mathscr{N}_{\text{col}}(\mathrm{diag}(\mathbf{M}_{(\_,h)})^{-1}\mathbf{M})$.
  Thus we call % $\min\limits_{h\ne h'}\mathrm{KL}\left(\mathbf{M}^\sharp_{(\_,h)},
  %    \mathbf{M}^\sharp_{(\_,h')}\right)$
  $\mathfrak{R}^{\text{s}}(\mathbf{M};h)$
  the asymptotic rate of convergence (RoC) of SCBI.
\end{thm}

\begin{proof}
  We treat $\theta_k$ as random variables, then
  $$\mathbb{E}_{\mu_{k+1}^{\phantom{A}}}\!\!\left(\log\left[\dfrac{\theta_{k+1}(h')}{\theta_{k+1}(h)}\right]\right)=
  \mathbb{E}_{\mu_{k}^{\phantom{A}}}\!\!\left(\log\left[\dfrac{\theta_{k}(h')}{\theta_{k}(h)}\right]\right)+W_k^{h'},$$
  where
  $$W_k^{h'}=-\mathbb{E}_{\mu_{k}^{\phantom{A}}}\left[\mathrm{KL}\left(
    \mathscr{N}_{\text{vec}}(\mathbf{M}^{\left\langle{n}\theta_k\right\rangle}_{\phantom{~~}(\_,h)}),
    \mathscr{N}_{\text{vec}}(\mathbf{M}^{\left\langle{n}\theta_k\right\rangle}_{\phantom{~~}(\_,h')})
    \right)\right].$$

  We can get it from the following calculation ($\Delta$ represents the simplex $\Delta^{m-1}$):
  \begin{eqnarray}
    \label{sm:eq:exp_of_log_ratio}
    &&\mathbb{E}_{\mu_{k+1}^{\phantom{A}}}\!\!\left(\log\left[\dfrac{\theta_{k+1}(h')}{\theta_{k+1}(h)}\right]\right)\nonumber\\
    &=&\int\limits_{\Delta}\log\left[\dfrac{\theta(h')}{\theta(h)}\right]\mathrm{d}\mu_{k+1}(\theta)\nonumber\\
    &=&\!\!\!\int\limits_{\Delta}\sum_{d=1}^{n}\tau_d(T_d^{-1}(\theta))
        \log\left[\dfrac{\theta(h')}{\theta(h)}\right]
        \mathrm{d}(T_{d\ast}(\mu_{k+1}))(\theta)\nonumber\\
    &=&\!\!\!\sum_{d=1}^{n}\int\limits_{\Delta}\tau_d(\theta)
        \log\left[\dfrac{T_d(\theta)(h')}{T_d(\theta)(h)}\right]
        \mathrm{d}(\mu_{k})(\theta)\nonumber\\
    &=&\!\!\!\int\limits_{\Delta}\sum_{d=1}^{n}\dfrac{T_d(\theta)(h)}{n\theta(h)}
        \log\!\left[\dfrac{T_d(\theta)(h')}{T_d(\theta)(h)}\right]
        \mathrm{d}(\mu_{k})(\theta)\nonumber\\
    &=&\!\!\!\int\limits_{\Delta}\sum_{d=1}^{n}\dfrac{T_d(\theta)(h)}{n\theta(h)}\left\{
        \log\!\left[\dfrac{T_d(\theta)(h')}{n\theta(h')}\dfrac{n\theta(h)}{T_d(\theta)(h)}\right]\right.\nonumber\\
    &&\left.+\log\left[\dfrac{n\theta(h')}{n\theta(h)}\right]\right\}
       \mathrm{d}(\mu_{k})(\theta)\nonumber\\
    &=&\int\limits_\Delta-\mathrm{KL}\left(
        \mathscr{N}(\mathbf{M}^{\left\langle{n}\theta\right\rangle}_{\phantom{~~}(\_,h)}),
        \mathscr{N}(\mathbf{M}^{\left\langle{n}\theta\right\rangle}_{\phantom{~~}(\_,h')})
        \right)\mathrm{d}\mu_k\nonumber\\
    &&+\int\limits_\Delta\log\left[\dfrac{\theta(h')}{\theta(h)}\right]
       \mathrm{d}\mu_k\nonumber\\
    &=&W_k^{h'}+\mathbb{E}_{\mu_k^{\phantom{A}}}\left(\log\left[\dfrac{\theta_k(h')}{\theta_k(h)}\right]\right).
  \end{eqnarray}

  Next, we show
  \begin{equation}
    \label{sm:eq:single_hypo_comparison}\small
    \lim_{k\rightarrow\infty}\mathbb{E}_{\mu_k}\dfrac{1}{k}\log\left(\dfrac{\theta_k(h')}{\theta_k(h)}\right)
    =-\mathrm{KL}\left(\mathbf{M}^\sharp_{(\_,h)},\mathbf{M}^\sharp_{(\_,h')}\right),
  \end{equation}
  and then by a similar argument in the proof of Theorem~\ref{sm:thm:RoC_BI},
  we can show the result in this theorem.

  To show Eq.~\eqref{sm:eq:single_hypo_comparison}, we can make use of
  Eq.~\eqref{sm:eq:exp_of_log_ratio}. By showing that $W_k^{h'}$ converges to 
  $-\mathrm{KL}\left(\mathbf{M}^\sharp_{(\_,h)},\mathbf{M}^\sharp_{(\_,h')}\right)$,
  we can conclude that
  $\mathbb{E}_{\mu_k}\dfrac{1}{k}\log\left(\dfrac{\theta_k(h')}{\theta_k(h)}\right)$,
  as the average of $(W_i^{h'})$ on the first $k$-terms, converges to
  $-\mathrm{KL}\left(\mathbf{M}^\sharp_{(\_,h)},\mathbf{M}^\sharp_{(\_,h')}\right)$
  as well.

  To prove this, we need the following result from direct calculation:

  \begin{smlem}\label{sm:lem:kl_formula}
    Given a $n\times2$ positive matrix $[\mathbf{a},\mathbf{b}]$ with columns
    as $n$-vectors $\mathbf{a}=(a_1,a_2,\dots,a_n)^\top$ and
    $\mathbf{b}=(b_1,b_2,\dots,b_n)^\top$ with
    $\sum_{i=1}^{n}a_i=\sum_{i=1}^{n}b_i=1$, consider the $2\times2$
    cross-ratios: $C_i:=CR(1,2;1,i)=\dfrac{a_{1}b_i}{a_ib_1}$, then
    $\mathrm{KL}(\mathbf{a},\mathbf{b})=
    \log\left(\sum_{i=1}^{n}a_iC_i\right)-\sum_{i=1}^{n}a_i\log C_i$.
    With fixed $C_i\in(0,\infty)$ for $i=1,2,\dots,n$,
    $\mathrm{KL}(\mathbf{a},\mathbf{b})$ 
    is continuous and bounded about $\mathbf{a}\in\Delta^{n-1}$.
  \end{smlem}

  \begin{proof}[Proof of Lemma~\ref{sm:lem:kl_formula}]
    The formula of $\mathrm{KL}(\mathbf{a},\mathbf{b})$ is from direct
    calculation.

    The KL-divergence is continuous and bounded since by the formula, every part
    is continuous and bounded given the restrictions on $\mathbf{a}$ and $C_i$. 
  \end{proof}

  Now we continue to prove Theorem~\ref{sm:thm:RoC_SCBI}:

  By continuity of the KL-divergence given fixed cross-ratios,
  for any $\epsilon>0$, we find a number $\delta>0$ such that for
  any $\theta\in\Delta^{m-1}$ with $\theta(h)>1-\delta$,
  \begin{equation}\scriptsize
    \label{sm:eq:condition_1}
    \left| %W_k^{h'}
    \mathrm{KL}\!\left(\!
    \mathscr{N}_{\text{vec}}\!(\mathbf{M}^{\left\langle{n}\theta\right\rangle}_{\phantom{~}(\_,h)}\!), 
    \mathscr{N}_{\text{vec}}\!(\mathbf{M}^{\left\langle{n}\theta\right\rangle}_{\phantom{~}(\_,h')})\!
    \right)
    \!-\!\mathrm{KL}\!\left(\mathbf{M}^\sharp_{(\_,h)},\mathbf{M}^\sharp_{(\_,h')}
    \!\right)\right|<\dfrac{\epsilon}{3}.
  \end{equation}
  Further, according to Theorem~\ref{sm:thm:consistency_SCBI},
  and the boundedness from Lemma~\ref{sm:lem:kl_formula},
  we can find a number $N>0$, such that for any $k>N$, we have
  $\mu_k(\{\theta:\theta(h)<1-\delta\})<C$ where $C$ satisfies
  \begin{equation}\small
    \label{sm:eq:condition_2}
    C\cdot\sup\limits_{\theta\in\Delta^{m\!-\!1}}\left\{
    \mathrm{KL}\!\left(\!
    \mathscr{N}_{\text{vec}}\!(\mathbf{M}^{\left\langle{n}\theta\right\rangle}_{\phantom{~}(\_,h)}\!), 
    \mathscr{N}_{\text{vec}}\!(\mathbf{M}^{\left\langle{n}\theta\right\rangle}_{\phantom{~}(\_,h')})\!
  \right)\right\}<\dfrac{\epsilon}{3}.
  \end{equation}

  The expectation $W_k^{h'}$ can be split into two parts,
  $W_k^{h'}=-W_>-W_<$ where
  \begin{equation}\scriptsize
    W_>=\!\!\!\int\limits_{\theta(h)>1-\delta}\!\!\mathrm{KL}\!\left(\!
    \mathscr{N}_{\text{vec}}\!(\mathbf{M}^{\left\langle{n}\theta\right\rangle}_{\phantom{~}(\_,h)}\!),
    \mathscr{N}_{\text{vec}}\!(\mathbf{M}^{\left\langle{n}\theta\right\rangle}_{\phantom{~}(\_,h')})\!
    \right)
    \mathrm{d}\mu_k(\theta),
  \end{equation}
  and
  \begin{equation}\scriptsize
    W_<=\!\!\!
    \int\limits_{\theta(h)\le1-\delta}\!\!\mathrm{KL}\!\left(\!
    \mathscr{N}_{\text{vec}}\!(\mathbf{M}^{\left\langle{n}\theta\right\rangle}_{\phantom{~}(\_,h)}\!),
    \mathscr{N}_{\text{vec}}\!(\mathbf{M}^{\left\langle{n}\theta\right\rangle}_{\phantom{~}(\_,h')})\!
    \right)
    \mathrm{d}\mu_k(\theta).
  \end{equation}
  Similarly, since $\mu_k$ is a probability measure,
  $\mathrm{KL}\!\left(\mathbf{M}^\sharp_{(\_,h)},\mathbf{M}^\sharp_{(\_,h')}
    \!\right)=K_>+K_<$ where
  \begin{equation}\small
    K_>\!=\!\!\!\int\limits_{\theta(h)>1-\delta}\!\!\!
    \mathrm{KL}\!\left(\mathbf{M}^\sharp_{(\_,h)},\mathbf{M}^\sharp_{(\_,h')} 
    \!\right)
    \mathrm{d}\mu_k(\theta),
  \end{equation}
  and
  \begin{equation}\small
    K_<\!=\!\!\!
    \int\limits_{\theta(h)\le1-\delta}\!\!\!\!\!\!
    \mathrm{KL}\!\left(\mathbf{M}^\sharp_{(\_,h)},\mathbf{M}^\sharp_{(\_,h')} 
    \!\right)
    \mathrm{d}\mu_k(\theta).
  \end{equation}
  Then we have
  \begin{equation}\scriptsize
  \left| W_k^{h'}+\mathrm{KL}\!\left(\mathbf{M}^\sharp_{(\_,h)},
  \mathbf{M}^\sharp_{(\_,h')}\!\right) \right|\le|K_>-W_>|+|W_<|+|K_<|.
  \end{equation}
  The choice of $\delta$ can make a good estimate of the integral on
  $\theta(h)>1-\delta$.
  \begin{eqnarray}
    \label{sm:eq:distance}
    &&|K_>-W_>|\nonumber\\
    &\le&\dfrac{\epsilon}{3}(1-C)\nonumber\\
    &<&\dfrac{\epsilon}{3}.
  \end{eqnarray}

  For the other two terms, directly from condition Eq.~\eqref{sm:eq:condition_2},
  we have $|K_<|<\frac{\epsilon}{3}$ and $|W_<|<\frac{\epsilon}{3}$,
  and hence $|K_>-W_>|+|K_<|+|W_<|<\epsilon$.
  
  Therefore, $W_k^{h'}$ converges to
  $-\mathrm{KL}\left(\mathbf{M}^\sharp_{(\_,h)},\mathbf{M}^\sharp_{(\_,h')}\right)$.
  
\end{proof}

\subsection{An Example on a $2$ by $2$ Matrix}

Let $\mathcal{H} = \{h_1, h_2\}$, $\mathcal{D} = \{d_1, d_2\}$, and the shared joint distribution be 
$\tiny{\mathbf{M}^{JD} = \begin{blockarray}{ccc}
&h_1 & h_2   \\
\begin{block}{c (cc)}
 d_1 & 0.3 & 0.3 \\
 d_2 & 0.1 & 0.3 \\
\end{block}
\end{blockarray}}$.
Further assume that 
the learner has uniform prior on $\mathcal{H}$, i.e. $S_0=\theta_0 = (0.5,0.5)$ 
and the true hypothesis given to the teachers is $h_1$. 
In round~$1$, the BI teacher will sample a data from $\mathcal{D}$ according to the first column of 
$\tiny{\mathbf{M} = \begin{pmatrix}0.75 & 0.5 \\ 0.25 & 0.5 \end{pmatrix}}$, which is obtained by column normalizing $\mathbf{M}^{JD}$.
On the contrast, the SCBI teacher will form his likelihood matrix by first doing $(\rr_1, \cc_1)$-Sinkhorn scaling on
$\mathbf{M}$, then column normalization if needed, where $\rr_1 = (1,1)$ and 
$\cc_1 = (1, 1)$ based on the uniform priors. 
The resulting limit matrix (with precision of three decimals) is 
$\tiny \mathbf{M}^{*}_1= \begin{pmatrix}0.634 & 0.366 \\ 0.366 & 0.634 \end{pmatrix}$, which is already column normalized.
Hence the SCBI teacher will teach according to the first column of the $\mathbf{M}_1 = \mathbf{M}^{*}_1$.
Suppose that $d_1$ is sampled by both teachers. 
The posterior for the BI learner is $S^{\text{b}}_1(d_1) = (0.6,0.4)$ (normalizing the first row of $\mathbf{M}$).
The posterior for the SCBI learner is $S^{\text{s}}_1(d_1) = (0.643, 0.366)$ (the first row of $\mathbf{M}_1$).
% will generate a data from $\mathcal{D}$ 
% according to the distribution determined by the $h_1$-th column of $\mathbf{M}$, i.e. $40\%$ to sample $d_1$ and $60\%$ to sample $d_2$.
%Suppose that $d_2$ is sampled, both of the posteriors for the BI and SCBI learners are equal to $S_1(d_2) = (0.6,0.4)$.

Similarly, in round~$2$, the SCBI teacher would update his likelihood matrix by first doing $(\rr_2, \cc_2)$-Sinkhorn scaling on $\mathbf{M}_1$, 
where $\rr_2 = (1,1)$ and 
$\cc_2 = (0.643, 0.366) \times 2 = (1.268, 0.732)$. 
The resulting limit matrix is 
$\tiny \mathbf{M}^{*}_2= \begin{pmatrix}0.758 & 0.242 \\ 0.51 & 0.49 \end{pmatrix}$. 
Then through column normalizing $\mathbf{M}^{*}_2$, a updated likelihood matrix
$\tiny \mathbf{M}_2 = \begin{pmatrix} 0.60 & 0.33 \\ 0.4 & 0.67 \end{pmatrix}$ is obtained.
The SCBI teacher will teach according to the first column of the $\mathbf{M}_2$.
Whereas the BI teacher will again sample another data according to the the first column of $\mathbf{M}$.
Suppose that $d_1$ is sampled for both teachers. The posteriors for BI and SCBI learners are $S^{\text{b}}_2(d_1,d_1) = (0.692, 0.308) $ and $S^{\text{s}}_1(d_1, d_1)=(0.758,0.242)$ respectively. 

Although same teaching points are assumed, the SCBI learner's posterior on the true hypothesis $h_1$ is higher than the BI learner in both rounds.
Moreover, notice that the KL divergence between $h_1$ and $h_2$ is increasing as the likelihood matrix is updating through the SCBI. 
This will eventually lead much faster convergence for the SCBI learner.

\section{Calculations about Sample Efficiency}
% \paragraph{Computing Expectation of Average Difference RoC between SCBI and BI}
Here we compute the expectation $\mathfrak{E}$ mentioned in Sec.~4.1 of the
paper, for matrices of size $n\times 2$.

We first calculate the average of RoC for a particular matrix $\mathbf{M}$.
For simplicity, let $\mathbf{x}=\mathbf{M}_{(\_,1)}$ and
$\mathbf{y}=\mathbf{M}_{(\_,2)}$.
\begin{eqnarray}
  \label{sm:eq:avg_roc_bi}
  &&\dfrac{1}{2}\sum_{h=1}^{2}\mathfrak{R}^{\text{b}}(\mathbf{M};h) \nonumber\\
  &=&\dfrac{1}{2}\left(\mathrm{KL}(\mathbf{M}_{(\_,1)},\mathbf{M}_{(\_,2)}) +
      \mathrm{KL}(\mathbf{M}_{(\_,2)},\mathbf{M}_{(\_,1)}) \right) \nonumber\\
  &=&\dfrac{1}{2}\left(\mathrm{KL}(\mathbf{x},\mathbf{y}) +
      \mathrm{KL}(\mathbf{y},\mathbf{x}) \right) \nonumber\\
  % &=&\dfrac{1}{2}\left(\sum_{i=1}^{n}x_i(\ln x_i-\ln y_i) +
  %    \sum_{i=1}^{n}y_i(\ln y_i - \ln x_i)\right) \nonumber\\
  &=&\dfrac{1}{2}\left(\sum_{i=1}^{n}(\mathbf{x}_i-\mathbf{y}_i)(\ln \mathbf{x}_i-\ln \mathbf{y}_i)\right).
\end{eqnarray}

To calculate that for SCBI, denote $\mathbf{x/y}=\mathscr{N}_{vec}(\mathbf{v})$ 
the normalization of vector $\mathbf{v}$ with$\mathbf{v}_i=\mathbf{x}_i/\mathbf{y}_i$.

\begin{eqnarray}
  \label{sm:eq:avg_roc_scbi}
  &&\dfrac{1}{2}\sum_{h=1}^{2}\mathfrak{R}^{\text{s}}(\mathbf{M};h) \nonumber\\
  &=&\!\!\dfrac{1}{2}\left[\mathrm{KL}\left(\dfrac{\mathbf{e}}{n},\mathbf{x/y}\right)
      +\mathrm{KL}\left(\dfrac{\mathbf{e}}{n},\mathbf{y/x}\right)\right] \nonumber\\
  &=&\!\!\dfrac{1}{2}\left[\dfrac{1}{n}\sum_{i=1}^{n}\left(-2\ln n -
      \ln\dfrac{\mathbf{x}_i}{\mathbf{y}_i} + \ln \left(\sum_{j=1}^{n}\dfrac{\mathbf{x}_j}{\mathbf{y}_j}\right) \right.\right.
      \nonumber\\
    && \left.\left.-
      \ln\dfrac{\mathbf{y}_i}{\mathbf{x}_i} + \ln \left(\sum_{j=1}^{n}\dfrac{\mathbf{y}_j}{\mathbf{x}_j}\right)
      \right)\right] \nonumber\\
  &=&\!\!\dfrac{1}{2}\!\!\left[
      % \dfrac{1}{n}\sum_{i=1}^{n}\left(
      % -\ln\dfrac{x_i}{y_i} - \ln\dfrac{y_i}{x_i} \right)
      \ln\! \left(\sum_{j=1}^{n}\dfrac{\mathbf{x}_j}{\mathbf{y}_j}\right) + \ln\!
      \left(\sum_{j=1}^{n}\dfrac{\mathbf{y}_j}{\mathbf{x}_j}\right)\right]\!-\ln n.
\end{eqnarray}

The simulation of $\mathfrak{P}$ is based on the above calculations.
For $\mathfrak{E}$, the above expressions can be further simplified.

Given $\mathbf{M}=(\mathbf{x},\mathbf{y})$ uniformly distributed in
$(\Delta^{n-1})^2$, with measure $\nu\otimes\nu$ where $\nu$ is the measure of
uniform probability distribution on $\Delta^{n-1}$, we can calculate the
expected value, 
\begin{eqnarray}
  \label{sm:eq:exp_avg_roc_bi}
  && \mathbb{E}\left[\dfrac{1}{2}\sum_{h=1}^{2}\mathfrak{R}^{\text{b}}(\mathbf{M};h)\right]
     \nonumber\\
  &=& \dfrac{1}{2}\!\!\int\limits_{(\Delta^{n-1})^2}\!\!\sum_{i=1}^{n}(\mathbf{x}_i-\mathbf{y}_i)
      (\ln\mathbf{x}_i-\ln\mathbf{y}_i)
      \mathrm{d}\nu(\mathbf{x})\mathrm{d}\nu(\mathbf{y})\nonumber\\ 
      % \mathrm{d}\mathbf{x}\mathrm{d}\mathbf{y} \nonumber\\ 
  &=& \!\!n\!\!\int\limits_{\Delta^{n-1}}\!\!\mathbf{x}_1\ln \mathbf{x}_1 \mathrm{d}\nu(\mathbf{x}) -
      n\!\!\!\int\limits_{(\Delta^{n-1})^2} \!\!\!\!\mathbf{x}_1\ln
      \mathbf{y}_1\mathrm{d}\nu(\mathbf{x})\mathrm{d}\nu(\mathbf{y})  \nonumber\\
  &=&  n\int_0^1 x (n-1)(1-x)^{n-2}\ln x\mathrm{d}x+
      \nonumber\\ 
  && \!\! n\int_0^1\!\!\!\int_0^1 x
      (n-1)^2(1-x)^{n-2}(1-y)^{n-2}\ln y\mathrm{d}x\mathrm{d}y \nonumber\\
%  &=& \dots (\text{Maybe directly go to result?}) \nonumber\\
  &=& \dfrac{n-1}{n}.
\end{eqnarray}

Here we use the fact that 
$$\int\limits_{\{\theta\in\Delta^{n-1}:\theta(h)=a\}}\mathrm{d}\mathbf{x}=(n-1)(1-a)^{n-2}.$$
% This is kind of obvious~
% \jw{Do I need to explain why $d\nu$ looks like this and how to get final result?}

Furthermore, since integral on $(\Delta^{n-1})^2$ with measure $\nu\otimes\nu$ is
symmetric on $\mathbf{x}$ and $\mathbf{y}$, we have

\begin{equation}\small
  \label{sm:eq:value_of_E}
  \mathfrak{E}=\!\!\!\!
  \int\limits_{(\Delta^{n-1})^{2}}
  \!\!\!\!\!\ln\left(\sum_{i=1}^{n}\frac{\mathbf{x}_i}{\mathbf{y}_i}\right)
  \mathrm{d}\nu(\mathbf{x})\mathrm{d}\nu(\mathbf{y})-\ln n-\frac{n-1}{n}.
\end{equation}

In general, we calculate the integral in Eq.~\eqref{sm:eq:value_of_E} by Monte
Carlo method since other numerical integral methods we tried becomes slow
dramatically as 
$n$ grows. In particular, when $n=2$, an expression related to the dilogarithm
$Li_2$ can be obtained (can be easily checked in Wolfram software).

\begin{figure}[!ht]
  \centering
  \includegraphics[scale=0.07]{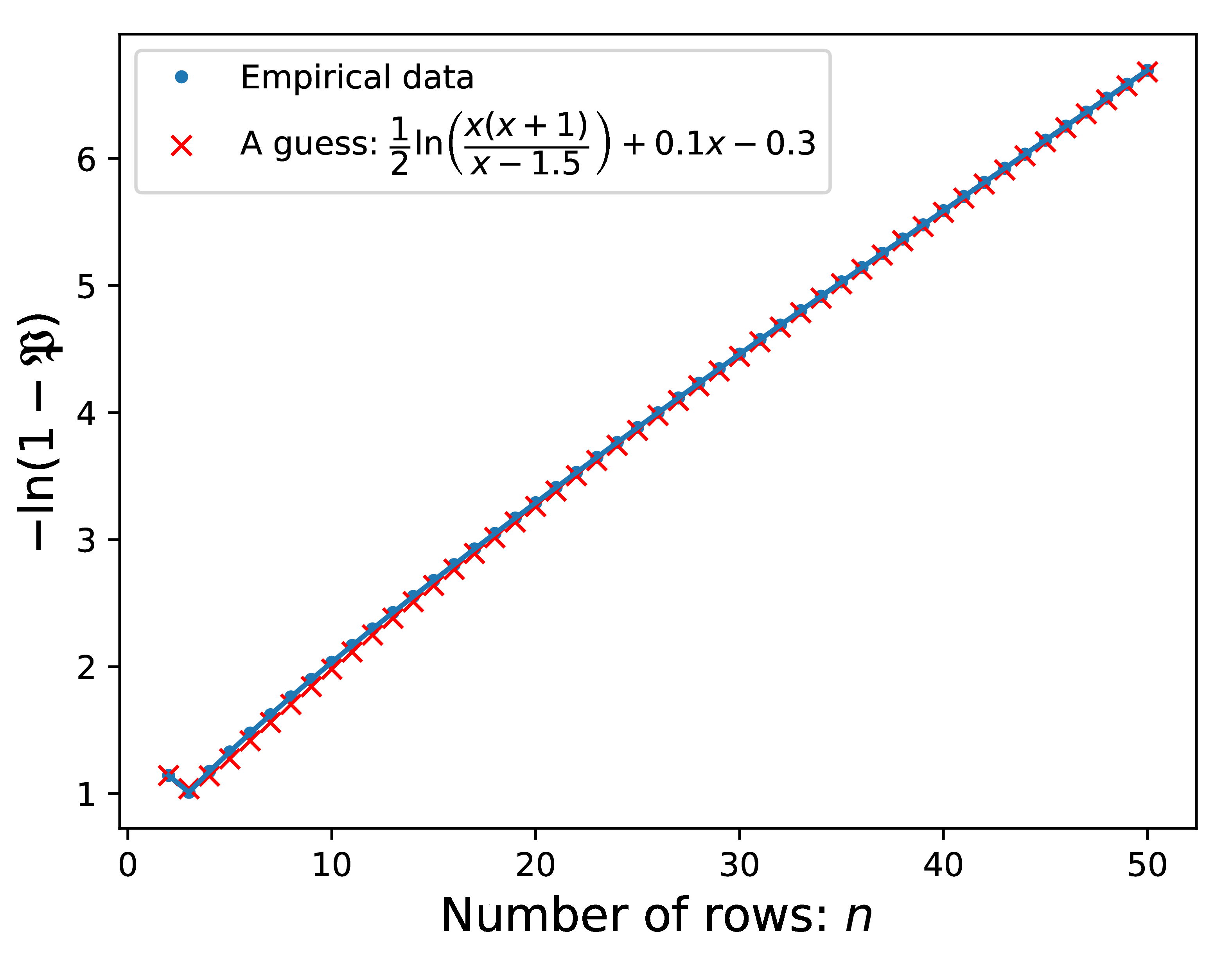}
  \caption{A guess of the $\mathfrak{P}$ values for each $n$ together with empirical data}
  \label{sm:fig:empirical_2col}
\end{figure}

And we have an empirical formula to describe the relation between 
$\mathfrak{P}$ and $n$, shown in Fig.~\ref{sm:fig:empirical_2col}:
\begin{equation}
    \label{sm:eq:empirical_2col}
    \mathfrak{P}_{(n,2)}=1-\sqrt{\dfrac{x-1.5}{x(x+1)}}e^{-0.1x+0.3}
\end{equation}

\section{Theory and Empirical Data on Stability}
\subsection{Proof of Proposition~\ref{prop:asympt_posterior}}
\begin{prop}[Proposition~\ref{prop:asympt_posterior}]
  \label{sm:prop:asympt_posterior}
  Given a sequence of identical independent $\mathcal{D}$-valued 
  random variables $(\mathrm{D}_i)_{i\ge1}$ following uniform distribution. Let 
  $\mu_0\in\mathcal{P}(\Delta^{m-1})$ be a prior distribution on $\Delta^{m-1}$, and
  $\mu_{k+1}=\Psi_{\mathrm{D}_k}^{\mathbf{L}}(\mu_k)$, then $\mu_k$ converges, 
  in probability, to
  $\sum_{i\in\mathcal{H}}a_i\delta_{i}$, where
  $a_i=\mathbb{E}_{\mu_0}\left[\theta(i)\right]$.
\end{prop}
To show the above proposition, we need the following lemma:
\begin{smlem}
  \label{sm:lem:ex_invariant_uniform_sample}
  Given the conditions in Proposition~\ref{sm:prop:asympt_posterior},
  then for any $k\in\mathbb{N}$ and $h\in\mathcal{H}$,
  \begin{equation}
    \label{sm:eq:ex_inv_uniform_sample}
    \mathbb{E}_{\mu_k}(\theta(h))=\mathbb{E}_{\mu_0}(\theta(h)).
  \end{equation}
\end{smlem}

\begin{proof}
  It suffices to show
  $\mathbb{E}_{\mu_{k+1}}(\theta(h))=\mathbb{E}_{\mu_k}(\theta(h))$ for any $k$.
  \begin{eqnarray}
    \mathbb{E}_{\mu_{k+1}}(\theta(h))
    &=&
        \int_{\Delta^{m-1}}\theta(h)\mathrm{d}(\mu_{k+1})(\theta)
        \nonumber\\
    &=&
        \sum_{d\in\mathcal{D}}\mathrm{D_k}(d)\int_{\Delta^{m-1}}T_d(\theta)(h)\mathrm{d}\mu_k(\theta)
        \nonumber\\
    &=&
        \int_{\Delta^{m-1}}\dfrac{1}{n}\sum_{d\in\mathcal{D}}T_d(\theta)(h)\mathrm{d}(\mu_{k})(\theta)
        \nonumber\\
    &=&\int_{\Delta^{m-1}}\theta(h)\mathrm{d}(\mu_{k})(\theta)=\mathbb{E}_{\mu_{k}}(\theta(h)).
        \nonumber
  \end{eqnarray}
\end{proof}

\begin{proof}[Proof of Proposition~\ref{sm:prop:asympt_posterior}]
  We first show the following result:

  For any $\epsilon>0$, let
  $$\Delta_{\epsilon}:=\left\{\theta\in\Delta^{m-1}:\theta(i)\le1-\epsilon,\forall
    i=1,2,\dots,m\right\},$$
  then $\lim\limits_{k\rightarrow\infty}\mu_k(\Delta_{\epsilon})=0$.

  We prove this by contradiction. Suppose the limit does not exist or is not
  $0$, then there is a positive number $C$ and a subsequence
  $(\mu_{k_i})_{i\in\mathbb{N}}$ such that $\mu_{k_i}(\Delta_{\epsilon})>C$ for
  all $i$.
  
  We define a linear functional
  $\mathscr{L}(\mu):=\mathbb{E}_{\mu}f(\theta)$, where
  $f(\theta)=||\theta-u||_2^2$  with $u=\dfrac{\mathbf{e}}{m}$ the center of
  $\Delta^{m-1}$.

  By definition, for a random variable following uniform distribution on
  $\mathcal{D}$,
  $\mathscr{L}\left(\Psi_{\mathrm{D}}^{\mathbf{L}}(\mu)\right)=
  \mathbb{E}_{\mu}\left(\mathbb{E}_{d\sim\mathrm{D}}f(T_d(\theta))\right)$. 

  Consider that $f$ is a strictly convex function, by Jensen's inequality,
  $\mathbb{E}_{d\sim\mathrm{D}}f(T_d(\theta))\ge
  f(\mathbb{E}_{d\sim\mathrm{D}}T_d(\theta))=f(\theta)$, with equality if and
  only if $T_d(\theta)=\theta$ for all $d\in\mathcal{D}$, equivalently by the
  assumptions on matrix $\mathbf{L}$, $\theta=\delta_h$ for some $h\in\mathcal{H}$.
  (This is because we assume $\mathbf{L}$ have distinct
  columns, thus not all $2$-by-$2$ cross-ratios are $1$, for any pair of columns.
  however, after $T_d$ all $2$-by-$2$ cross-ratios are $1$, indicating the existence of degeneration on every
  pair of columns. This can only happen when $\theta=\delta_h$ for some $h\in\mathcal{H}$.)
  
  % \jw{Is this equivalence obvious enough?}

  Thus for any $\theta\in\Delta_\epsilon$,
  $\mathbb{E}_{d\sim\mathrm{D}}f(T_d(\theta))>f(\theta)$. As $\Delta_\epsilon$
  is compact, there is a lower bound $B>0$, such that
  $\mathbb{E}_{d\sim\mathrm{D}}f(T_d(\theta))-f(\theta)>B$ for all
  $\theta\in\Delta_\epsilon$. Thus 
  \begin{eqnarray}
    \!\!&&\mathscr{L}(\mu_{k_{i}+1})\nonumber\\
    \!\!&=&
        \!\!\!\!\!\int\limits_{\Delta^{m-1}\backslash\Delta_\epsilon}
        \!\!\!\mathbb{E}_{d\sim\mathcal{D}}f(T_d(\theta))\mathrm{d}\mu_{k_i}+
        \int\limits_{\Delta_\epsilon}\mathbb{E}_{d\sim\mathcal{D}}
        f(T_d(\theta))\mathrm{d}\mu_{k_i}
        \nonumber\\
    \!\!&>&
        \!\!\!\!\int\limits_{\Delta^{m-1}\backslash\Delta_\epsilon}
        \!\!f(\theta)\mathrm{d}\mu_{k_i}+
        \int\limits_{\Delta_\epsilon}f(\theta)\mathrm{d}\mu_{k_i}+BC
            \nonumber\\
    \!\!&=&
            \mathscr{L}(\mu_{k_i})+BC\nonumber
  \end{eqnarray}
  for all $i\in\mathbb{N}$ and simply
  $\mathscr{L}(\mu_{k+1})\ge\mathscr{L}(\mu_k)$ for general $k$.

  Therefore, $\mathscr{L}(\mu_k)$ is unbounded as $k\rightarrow\infty$ since
  there is at least a $BC>0$ increment at each $k_i$.

  However, by definition, $f$ is bounded by $m$ since $\sqrt{m}$ is
  the diameter of $\Delta^{m-1}$ under $2$-norm, thus $\mathscr{L}(\mu)\le m$.

  Such a contradiction shows that the opposite of our assumption,
  $\lim\limits_{k\rightarrow\infty}\mu_k(\Delta_\epsilon)=0$, is valid.

  Consider that $\epsilon$ is arbitrary, and in
  Lemma~\ref{sm:lem:ex_invariant_uniform_sample} we show that 
  $\mathbb{E}_{\mu_k}\theta(h)$ is invariant, thus $\mu_k$ approaches
  $\sum_{i\in\mathcal{H}}a_i\delta_i$ with $a_i=\mathbb{E}_{\mu_0}\theta(i)$ in
  probability. 
\end{proof}

\subsection{Empirical Data for Stability: Perturbation on Prior}

We sample $5$ matrices of size $3\times3$, each of them are 
column-normalized, and their columns are sampled independently and uniformly on
$\Delta^2$, listed below:
\begin{eqnarray}
  \label{sm:eq:matrix_samples_3}
  \mathbf{M}_1&=&\left(
    \begin{array}{ccc}
      0.6559 & 0.5505 & 0.7310  \\
      0.1680 & 0.3359 & 0.0403  \\
      0.1760 & 0.1136 & 0.2287
    \end{array}
  \right)\nonumber\\
  \mathbf{M}_2&=&\left(
    \begin{array}{ccc}
      0.2461 & 0.6600 & 0.4310  \\
      0.6785 & 0.0655 & 0.2325  \\
      0.0754 & 0.2746 & 0.3365
    \end{array}
  \right)\nonumber\\
  \mathbf{M}_3&=&\left(
  \begin{array}{ccc}
    0.7286 & 0.1937 & 0.7620  \\
    0.0739 & 0.4786 & 0.1999  \\
    0.1974 & 0.3277 & 0.0382
  \end{array}
  \right)\nonumber\\
  \mathbf{M}_4&=&\left(
  \begin{array}{ccc}
    0.4745 & 0.2024 & 0.5946  \\
    0.2898 & 0.7499 & 0.1313  \\
    0.2357 & 0.0477 & 0.2741
  \end{array}
  \right)\nonumber\\
  \mathbf{M}_5&=&\left(
  \begin{array}{ccc}
    0.2207 & 0.5466 & 0.1605  \\
    0.3828 & 0.3807 & 0.5697  \\
    0.3965 & 0.0727 & 0.2698
  \end{array}
  \right)
\end{eqnarray}

And the $5$ sampled priors are:
\begin{eqnarray}
  \label{sm:eq:prior_samples_3}
  \theta_1&=&\left(0.3333, 0.3333, 0.3333\right)^{\top} \nonumber\\
  \theta_2&=&\left(0.1937, 0.4291, 0.3771\right)^{\top} \nonumber\\
  \theta_3&=&\left(0.4544, 0.0814, 0.4641\right)^{\top} \nonumber\\
  \theta_4&=&\left(0.5955, 0.2995, 0.1051\right)^{\top} \nonumber\\
  \theta_5&=&\left(0.4771, 0.0593, 0.4636\right)^{\top}
\end{eqnarray}

These names (including the rank $4$ samples below) are overriding the 
previously defined identical symbols 
in this part and in the corresponding subsection in the main paper.
In the $4\times4$ cases, we sample $3$ matrices in the same way as in $3\times3$
case.

\begin{eqnarray}
  \label{sm:eq:matrix_samples_4}
  \mathbf{M}_1'\!\!&=&\!\!\left(
                  \begin{array}{cccc}
                    0.3916 & 0.2306 & 0.0460 & 0.0404  \\
                    0.1408 & 0.6350 & 0.2139 & 0.2310  \\
                    0.2375 & 0.0275 & 0.1667 & 0.2412  \\
                    0.2301 & 0.1068 & 0.5734 & 0.4874
                  \end{array}
  \right)\nonumber\\
  \mathbf{M}_2'\!\!&=&\!\!\left(
                  \begin{array}{cccc}
                    0.3744 & 0.6892 & 0.0112 & 0.3200  \\
                    0.3204 & 0.2320 & 0.4498 & 0.3530  \\
                    0.0291 & 0.0688 & 0.3865 & 0.0653  \\
                    0.2761 & 0.0100 & 0.1526 & 0.2618
                  \end{array}
  \right)\nonumber\\
  \mathbf{M}_3'\!\!&=&\!\!\left(
                  \begin{array}{cccc}
                    0.2885 & 0.0873 & 0.2319 & 0.1009  \\
                    0.0653 & 0.2239 & 0.0575 & 0.2584  \\
                    0.5934 & 0.3276 & 0.2283 & 0.3925  \\
                    0.0529 & 0.3612 & 0.4823 & 0.2482
                  \end{array}
  \right)
\end{eqnarray}

And $3$ corresponding priors are sampled:

\begin{eqnarray}
  \label{sm:eq:prior_samples_4}
  \theta_1'&=&\left(0.2500, 0.2500, 0.2500, 0.2500 \right)^{\top} \nonumber\\
  \theta_2'&=&\left(0.1789, 0.3664, 0.2915, 0.1632 \right)^{\top} \nonumber\\
  \theta_3'&=&\left(0.4460, 0.4676, 0.0821, 0.0043 \right)^{\top}
\end{eqnarray}

\begin{figure}[!ht]
  \centering
  \includegraphics[scale=0.06,trim=0 50 0 0,clip]{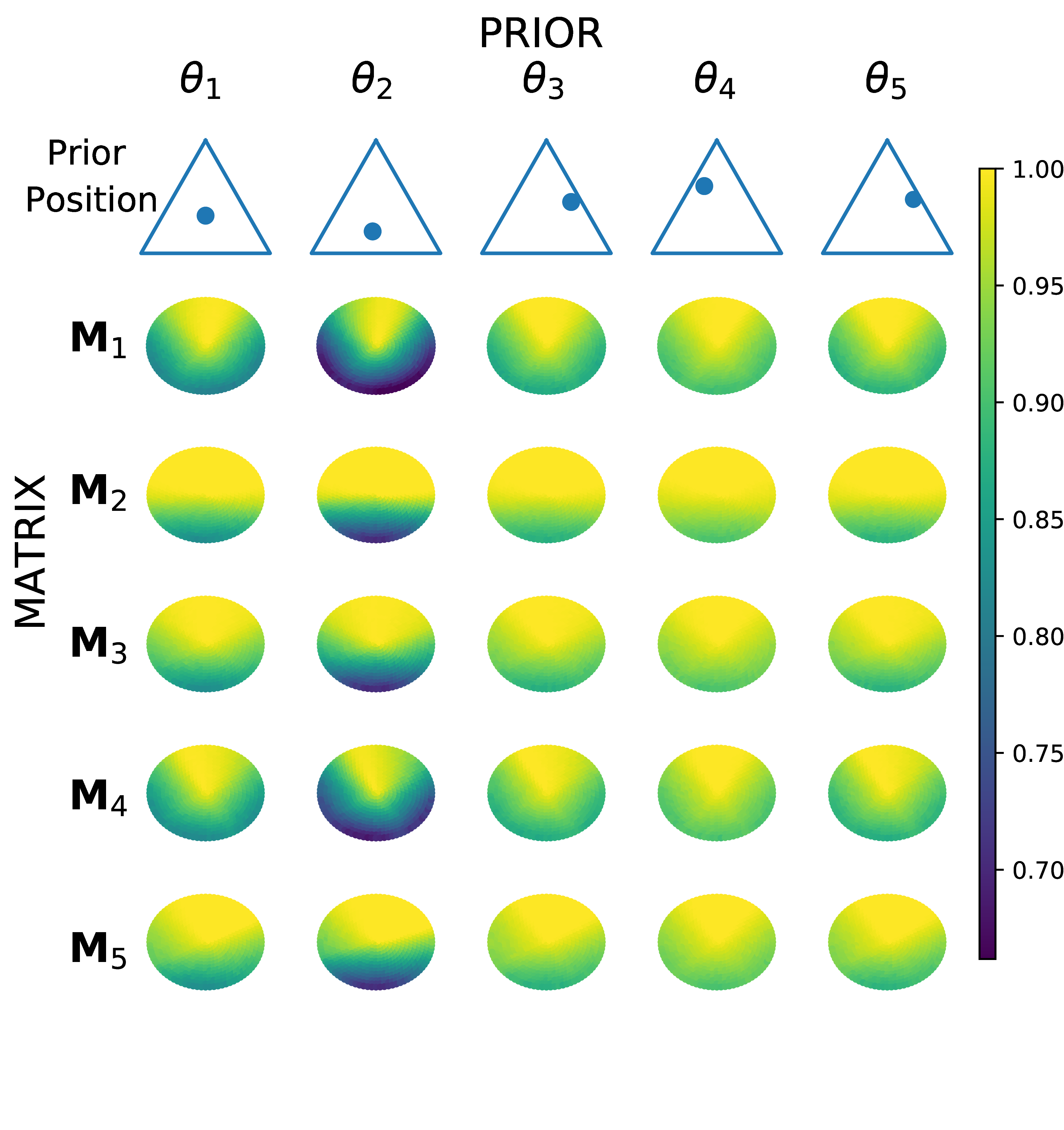}
  \caption{Successful rate of SCBI perturbed on prior. Each entry corresponds
    to a pair $\mathbf{M}$ and $\theta_0^T$. The first row shows for each
    prior $\theta_0^T$, the position it locates in $\Delta^2$ and the range of
    $\theta_0^L$ in the simulation. The $5$ rows below are the zoom-in version
    of the shaded area in each case, whose color at each point represents the
    successful rate when $\theta_0^L$ locates at that point.\label{sm:fig:mercury} }
\end{figure}

The value we use to test the effectiveness of perturbed SCBI is called the
successful rate, which is $\mathbb{E}\left[\theta_\infty^L(h)\right]=
\mathbb{E}_{\mu_\infty^L}\left[\theta(h)\right]$ where $h$  
is the true hypothesis that the teacher teaches (Definition~5.2). 
Successful rate is well defined, i.e. the limit exists, 
according to the convergences in probability
(Theorem~\ref{sm:thm:consistency_SCBI} and Proposition~\ref{sm:prop:asympt_posterior})
with an $\epsilon$ discussion based on them.
To find the successful rate,
we use Monte-Carlo method on $10^4$ teaching sequences, and use
Proposition~\ref{sm:prop:asympt_posterior} to accelerate the simulation.

We can estimate an upper bound of the standard deviation (precision) of 
the empirical successful rate calculated based on Proposition~\ref{sm:prop:asympt_posterior}. The successful rate
of a single teaching sequence
is between $0$ and $1$, thus with a standard deviation smaller than $1$.
So the standard deviation of the empirical successful rate is bounded
by $(N)^{-1/2}$ where $N$ is the number of sample sequences.
Actually the precision is much smaller since the successful rate for a 
single sequence is much more stable.

Our first simulation is shown in Fig.~\ref{sm:fig:mercury}, where we take
$\theta_0^L$ evenly on a series of concentric circles centered at $\theta_0^T$.
There are $14$ such circles with radius $0.005$ to $0.07$. On the $i$-th layer
(smallest circle is the first layer) we take $6i$ many points evenly separated,
the upper right figure in Fig.~\ref{sm:fig:radial_3} shows how the points are
taken in detail.

Thus we have $6$ groups of points each distributed along a ray. We plot the
successful rate versus the distance from the center along each ray in
Fig.~\ref{sm:fig:radial_3} for all the $25$ combinations of $\mathbf{M}$ and
$\theta_0^T$.

\begin{figure*}[!ht]
  \centering
  \includegraphics[scale=0.079]{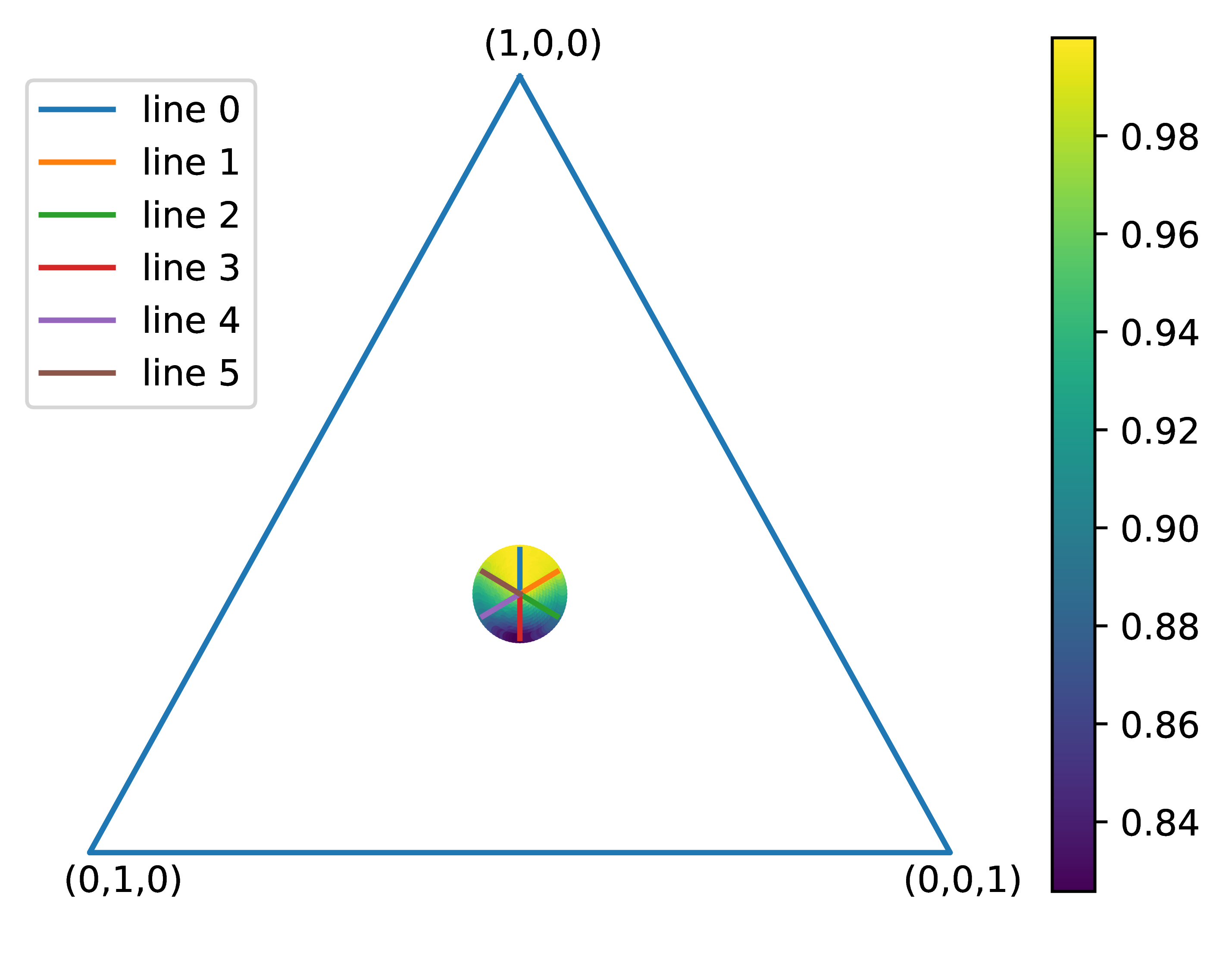}
  \includegraphics[scale=0.079]{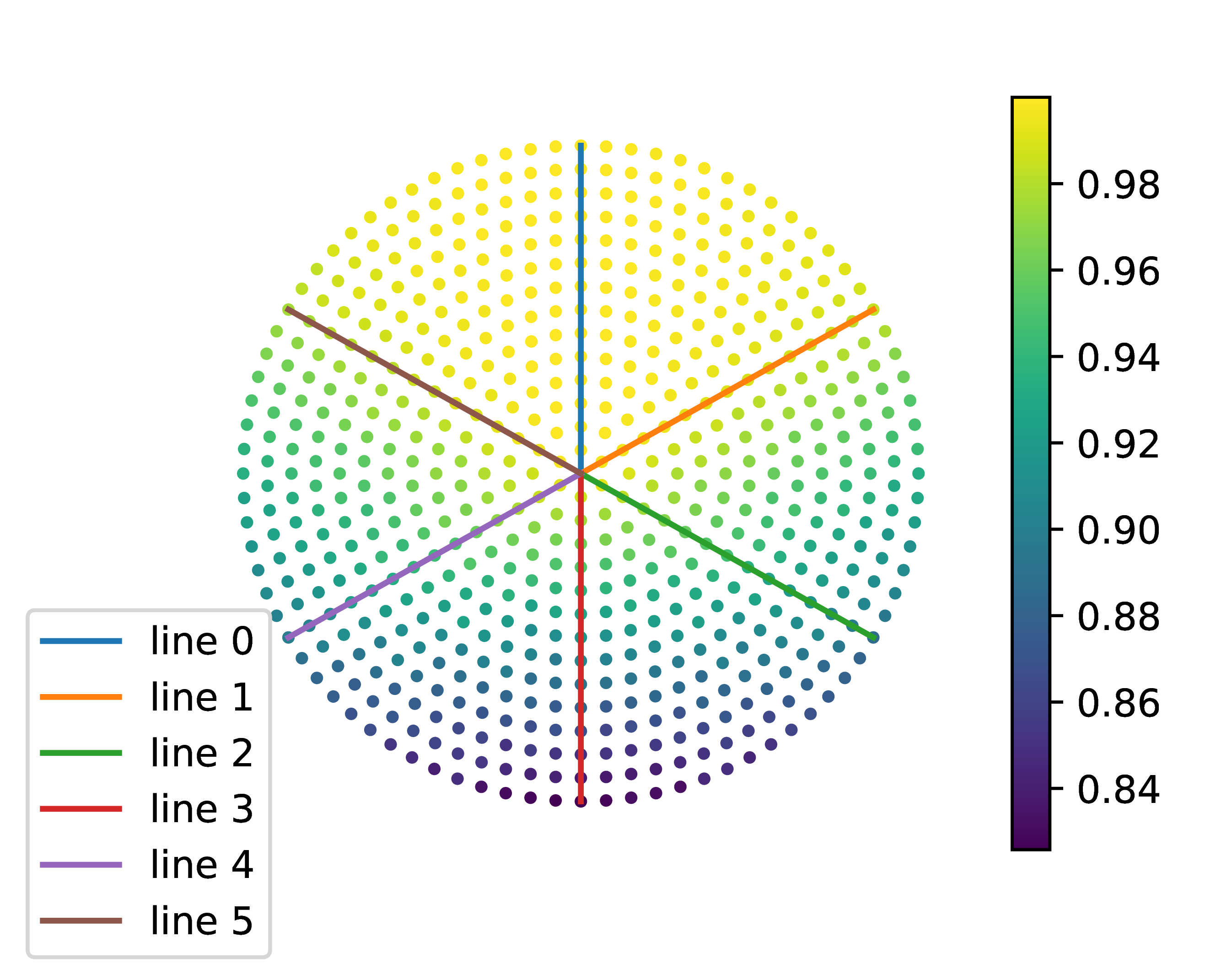}
  \includegraphics[scale=0.079,trim=0 50 0 0,clip]{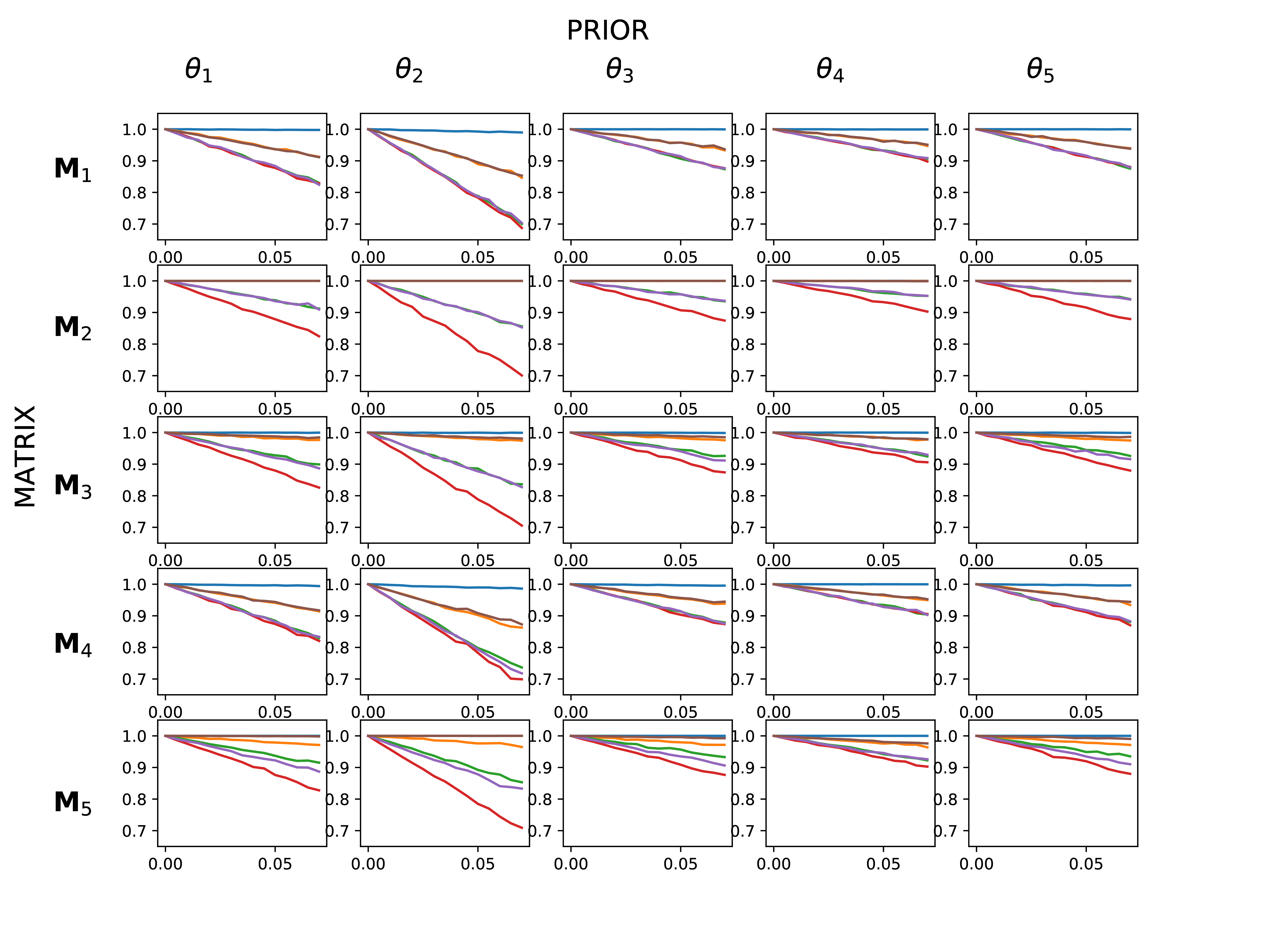}
  \caption{Upper-Left: the six rays at the center $\theta_1$. Upper-Right:
    zoom-in figure of the six rays in general. Lower: Successful rate versus
    distance to center along $6$ rays. Fig.~5 in the main paper contains 
    the Row 3 Column 1 picture of the lower one (with a different $y$-scale).\label{sm:fig:radial_3} } 
\end{figure*}

To have a similar directional data for matrices of size $4\times4$, we take a
sample of $15$ directions in $\mathbb{R}^3$ (showing in Fig.~\ref{sm:fig:radial_4} in spherical coordinates
centered at $(1/4,1/4,1/4,1/4)^\top$, with $(1,0,0,0)^\top$ as $\phi=0$ axis and $(0,1,0,0)^\top$ on the half-plane given by $\theta=0$) and simulate the perturbations of
$\theta_0^L$ in $\Delta^3$ along the $15$ directions. On each ray, we take $20$
evenly placed $\theta_0^L$ with distance to the center $\theta_0^T$ from $0.005$
to $0.100$. Then we plot the
successful rate versus the distance in Fig.~\ref{sm:fig:radial_4} for all $9$
cases as before.

\begin{figure*}[!ht]
  \centering
  \includegraphics[scale=0.13]{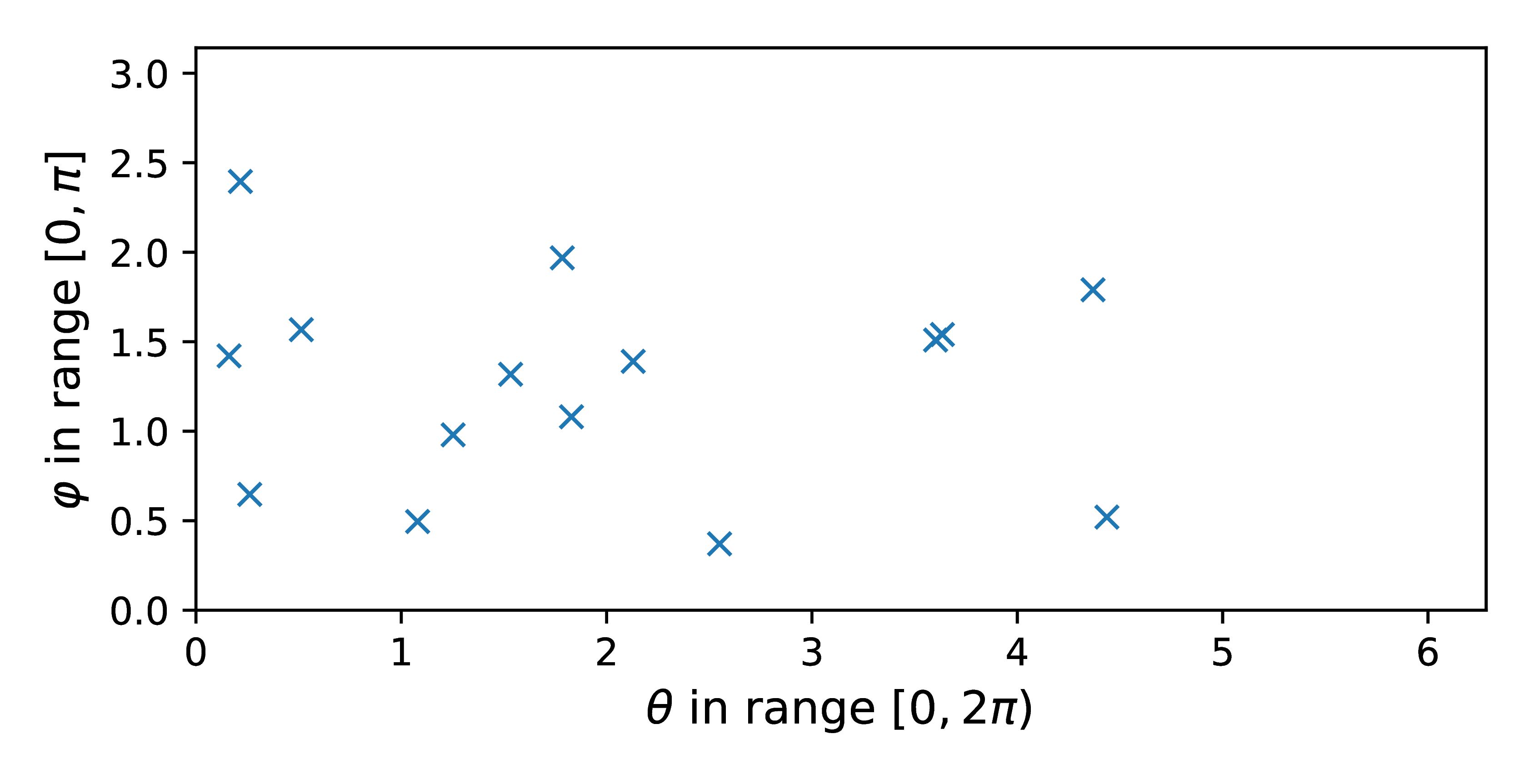}
  \includegraphics[scale=0.086,trim=0 35 0 0,clip]{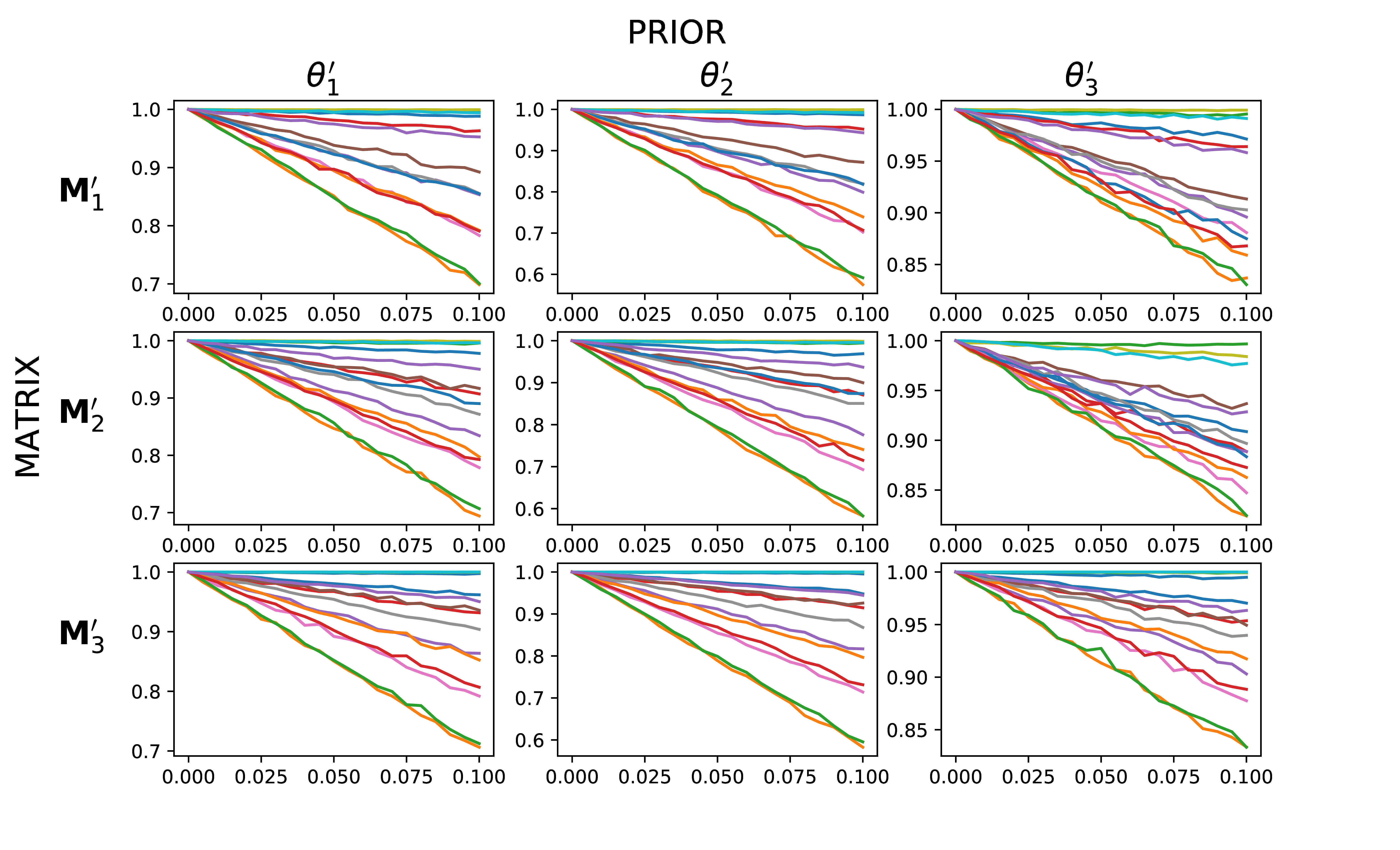}
  \caption{Upper: the sampled directions in spherical coordinates.
  Lower: Successful rate versus distance to center, along 15 rays, for all the
  $9$ cases of matrices of size $4\times4$. The plot at Row 1 Column 1 appears in Fig.~5 of the main file.}
  \label{sm:fig:radial_4}
\end{figure*}

\begin{remark}
  This part provides evidences of linear influence of the perturbation distance on
  the successful rate along a fixed direction.  
\end{remark}

Next we explore the global behavior of perturbations on prior.
Here we sample for each combination of $\mathbf{M}$ and $\theta_0^T$ a set of
$300$ points for $\theta_0^L$ evenly distributed in $\Delta^3$.

In Fig.~\ref{sm:fig:trapezoid}, we plot the successful rate versus the value of
$\theta_0^L(h)$, for all $25$ situations.
\begin{figure*}[!ht]
  \centering
  \includegraphics[scale=0.062,trim=20 50 70 10, clip]{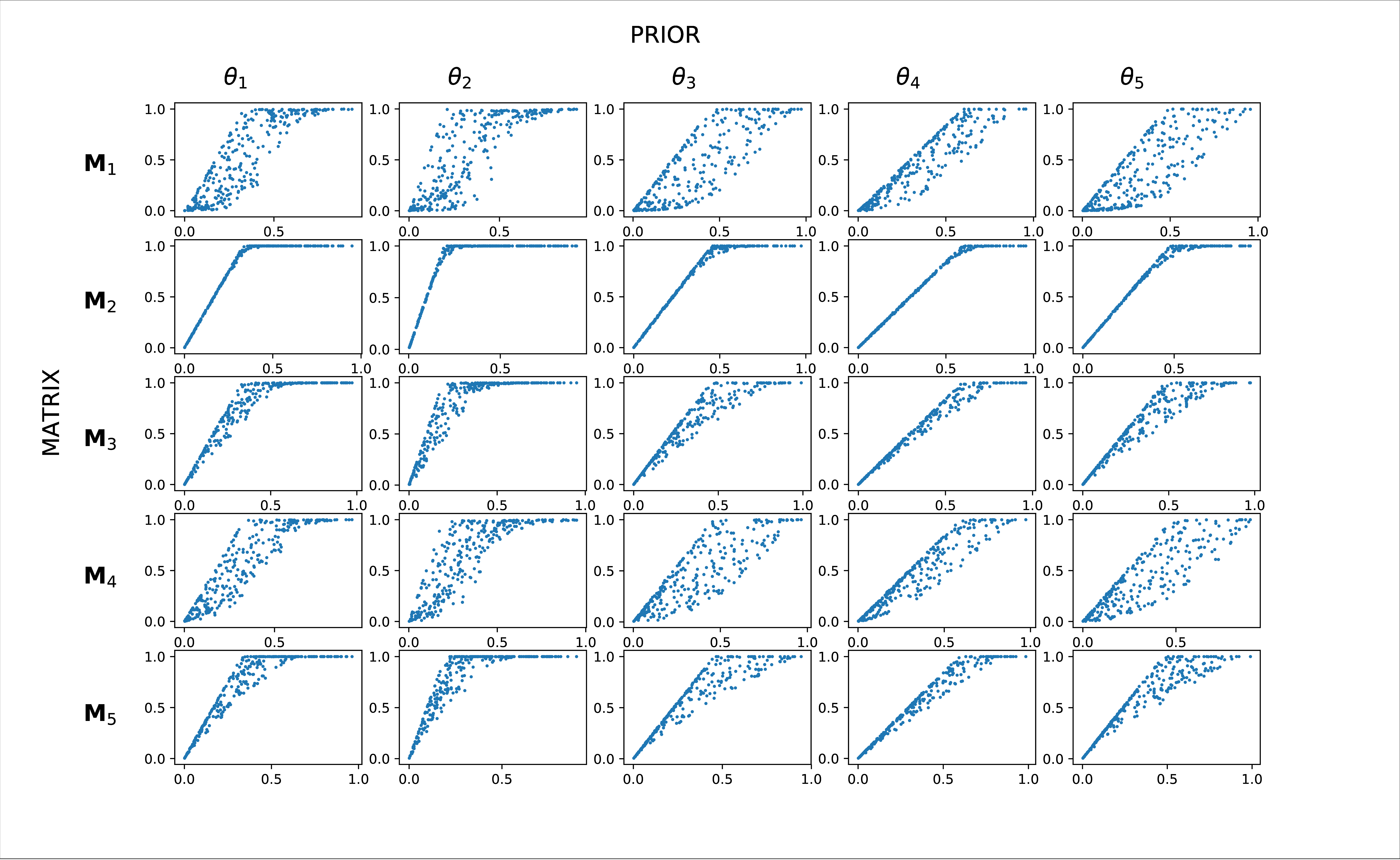}
  \caption{The 25 cases of $3\times3$ matrices, with successful rate versus
    $\theta_0^T(h)$ plotted.
    Plot at Row 3 Column 1 appears in Fig.~5 of the main file.
    \label{sm:fig:trapezoid} }
\end{figure*}

We plot in Fig.~\ref{sm:fig:reverse} the distance to center as $x$-coordinates,
for $9$ situations with matrices of size $4\times4$.

\begin{figure*}[!ht]
  \centering
  \includegraphics[scale=0.08]{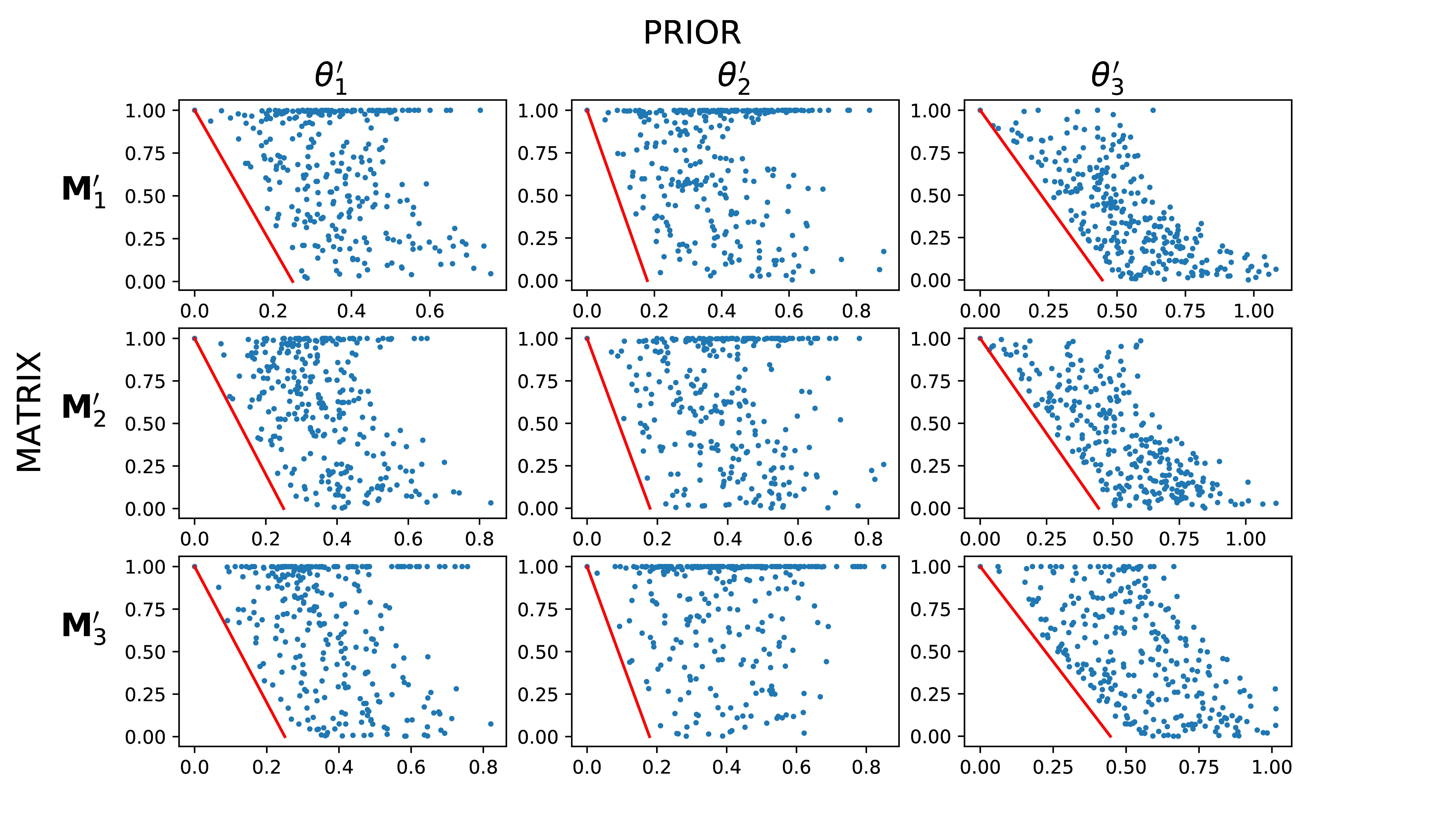}
  \caption{The 9 cases of $4\times4$ matrices, with successful rate versus the
    distance to the center plotted. Plot at Row 1 Column 1 appears in Fig.~5 of the main file. Red line is the lower bound given in
    Conjecture~5.3.\label{sm:fig:reverse} }
\end{figure*}
In this part, we observe that there is a lower bound of the successful rate
which depends linearly on the distance to center, with slope bounded by
$\frac{1}{\theta_0^T(h)}$ (Conjecture~5.3).

\subsection{Empirical Data for Stability: Perturbation on Matrix}

Fig.~\ref{sm:fig:mv} shows the behavior of perturbations on all sampled
$3\times3$ matrices in Section~5. Perturbations are taken only along the
relevant column / irrelevant column, since a perturbation on the 
target column is equivalent of the combination of a perturbation on
other two columns (they have the same set of Cross-ratios, which determines
the SCBI behavior). The cycle path in each plot is the equi-normalized-KL path,
with any point on the path having the same normalized-KL to the target 
column as that of the original matrix $\mathbf{T}$. 

These graphs should not be confused with the ones occur in the prior perturbed
part, as we are plotting each column of the matrix here (the simplex is actually $\mathcal{P}(\mathcal{D})$), while we were plotting
the priors in previous discussion (the simplex is $\mathcal{P}(\mathcal{H})$).
\begin{figure*}[!hb]
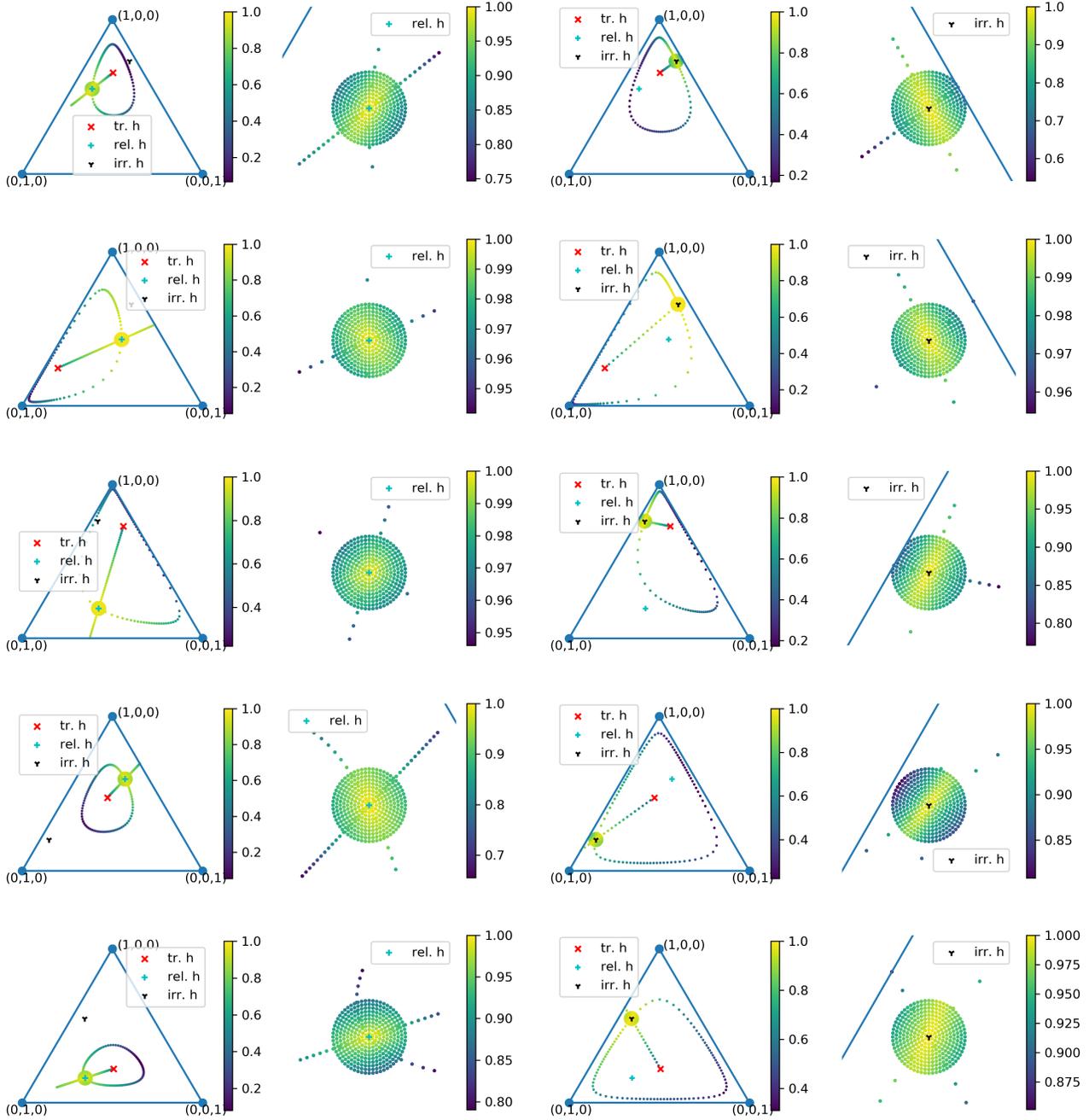

  \centering
  \foreach \x in {0,5,10,15,20}{
    \includegraphics[scale=0.56, trim=10 0 20 0,clip]{figures_SI/MV_Irr_Case_\x_Original.png}
    \includegraphics[scale=0.56, trim=20 0 10 0,clip]{figures_SI/MV_Irr_Case_\x_Zoomin.png}
    \includegraphics[scale=0.56, trim=0 0 20 0,clip]{figures_SI/MV_Rel_Case_\x_Original.png}
    \includegraphics[scale=0.56, trim=10 0 5 0,clip]{figures_SI/MV_Rel_Case_\x_Zoomin.png}

  }
  \caption{Perturbations on matrix $\mathbf{L}$. First column: Perturbations on 
    the irrelevant column of $\mathbf{L}$. Second column: zoom-in of the first row.
    Third column: Perturbations on the relevant column of $\mathbf{L}$. Last column:
    zoom-in of the third column.
    The scales of color in the zoomed figures are different from that of the
    original ones.
    Fig.~6 in the main paper is the third row here.
    \label{sm:fig:mv} }
\end{figure*}

%%%%%%%%%%%%%%%%%%%%%%%%%%%%%%%%%%%%%%%%%%%%%%%%%%%%%%%%%%%%%%%%%%%%%%%%%%%%%%%%
%  Other supporting parts
\newpage
\section*{Acknowledgements}

This material is based on research sponsored by the Air Force Research Laboratory and DARPA under agreement number FA8750-17-2-0146 and the Army Research Office and DARPA under agreement HR00112020039. The U.S. Government is authorized to reproduce and distribute reprints for Governmental purposes notwithstanding any copyright notation thereon.
This work was also supported by DoD grant 72531RTREP and NSF MRI 1828528 to PS.

\newpage
\bibliographystyle{icml2020}
\bibliography{references}

%%%%%%%%%%%%%%%%%%%%%%%%%%%%%%%%%%%%%%%%%%%%%%%%%%%%%%%%%%%%%%%%%%%%%%%%%%%%%%%%
%  Original Bibliography Part
% \subsubsection*{References}

% References follow the acknowledgements.  Use an unnumbered third level
% heading for the references section.  Any choice of citation style is
% acceptable as long as you are consistent.  Please use the same font
% size for references as for the body of the paper---remember that
% references do not count against your page length total.

% \begin{thebibliography}{}
% \setlength{\itemindent}{-\leftmargin}
% \makeatletter\renewcommand{\@biblabel}[1]{}\makeatother
% \bibitem{} J.~Alspector, B.~Gupta, and R.~B.~Allen (1989).
%     \newblock Performance of a stochastic learning microchip.
%     \newblock In D. S. Touretzky (ed.),
%     \textit{Advances in Neural Information Processing Systems 1}, 748--760.
%     San Mateo, Calif.: Morgan Kaufmann.

% \bibitem{} F.~Rosenblatt (1962).
%     \newblock \textit{Principles of Neurodynamics.}
%     \newblock Washington, D.C.: Spartan Books.

% \bibitem{} G.~Tesauro (1989).
%     \newblock Neurogammon wins computer Olympiad.
%     \newblock \textit{Neural Computation} \textbf{1}(3):321--323.
% \end{thebibliography}
%%%%%%%%%%%%%%%%%%%%%%%%%%%%%%%%%%%%%%%%%%%%%%%%%%%%%%%%%%%%%%%%%%%%%%%%%%%%%%%%
\end{document}